%% file: iclr2027_conference.tex
\documentclass{article} 
\usepackage{iclr2027_conference,times}

\input{math_commands.tex}

\usepackage{hyperref}
\usepackage{url}

\title{{Breaking Homogeneity}: Diversifying Persona Sets for Creative LLM Outputs}

\author{Sang Bin Moon \\
School of Electrical and Computer Engineering\\
Purdue University\\
West Lafayette, IN 47909, USA \\
\texttt{moon182@purdue.edu} \\
\And
Nicole Cho, Daniel Borrajo \& Sumitra Ganesh \\
J.P. Morgan AI Research \\
New York, NY 10172, USA \\
\texttt{\{nicole.cho, daniel.borrajo,} \\
\texttt{sumitra.ganesh\}@jpmorgan.com} \\
\AND
Abolfazl Hashemi \\
School of Electrical and Computer Engineering\\
Purdue University\\
West Lafayette, IN 47909, USA \\
\texttt{abolfazl@purdue.edu} \\
}

\usepackage{graphicx} 
\usepackage{tcolorbox}
\usepackage{amsmath, amssymb, amsfonts, amsthm}
\usepackage{algorithm}
\usepackage{algpseudocode}
\usepackage{xcolor}
\usepackage{bbm}
\usepackage{booktabs}
\usepackage{cleveref}
\usepackage{multirow}
\usepackage{array}
\usepackage{longtable}
\usepackage{enumitem}
\usepackage{caption}
\usepackage{makecell}
\usepackage{multirow}
\usepackage{etoc}

\newtheorem{theorem}{Theorem}
\newtheorem{lemma}{Lemma}

\newtheorem{assumption}{Assumption}
\newtheorem{remark}{Remark}

\algrenewcommand\algorithmiccomment[1]{\hfill\textcolor{gray}{$\triangleright$ #1}}

\newcommand{\dper}[1]{d_{#1}^{\mathrm P}}
\newcommand{\dcov}[1]{F_{\mathrm{cov},#1}}
\newcommand{\ddisp}[1]{F_{\mathrm{disp},#1}}

\tcbuselibrary{listings,breakable}

\newtcblisting{promptbox}[1]{
    breakable,
    before skip=8pt,
    after skip=8pt,
    colback=gray!10,
    colframe=black,
    coltitle=white,
    boxrule=1.5pt,
    fonttitle=\bfseries,
    title=#1,
    listing only,
    listing options={
        basicstyle=\normalfont,
        breaklines=true,
        breakindent=0pt,
        breakautoindent=false,
        columns=fullflexible,
        keepspaces=true,
        extendedchars=true,
        literate={—}{{---}}1
                 {–}{{--}}1
                 {é}{{\'e}}1
                 {…}{{\dots}}1
    }
}

\iclrfinalcopy 
\begin{document}
\etocdepthtag.toc{mainpaper}

\maketitle

\begin{abstract}
Language models often produce homogeneous responses to open-ended tasks;
such homogeneity can spawn groupthink---the convergence of ideas toward a singular and potentially suboptimal decision. 
We formulate persona diversification as a set-level conditioning problem and study two orthogonal design choices: selecting versus generating personas, and space-filling versus frontier-seeking diversity. We instantiate this design space with four methods spanning coverage and dispersion subset selections, uniform-coverage sampling, and evolutionary persona generation.
Evaluations on the Alternative Uses Task (AUT), Infinity-Chat, and Divergent Association Task (DAT) show the benefits of the proposed methods across tasks and creativity objectives.
On AUT, evolutionary persona generation increases response diversity by 78.8\%, originality by 26.1\%, flexibility by 49.5\%, and holistic creativity by 13.9\% over task-only prompting, while maintaining 98.5\% validity; on Infinity-Chat, it nearly doubles persona-induced response separation relative to random personas. Moreover, evolutionary personas compose with creativity-optimized prompting, further increasing its response diversity by 18.6\% and creativity by 6.3\%. These results establish persona-set geometry as a task-agnostic mechanism for eliciting divergent LLM outputs, 
and support persona diversification as a reusable complement to prompt optimization.
\end{abstract}

\section{Introduction}
Homogeneity in Large Language Model (LLM) outputs, specifically those that pertain to a narrow subset of WEIRD (Western, Educated, Industrialized, Rich, and Democratic) responses, severely limits the potential of LLMs \citep{anthis2025llmsocialsimulationspromising} and engenders a problematic artificial hivemind \citep{jiang2026artificial}. This homogeneity is known to be spawned by the RLHF fine-tuning pipeline endemic to frontier models \citep{kirk2024understandingeffectsrlhfllm, west2025basemodelsbeataligned}; such homogeneity not only stifles creativity but also creates systematic risks of groupthink. For example, relying on LLMs for investment recommendations can engender similar or identical advice across independent teams - which results in compounded risk \citep{imf2024ai,bis2024ai}. 

Conditioning a language model on a \emph{persona}, defined as the societal, biological, and personal traits of the human user, has been presented as a potential lever for steering behavior and eliciting distinct perspectives \citep{choi2024picle,ge2024scaling}. 
However, empirically measuring the impact of diverse personas on the model's creative ability to generate divergent solutions for the same task is relatively understudied \citep{paglieri2026personageneratorsgeneratingdiverse}. We observe three gaps in current research: first, the limited breadth of evolutionary methods that truly make a \textit{set} of personas diverse; second, the underexplored notions of what actually makes a set of personas \textit{diverse} and a set of outputs \textit{creative}; and third, the understudied realm of whether a set of diverse personas can trigger the model's divergent trajectories to generate more creative outputs. 

In this work, we utilize persona as an explicit, optimizable variable and hypothesize that 
a set of truly diverse personas can successfully trigger the alternative, divergent trajectories embedded within the model that would otherwise remain inactivated. Moreover, we hypothesize that divergent gains compound when a \emph{set} of personas is explicitly diversified, ensuring that the induced responses spread across many regions at once, for the same task. Thus, our focus is threefold: what makes a set of personas \emph{diverse}, how can we \emph{induce} that diversity with guarantees on the resulting persona-side spread, and does the spread ultimately improve creativity?

\begin{figure}[t!]
\centering
\includegraphics[width=\linewidth]{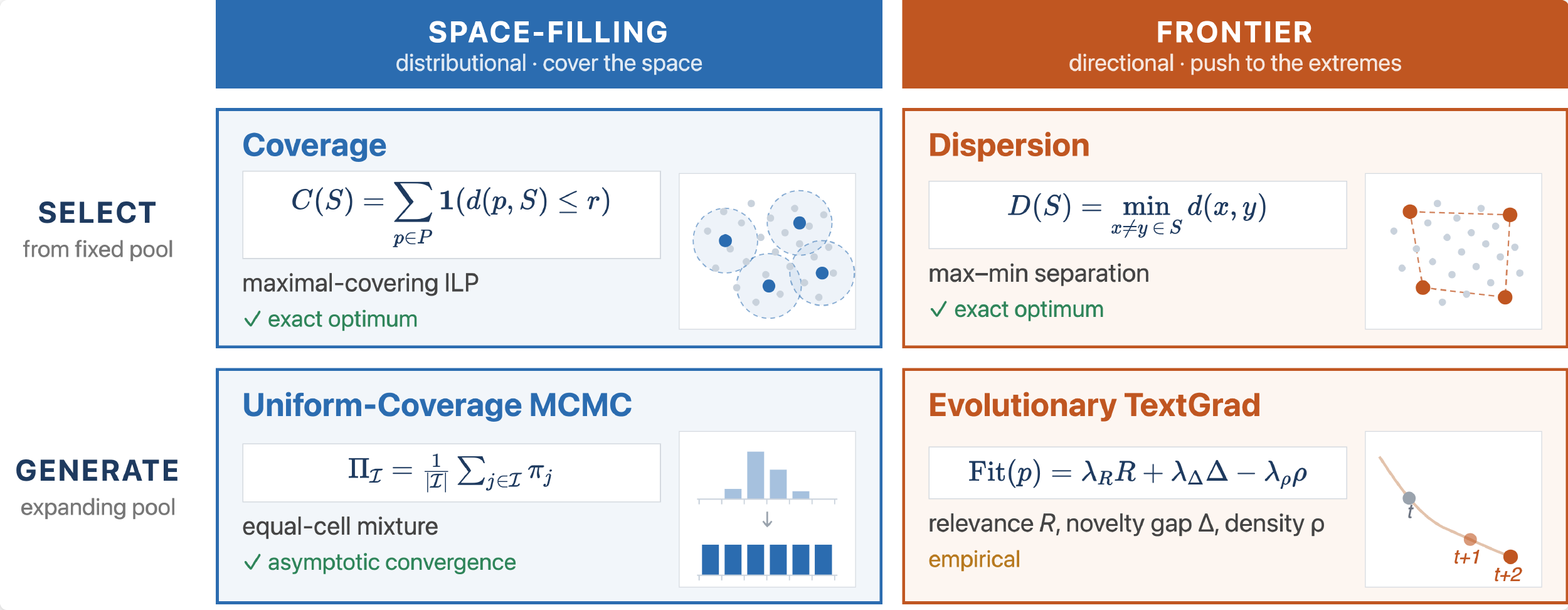}
\vspace{-6mm}
\caption{\textbf{Persona diversification design space.} Two \emph{principles} of diversity (columns) are each instantiated at two
\emph{levels of control} (rows). The \emph{space-filling} principle is implemented by \textsc{Coverage} and UC-MCMC, whereas the \emph{frontier-seeking} principle is instantiated by \textsc{Dispersion} and evolutionary \textsc{TextGrad}. Three algorithms have formal guarantees for persona-space objectives and sampling targets, while evolutionary generation is empirical.}
\label{fig:framework}
\vspace{-4mm}
\end{figure}

To answer these questions, we organize persona diversification along two axes (Figure~\ref{fig:framework}): selection from a fixed candidate pool versus generation beyond it, and space-filling coverage versus frontier-seeking separation. Crossing these two axes yields four methods: \textsc{Coverage} and \textsc{Dispersion} for selection, and Uniform-Coverage MCMC (UC-MCMC) and evolutionary \textsc{TextGrad} for generation \citep{yuksekgonul2025textgrad}. \textsc{Coverage} and \textsc{Dispersion} achieve global optima for their population objectives, UC-MCMC converges asymptotically to an equal-cell target over reachable cells, and evolutionary generation is evaluated empirically.
 

Our central finding is that the performance gains are driven by the specific structural geometry of the persona sets, not merely the presence of a conditioning variable.
Across the benchmarks, greater persona-set dispersion is positively associated with greater response dispersion. 
By quantifying this displacement against the model’s unconditioned output, we demonstrate the specific advantage a persona-conditioned query provides over a standard task-only prompt: namely, increased creativity and a meaningful departure from the model's narrow default behavior.

Our diversification objectives depend on persona geometry rather than downstream task performance, allowing the same persona sets to be reused across tasks. Persona conditioning supplements prompt engineering and can be combined with advanced prompting and reasoning scaffolds.

\paragraph{Contributions.}
\begin{itemize}[noitemsep, leftmargin=*, topsep=0pt, partopsep=0pt, label={\tiny\raisebox{0.5ex}{$\blacktriangleright$}}]
    \item A \emph{design space} for inducing persona diversity that crosses two levels of control (selection vs. generation) with two diversity principles (space-filling vs. frontier), unifying four methods under one framework (Figure~\ref{fig:framework}). The construction is task-agnostic and compatible with prompt optimization and reasoning scaffolds.
    \item On the \emph{selection} axis, exact algorithms for both principles, maximal-covering location \textsc{Coverage} and max--min \textsc{Dispersion}, with global-optimality guarantees.
    \item On the \emph{generation} axis, two complementary methods: a Metropolis-within-Gibbs sampler with provable asymptotic convergence to an equal-cell mixture over reachable semantic cells, and an evolutionary \textsc{TextGrad} procedure that searches for isolated, low-density personas.
    \item Empirical evidence that persona diversification can improve response creativity across three benchmarks, including comparisons with task-only, random-persona and prompt engineering baselines.
\end{itemize}

\begin{figure}[t!]
    \centering
    \includegraphics[width=0.95\columnwidth]{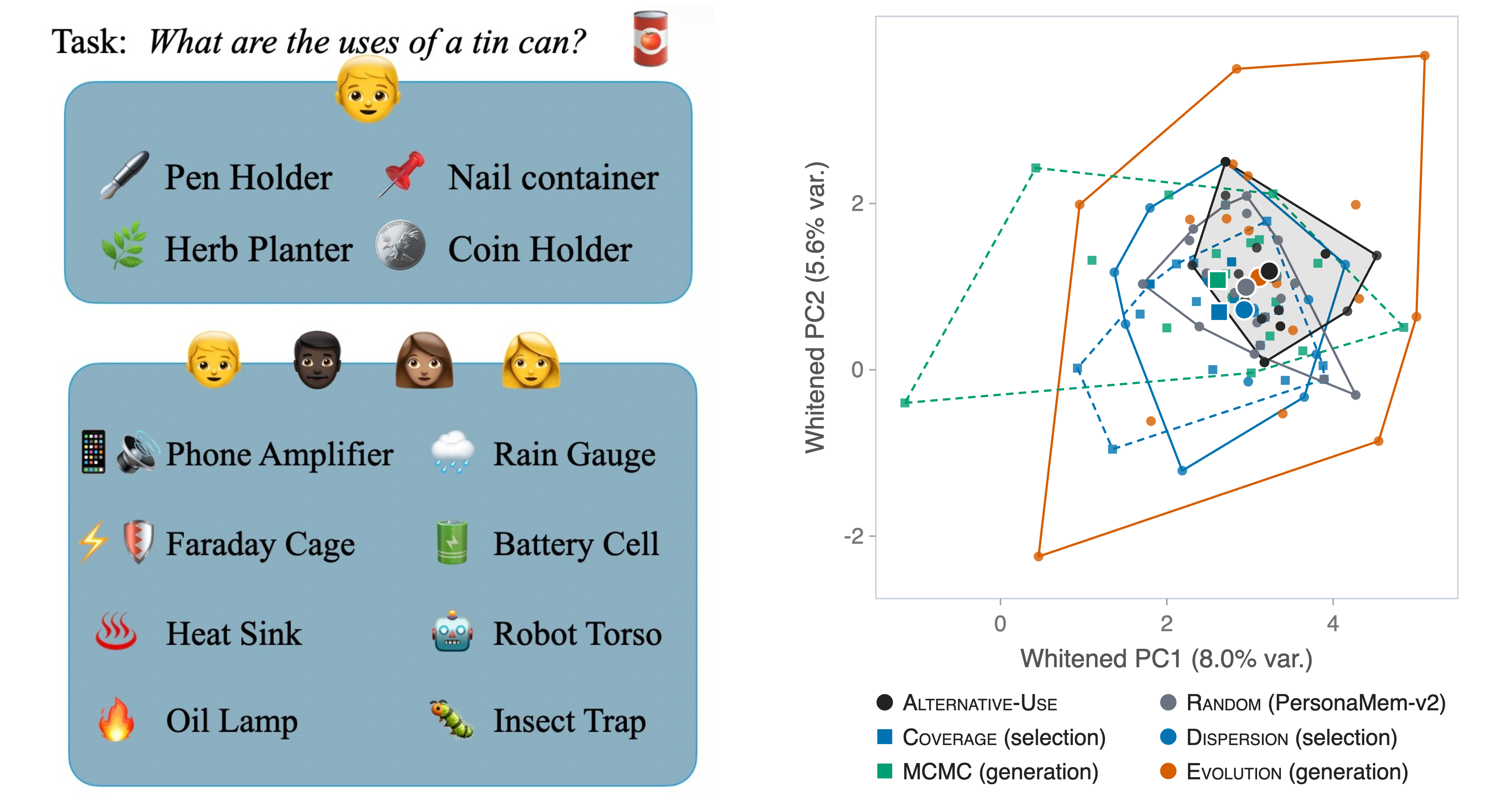}
    \vspace{-3mm}
    \label{fig:tin-can}
    \caption{\textbf{Persona conditioning diversifies Alternative Uses Task responses.}
    \textit{Left:} Example uses for a tin can under task-only prompting (top) and persona conditioning (bottom), illustrating a broader range of ideas.
    \textit{Right:} Mahalanobis-whitened PCA projection of response embeddings across persona conditions, with convex hulls showing their semantic spread.}
    \vspace{-6mm}
\end{figure}

\section{Problem Formulation}
\label{sec:problem}
While a single persona provides pointwise conditioning, it cannot dictate the structure of a broader collection of viewpoints. To address this set-level challenge, we define the \emph{persona diversification} problem. The objective is to design a collection of persona-conditioned queries that hold the task prompt and base response model constant, isolating the persona set as the sole optimization variable. We formulate this shared problem and its underlying semantic geometry below.

\subsection{The Set-Level Conditioning Problem}
Let $\mathcal T$ be the text-string space and $\Omega\subseteq\mathcal T$ be the specialized persona-string space. A persona is defined as a short, natural-language description $p\in\Omega$, and a persona condition is a set $\mathcal S$ of fixed cardinality $|\mathcal S|=k$. For a task prompt $x\sim\mathcal D$ and a frozen downstream response model $G_{\theta_{\mathrm R}}$, the model generates a response conditioned on $p$ as
\begin{equation}
    y_p\sim G_{\theta_{\mathrm R}}(\,\cdot\mid x,p),
    \qquad p\in\mathcal S,
    \label{eq:persona-conditioned-response}
\end{equation}
where $p=\varnothing$ denotes the persona-free reference baseline. Holding $x$ and $G_{\theta_{\mathrm R}}$ fixed, the downstream goal is for the induced response collection $\{y_p:p\in\mathcal S\}$ to span a broader, more creative range of valid, useful outputs than task-only prompting or an unoptimized persona set.

Crucially, our optimization variable and evaluation target are deliberately distinct. Our algorithms construct persona sets based entirely on frozen \emph{persona-space} geometry, whereas creativity is evaluated strictly on the induced \emph{response-space} outputs. We do not assume that persona-space distance monotonically transfers to response-space distance, novelty, or quality. Consequently, the guarantees we provide apply solely to the persona-space; whether this persona diversification reliably translates to improved response creativity remains an empirical hypothesis evaluated in \Cref{sec:experiment}.

\subsection{Semantic Geometry}
Let $\mathcal P=\{p_i\}_{i=1}^{N_{\mathcal P}}\subseteq\Omega$ denote a generic finite candidate population, with $[N_{\mathcal P}]=\{1,\ldots,N_{\mathcal P}\}$. Selection operates on the baseline population $\mathcal P_0$, while generation expands it into $\mathcal P_T$ before downstream selection. 
Let $\phi:\mathcal T\to\mathbb R^d$ denote the full, unit-normalized output of a frozen text encoder, which supports Matryoshka Representation Learning (MRL) \citep{kusupati2022matryoshka}. For Mahalanobis distance, we use the prefix $\widetilde\phi(p)=\phi(p)_{1:d_{\mathrm M}}$ without renormalizing after truncation.

This encoder induces persona-space dissimilarities $\dper{m}$, indexed by $m\in\{\mathrm{cos},2,\mathrm{Mah}\}$, as
\begin{align}
    \dper{\mathrm{cos}}(p,p') &= 1-\frac{\phi(p)^\top\phi(p')}{\|\phi(p)\|_2\|\phi(p')\|_2},    \label{eq:persona-cosine}\\
    \dper{2}(p,p') &= \|\phi(p)-\phi(p')\|_2,    \label{eq:persona-l2}\\
    \dper{\mathrm{Mah}}(p,p') &= \sqrt{(\widetilde\phi(p)-\widetilde\phi(p'))^\top \widehat\Sigma_{\mathrm P}^{-1} (\widetilde\phi(p)-\widetilde\phi(p'))}. \label{eq:persona-mahalanobis}
\end{align}
Here, $\widehat\Sigma_{\mathrm P}$ is the sample covariance of the truncated embeddings $\{\widetilde\phi(p):p\in\mathcal P_0\}$, regularized by adding $\varepsilon I_{d_{\mathrm M}}$, with $\varepsilon=10^{-6}$ and $d_{\mathrm M}=128$.

Both finite-pool selectors rely on the pairwise dissimilarity matrix and its realized spectrum
\begin{align}
    D^{\mathrm P}_{m,ij} &= \dper{m}(p_i,p_j),
    & \mathbf D_m^{\mathrm P} &= \bigl(D^{\mathrm P}_{m,ij}\bigr)_{i,j=1}^{N_{\mathcal P}}, \label{eq:persona-pairwise-matrix}\\
    \Lambda_m(\mathcal P) &= \{D^{\mathrm P}_{m,ij}:i,j\in[N_{\mathcal P}]\},
    & L_m &= |\Lambda_m(\mathcal P)|,   \label{eq:persona-pairwise-spectrum}
\end{align}
Let $r$ denote a unique dissimilarity value (or radius) drawn from this spectrum. The $L_m$ distinct values of $\Lambda_m(\mathcal P)$, including $0$, are sorted in ascending order such that $r_{m,(1)}<\cdots<r_{m,(L_m)}$.

\section{Diversification Algorithms}
\label{sec:method}
Building on the semantic geometry established in \Cref{sec:problem}, this section structures our algorithms as a progression in both control and exploratory reach. We begin with \textbf{Selection} operators (\Cref{sec:selection}) that enforce space-filling and frontier diversity within a fixed candidate population. To break the performance ceiling imposed by this finite pool, we escalate to open-ended \textbf{Generation} methods (\Cref{sec:mcmc,sec:evolution}), actively expanding the persona support through a space-filling MCMC sampler and a frontier-seeking evolutionary TextGrad search. To maintain narrative flow, all formal proofs and extended mathematical formulations are deferred to Appendix \ref{app:selection} and \ref{app:generation}.

\subsection{Diversification Design Space}
Our framework is organized along two independent axes, as illustrated in \Cref{fig:framework}. Selection extracts a subset $\mathcal S\subseteq\mathcal P$ of size $k$ from a frozen candidate pool $\mathcal P$, whereas generation expands the baseline pool $\mathcal P_0$ with new valid personas before applying a downstream selection protocol to obtain the final $k$-persona condition. We analyze generation and downstream selection separately because guarantees for the generation process do not automatically transfer to the selected subset.

\subsection{Diverse Persona Selection: Coverage and Dispersion}
\label{sec:selection}
To extract a $k$-sized subset $\mathcal{S}$ from a finite candidate pool $\mathcal{P}$, we operationalize the space-filling and frontier principles via two exact selection algorithms. For the space-filling principle, we formulate selection as thresholded coverage. A persona $p_i$ is considered covered if its distance to the nearest selected exemplar in $\mathcal{S}$ is at most a radius $r$. The fraction of the population covered evaluates to
\begin{equation}
    F_{\mathrm{cov},m}(\mathcal{S};\mathcal{P},r) = \frac{1}{N_{\mathcal{P}}}\sum_{i=1}^{N_{\mathcal{P}}} \mathbbm{1}\Bigl\{\min_{p_j\in\mathcal{S}} d^{\mathrm{P}}_m(p_i,p_j)\leq r\Bigr\}.
    \label{eq:coverage}
\end{equation}
To avoid metric-specific scaling from manually tuning $r$, we adaptively calibrate the smallest radius $r^\star_m$ at which some subset successfully covers the target of $n_{\mathrm{cov}} = \lceil c_{\mathrm{cov}}N_{\mathcal{P}}\rceil$ personas, written as
$$r^\star_m(\mathcal{P};c_{\mathrm{cov}}) = \min\Bigl\{r\geq 0 : \max_{\substack{\mathcal{S}\subseteq\mathcal{P}\\|\mathcal{S}|=k}} N_{\mathcal{P}} \, F_{\mathrm{cov},m}(\mathcal{S};\mathcal{P},r) \geq n_{\mathrm{cov}}\Bigr\}.$$
The algorithm lexicographically attains this minimal radius, then breaks ties among radius-optimal subsets by maximizing the covered count at the threshold. 

Since $F_{\mathrm{cov},m}$ changes only at the discrete values of the pairwise spectrum $\Lambda_m(\mathcal{P})$, finding $r^\star$ reduces to an exact discrete search over the sorted distances. For a target coverage count $n_{\mathrm{cov}}$, the greedy solution $G_m(r)\ge n_{\mathrm{cov}}$ certifies feasibility of a trial radius $r$ and establishes the upper bound $r^\star\leq r$. By searching over the sorted distances, the upper bound becomes tighter. At the same time, the classical submodular guarantee bounds the optimal covered count from above as $n_{\mathrm{opt}}\le\lfloor G_m(r)/(1-1/e)\rfloor$ \citep{nemhauser1978analysis}. If $n_{\mathrm{cov}}>\lfloor G_m(r)/(1-1/e)\rfloor$, the radius is certified infeasible, establishing the lower bound $r^\star>r$. These certificates narrow the bracket on $r^\star$, efficiently reducing the number of exact-solver calls. This iterative algorithm (Algorithm~\ref{alg:exact_coverage}) solves both the smallest feasible radius $r^\star$ and the maximum coverage at $r^\star$.

\begin{theorem}[Exact Calibrated Coverage]
\label{thm:exact_coverage}
In exact mode, provided every radius feasibility probe is certified, the adaptive search algorithm (Algorithm~\ref{alg:exact_coverage}) terminates and returns a globally minimal radius $r^\star = r^\star_m(\mathcal{P};c_{\mathrm{cov}})$, together with a subset $\mathcal S^\star_{\mathrm{acov},m} (\mathcal P;c_{\mathrm{cov}})$ of cardinality $k$ satisfying the maximum coverage at $r^\star$. For every $\mathcal S\subseteq\mathcal P$ with $|\mathcal S|=k$ and $1\leq k\leq N_{\mathcal P}$,
$$\dcov{m}\!\left(\mathcal S^\star_{\mathrm{acov},m}(\mathcal P;c_{\mathrm{cov}});\mathcal P,r^\star\right) \geq \dcov{m}(\mathcal S;\mathcal P,r^\star).$$
In particular, its covered count at $r^\star$ is at least $n_{\mathrm{cov}}$.
\end{theorem}
The full Integer Linear Program formulation, algorithmic pseudocode, and supporting lemmas and proofs are deferred to Appendix~\ref{app:coverage}.

On the other hand, the frontier principle prioritizes structural separation. The discrete max–min $p$-dispersion problem \citep{erkut1990discrete} maximizes the minimum dispersion (the smallest pairwise dissimilarity within a subset) over all size-$k$ subsets of the candidate pool. For $|\mathcal S|\geq2$, the minimum dispersion objective is defined as
\begin{equation}
    \ddisp{m}(\mathcal S) := \min_{\substack{p,p'\in\mathcal S\\p\neq p'}} \dper{m}(p,p').
    \label{eq:min-dispersion}
\end{equation}
Since finding a max--min diverse subset is generally NP-hard, our exact method again exploits the discrete topology of the solution space. Note that the optimal dispersion value must equal exactly one of the pairwise distances in the sorted spectrum $\Lambda_m(\mathcal{P})$, so we leverage the threshold-graph equivalence between max–min diversity and clique feasibility \citep{dellacroce2009clique}. The algorithm builds a graph by inserting edges in descending order of pairwise dissimilarity, transforming the optimization into a series of Boolean structural checks. After a new edge $\{p_a,p_b\}$ is added, the algorithm searches for a $(k-2)$-clique in their mutual neighborhood. If found, combining it with the endpoints yields a complete size-$k$ clique; if not, the algorithm checks for the next pair. Since the algorithm searches the descending discrete distance spectrum sequentially, the subset returned upon the first clique completion is guaranteed to be optimal.

\begin{theorem}[Exact Max--Min Dispersion]
For a finite candidate pool $\mathcal P$, $2\leq k\leq|\mathcal P|$, and fixed symmetric nonnegative dissimilarities $\dper{m}$, the incremental-clique algorithm (Algorithm~\ref{alg:exact_max_min_dispersion}) terminates and returns a size-$k$ subset $\mathcal{S}^\star_{\mathrm{disp},m}(\mathcal{P})$ satisfying
$$\ddisp{m}(\mathcal S^\star_{\mathrm{disp},m}(\mathcal P)) = \max_{\substack{\mathcal S\subseteq\mathcal P\\|\mathcal S|=k}} \ddisp{m}(\mathcal S).$$
\end{theorem}


\subsection{Uniform-Coverage MCMC Persona Generation}
\label{sec:mcmc}

To break the diversity ceiling imposed by a finite baseline pool $\mathcal P_0$, we transition to generation beyond $\mathcal P_0$. \emph{Uniform-Coverage MCMC} (UC-MCMC) applies the space-filling principle directly to the generation process by targeting equal probability mass across reachable cells of a fixed partition of equal-area directional cells. Within each cell, a frozen Large Language Model acts as the base law, ensuring the sampled personas remain linguistically plausible and structurally valid while the MCMC framework enforces the geometric uniformity.

Since the total number of reachable cells is unknown in advance, UC-MCMC employs a hindsight-spawning random scan that maintains a dynamically growing population of parallel Markov chains, with one chain for each discovered cell. Each iteration selects an active cell uniformly at random and proposes a new persona using the frozen LLM. If the generated persona lands in the currently selected cell, it undergoes a standard Metropolis–Hastings (MH) correction to preserve the within-cell target law. If the proposed persona lands in a previously empty, undiscovered cell, it serves as a hindsight seed that activates the new cell and initializes a new parallel chain.

\paragraph{Theoretical Guarantee.}
By consolidating the cell-kernel reversibility and eventual discovery properties, we establish three formal guarantees for the emitted sequence.

\begin{theorem}[Anytime Cell Uniformity and Target Convergence]
Fix a deterministic validity rule $\operatorname{val}$, a base generation law $b$, and a cell map $c:\Omega\to[M]$ assigning each persona to one of $M$ cells. The encoder, whitening transform, and partition defining $c$ are frozen. Assuming $b(p)>0$ for every valid persona, we define the valid state space and base mass of cell $j$ as
$$\Omega_j = \{p\in\Omega:\operatorname{val}(p)=1,\ c(p)=j\}, \qquad B_j=b(\Omega_j).$$
Let $\mathcal R=\{j:B_j>0\}$ be the set of reachable cells and $\pi_j=b(\,\cdot\mid\Omega_j)$ their within-cell target laws. Suppose Algorithm~\ref{alg:ucmcmc} is run from a fixed, nonempty active set $\mathcal I_0\subseteq\mathcal R$. With all proposal kernels and mixture weights frozen, the global proposal draws independently from $b$ with a fixed probability $\omega_0>0$. Let $\mathcal{I}_{t-1} \subseteq \mathcal{R}$ denote the active set after iteration $t$, $P_t^{\mathrm{emit}}$ the emitted persona, and $\mathcal F_{t-1}$ the complete history prior to iteration $t$. The following properties hold.

\begin{itemize}[noitemsep, leftmargin=*, topsep=0pt, partopsep=0pt, label={\tiny\raisebox{0.5ex}{$\blacktriangleright$}}]
    \item \textbf{Exact Anytime Cell Uniformity.} For every $t\geq1$ and $j\in[M]$, the spatial allocation of the emitted persona is strictly uniform over the currently active set
    $$\Pr\!\left(c(P_t^{\mathrm{emit}}) = j \mid \mathcal{F}_{t-1}\right) = \frac{\mathbbm{1}\{j \in \mathcal{I}_{t-1}\}}{|\mathcal{I}_{t-1}|}.$$
    \item \textbf{Finite-Time Discovery.} For every $T\ge1$, the probability that the active set has not yet expanded to cover the entire reachable set is bounded by
    $$\Pr(\mathcal{I}_T \neq \mathcal{R}) \leq \sum_{j \in \mathcal{R} \setminus \mathcal{I}_0} (1-\omega_0 B_j)^T.$$
    Consequently, all reachable cells are discovered in finite time almost surely.
    \item \textbf{Asymptotic Convergence.} As $t\to\infty$, the law of the emitted persona converges in total variation to the equal-cell mixture
    $$\left\|\mathcal L(P_t^{\mathrm{emit}})-\Pi_{\mathcal R}\right\|_{\mathrm{TV}}\longrightarrow 0, \qquad \Pi_{\mathcal R}=\frac{1}{|\mathcal R|}\sum_{j\in\mathcal R}\pi_j.$$
\end{itemize}
\end{theorem}


\subsection{Evolutionary TextGrad Persona Generation}
\label{sec:evolution}

To actively push personas toward the extreme, unexplored regions of the semantic space, we formulate generation as an evolutionary constrained optimization problem. Unlike UC-MCMC’s distributional target, this is an empirical search that iteratively expands the population $\mathcal{P}_t$ toward valid, low-density semantic frontiers. We evaluate each persona across two geometric axes: \emph{Isolation} (gap to its nearest neighbor) and \emph{Sparsity} (unnormalized von Mises-Fisher kernel-density score).

Since personas are discrete text rather than continuous vectors, we employ \emph{TextGrad} \citep{yuksekgonul2025textgrad} to guide offspring generation with natural-language feedback derived from the geometric fitness scores. At each iteration, the algorithm executes three phases:
\begin{enumerate}[noitemsep, leftmargin=*, topsep=0pt, partopsep=0pt]
    \item \textbf{Parent Selection.} A tournament selects elite incumbents as parents that maximize isolation or minimize density.
    \item \textbf{Mutation and Admission.} The frozen LLM mutates the parents using the textual gradient on the composite fitness function. Then the offspring is admitted according to the same fitness function and thresholds after the validity checks.
    \item \textbf{Sibling Deduplication.} Admitted offspring in the same iteration are mutually deduplicated to remove near-duplicate siblings.
\end{enumerate}


\begin{figure*}[t]
\centering

\begin{minipage}[t]{0.5\textwidth}
\vspace{0pt}
\centering
\setlength{\tabcolsep}{4.5pt}
\resizebox{\linewidth}{!}{%
\begin{tabular}{@{}llcccc@{}}
\toprule
& \multicolumn{1}{c}{} &
\multicolumn{2}{c}{\makecell{\textbf{Population}}} &
\multicolumn{2}{c}{\makecell{\textbf{Final Subset}}} \\
\cmidrule(lr){3-4}\cmidrule(lr){5-6}
\textbf{Category} &
\makecell[c]{\textbf{Principle}\\[-1pt]\textbf{(method)}} &
\makecell[c]{\textbf{Cell-}\\[-1pt]\textbf{Cov.}} &
\textbf{Hull} &
\makecell[c]{\textbf{Cov.}\\[-1pt]\textbf{@ $r_0^\star$}} &
\textbf{Disp.} \\
\midrule
Reference & \makecell{PersonaMem\\(\textsc{Random})}
& $62.3\%$ & $4.45$ & $0.589$ & $13.67$ \\
\midrule
\multirow[c]{2}{*}[-6pt]{Select} & \makecell{Space filling\\(\textsc{Coverage})}
& $62.3\%$ & $4.45$ & $0.901$ & $10.52$ \\
\cmidrule(l{6pt}){2-6}
 & \makecell{Frontier\\(\textsc{Dispersion})}
& $62.3\%$ & $4.45$ & $0.005$ & $20.98$ \\
\midrule
\multirow[c]{2}{*}[-6pt]{Generate} & \makecell{Space filling\\(\textsc{MCMC})}
& $\mathbf{96.5\%}$ & $6.53$ & $\mathbf{0.905}$ & $7.91$ \\
\cmidrule(l{6pt}){2-6}
 & \makecell{Frontier\\(\textsc{Evolution})}
& $76.8\%$ & $\mathbf{9.47}$ & $0.002$ & $\mathbf{24.42}$ \\
\bottomrule
\end{tabular}%
}
\vspace{-1mm}
\captionof{table}{\textbf{Persona geometry across persona diversification methods.} \emph{Population} metrics evaluate the candidate pool: cell coverage, which is the occupancy of a fixed $1{,}024$-cell grid, and Mahalanobis hull extent. \emph{Subset} metrics evaluate the extracted personas used for response generation: \emph{Cov.@$r_0^\star$}, which is the coverage objective (\cref{eq:coverage}) with $r_0^\star$ defined as \cref{eq:adaptive-radius} ($c=0.9$ of reference population $\mathcal{P}_0$), and dispersion, which is the minimum pairwise Mahalanobis distance. Selection methods share the fixed base pool, while generation methods expand it.}
\label{tab:persona_diversity}
\end{minipage}
\hfill
\begin{minipage}[t]{0.47\textwidth}
\vspace{0pt}
\centering
\includegraphics[width=\linewidth]{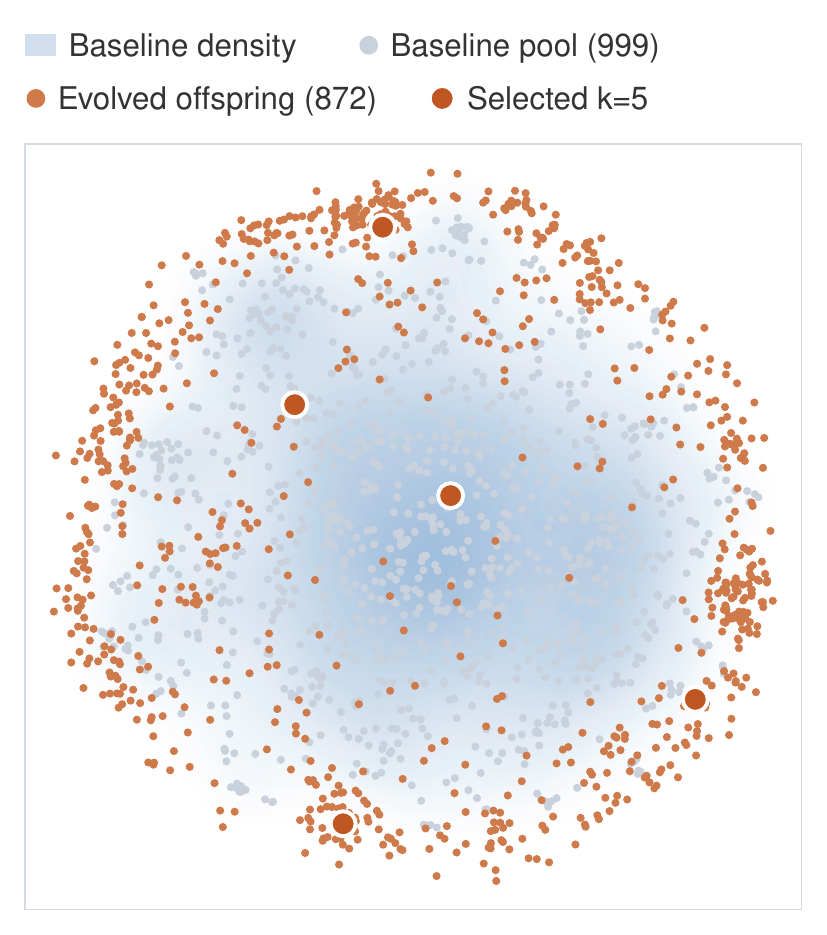}
\vspace{-3mm}
\captionof{figure}{\textbf{2D projection of persona embeddings illustrating evolutionary frontier generation.} The baseline pool (grey) is densely clustered in the center (blue shading), while evolved offspring (orange) populate low-density regions. The downstream selector extracts a $k=5$ subset (large outlined dots).}
\label{fig:persona_embedding}
\end{minipage}

\vspace{-3mm}
\end{figure*}

\section{Evaluation}
\label{sec:experiment}

We evaluate our framework on the Alternative Uses Task (AUT), Infinity-Chat (IC), and the Divergent Association Task (DAT). Responses are generated by Gemma-4-31B-it, with semantic embeddings provided by EmbeddingGemma and automated judgments by Qwen3.6-27B. To evaluate transfer without task-specific persona optimization, we hold each task-agnostic persona set fixed across benchmarks.

Original task-only prompts are denoted as \textsc{Alternative-Use} on AUT and \textsc{Baseline Task} on IC and DAT. For persona conditioning, \textsc{Random} samples uniformly from the PersonaMem-v2 \citep{jiang2025know} pool (itself a random subset of PersonaHub \citep{ge2024scaling}), while \textsc{Coverage} and \textsc{Dispersion} optimize selections from that pool. \textsc{MCMC} and \textsc{Evolution} expand the candidate pool before applying \textsc{Coverage} and \textsc{Dispersion} selectors, respectively. Prompt baselines include \textsc{Creativity-enhanced} \citep{goes2023pushing} on AUT, \textsc{Creative} \citep{schapiro2026creativityneuro} on DAT, and reasoning scaffold of DMAD (Diverse Multi-Agent Debate) \citep{liu2025breaking}. More prompt baselines are compared in Appendix~\ref{app:results}, but three top-performing baselines are reproduced in \Cref{sec:experiment}. Combined conditions add the selected personas with the corresponding prompt or reasoning scaffold.

Before filtering for validity, each condition generates 25 uses per AUT object, 50 responses per IC query, and 35 ten-noun lists for the DAT. We assess response diversity and creativity alongside validity and utility to identify tradeoffs. We employ a suite of metrics spanning human-derived categorical clustering, LLM-as-a-judge ratings, and semantic embedding metrics. Full prompts, additional baselines and controls, metric definitions, and results are provided in Appendix~\ref{app:experiment_details}--\ref{app:prompt}.


\paragraph{RQ1: What persona geometries do the two diversity principles produce at the selection and generation levels?}

As detailed in \Cref{tab:persona_diversity}, the framework effectively operationalizes its underlying principles, highlighting a fundamental tension between representation and separation. To contextualize these structural gains, we benchmark against \textsc{Random} as a standard, geometrically blind persona-conditioning pipeline.
At the selection level, \textsc{Coverage} covers $90.1\%$ of the baseline population, whereas \textsc{Dispersion} increases minimum pairwise distance by $53.5\%$ over \textsc{Random} ($13.67\to20.98$).
At the generation level, our methods break the fixed pool's structural boundaries. \textsc{MCMC} expands the population's cell coverage from $62.3\%$ to $96.5\%$, and its \textsc{Coverage}-selected subset covers $90.5\%$ of the population at $r^\star$.
In contrast, \textsc{Evolution} more than doubles the pool's Mahalanobis hull extent from $4.45$ to $9.47$. \Cref{fig:persona_embedding} visualizes this expansion, where the baseline pool remains densely clustered in the center and the evolved offspring deliberately push outward to establish a novel, low-density semantic boundary. Applying \textsc{Dispersion} to this expanded pool yields the highest minimum pairwise distance of $24.42$, a $78.6\%$ increase over \textsc{Random}.

Pool expansion alone can increase the metrics in \Cref{tab:persona_diversity} including, coverage, hull extend and dispersion. RQ1 therefore emphasizes the contrastic geometries achieves by two principles: \textsc{MCMC} achives greater cell coverage, whereas \textsc{Evolution} yields greater hull extent despite smaller pool than \textsc{MCMC}. These geometric gains do not guarantee improved response creativity. Therefore, RQ2--RQ4 evaluate whether persona sets optimized without downstream task induce more diverse and creative responses.

\paragraph{RQ2: How does persona-set dispersion relate to downstream response diversity?}

As shown in \Cref{fig:persona-response-dispersion}, across 14 configurations per benchmark, greater realized persona-set separation positively correlates with greater mean response dispersion on AUT ($r=.40$, descriptive 90\% CI $[-.19,.78]$), and Infinity-Chat ($r=.55$, $[.12,.81]$). Despite the variance, intervention-aligned comparisons reveal a robust directional trend: all six matched \textsc{Coverage}$\rightarrow$\textsc{Dispersion} escalations increased task-averaged response dispersion across all 12 contrasts. Additional analyses using Vendi score and convex-hull, showing even stronger correlations between persona-set dispersion and response diversity, appear in Appendix~\ref{app:persona_diversity}.

\begin{figure*}[t]
    \centering
    \includegraphics[width=\textwidth]{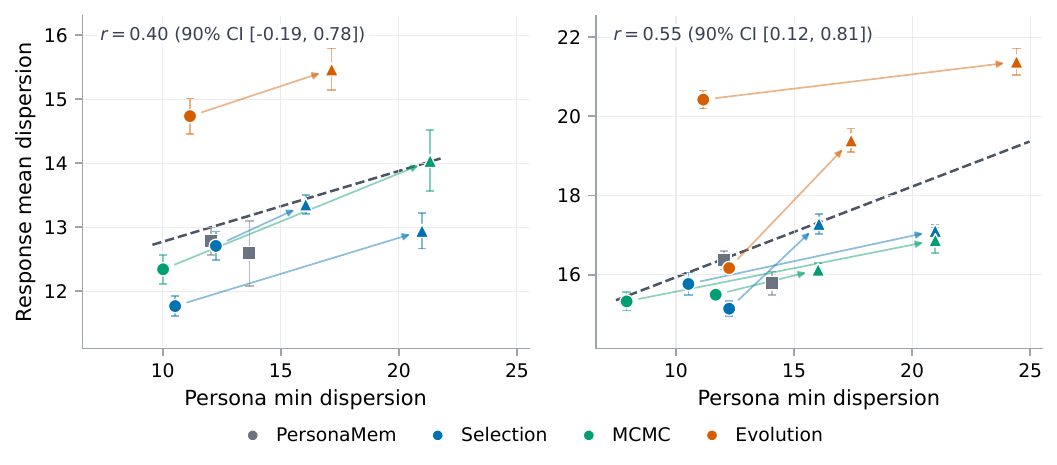}
    \vspace{-8mm}
    \caption{\textbf{Persona set and response dispersion on AUT (left) and Infinity-Chat (right).} Every points represent different configurations. Arrows connect six matched \textsc{Coverage}$\rightarrow$\textsc{Dispersion} contrasts while holding candidate-pool and distance metric fixed. Dashed lines show the linear fits and annotations report configuration-level Pearson correlations and 90\% confidence intervals.}
    \label{fig:persona-response-dispersion}
    \vspace{-4mm}
\end{figure*}

\paragraph{RQ3a: Do diversified task-agnostic personas improve meaningful creativity, and at what cost?} 

Yes. Task-agnostic diversification yields substantial creativity gains with minimal trade-offs. As shown in the top half of \Cref{tab:aut_creativity_summary}, \textsc{Random} persona conditioning improves diversity by $8.7\%$ ($1.42\to1.54$), originality by $10.2\%$ ($10.79\to11.89$) and creativity by $8.9\%$ ($2.72\to2.96)$ while sacrificing virtually no utility or validity. Our \textsc{Dispersion} selector further improves diversity, originality and flexibility, while our evolutionarily generated personas achieve the highest values across all four target metrics among the standard-prompt conditions. Relative to \textsc{Alternative-Use}, \textsc{Evolution} improves diversity by $78.8\%$ ($1.42\to2.53$), originality by $26.1\%$ ($10.79\to13.61$), flexibility by $49.5\%$ ($.288\to.431$) and creativity score by $13.9\%$ ($2.72\to3.10$). Compared to \textsc{Random}, \textsc{Evolution} yields relative improvements ranging from $4.7\%$ in creativity up to $\mathbf{64.3\%}$ in diversity. The cost of these leaps is modest: utility is reduced by just $0.22$ ($-5.1\%$) from the original task (\textsc{Alternative-Use}), while still producing valid responses $98.5\%$ of the time.

\paragraph{RQ3b: Do the creativity gains compose with prompt engineering?}

Yes. Persona diversification composes with, rather than competes against, advanced prompt engineering. As demonstrated in the bottom half of \Cref{tab:aut_creativity_summary}, applying our evolutionarily generated personas to the creativity-optimized prompt of \citet{goes2023pushing} amplifies its effects, yielding improvements of $+18.6\%$ in diversity, $+2.9\%$ in originality, $+15.3\%$ in flexibility, and $+6.3\%$ in creativity, all while maintaining utility ($3.38$) and improving validity ($.968\to.979$). It also shows competitive results against advanced reasoning frameworks, including \textsc{DMAD} \citep{liu2025breaking}, a state-of-the-art baseline that diversifies reasoning paths rather than personas.
Because the personas are generated without optimizing downstream task performance, they are particularly valuable for everyday, zero-shot user queries where elaborate prompt engineering is impractical.

\begin{table*}[t]
\centering\small
\setlength{\tabcolsep}{3pt}
\caption{\textbf{Response creativity evaluation on AUT.}
\emph{Diversity} is the hull extent in a 5-D PCA projection of Mahalanobis-whitened response embeddings, \emph{Originality} is the mean Mahalanobis distance to the common-use reference set, \emph{Flexibility} is the number of distinct categories per valid use, and \emph{Creativity} is the rank-normalized holistic score \citep{goes2023pushing} from an independent LLM judge. Subscripts give 95\% confidence-interval half-widths across five random seeds. 
Utility \citep{stevenson2022putting} and validity are \emph{controls} and are not bolded. \textbf{Bold} = best results within each group.}
\label{tab:aut_creativity_summary}
\vspace{-2mm}
\resizebox{\textwidth}{!}{%
\begin{tabular}{@{}lcccccc@{}}
\toprule
Variant & Diversity & Originality & Flexibility
& Creativity & Utility & Validity \\
\midrule
\textsc{Alternative-Use} & $1.42_{\pm 0.09}$ & $10.79_{\pm 0.32}$ & $.288_{\pm .022}$ & $2.72_{\pm 0.09}$ & $4.30_{\pm 0.05}$ & $.984_{\pm .014}$ \\
\addlinespace
\textsc{Random} & $1.54_{\pm 0.12}$ & $11.89_{\pm 0.48}$ & $.331_{\pm .024}$ & $2.96_{\pm 0.04}$ & $4.28_{\pm 0.02}$ & $.993_{\pm .006}$ \\
\addlinespace
\textsc{Dispersion} (selection) & $1.96_{\pm 0.17}$ & $12.00_{\pm 0.28}$ & $.357_{\pm .048}$ & $2.92_{\pm 0.04}$ & $4.23_{\pm 0.09}$ & $.986_{\pm .010}$ \\
\textsc{Evolution} (generation) & $\mathbf{2.53}_{\pm 0.20}$ & $\mathbf{13.61}_{\pm 0.17}$ & $\mathbf{.431}_{\pm .045}$ & $\mathbf{3.10}_{\pm 0.03}$ & $4.08_{\pm 0.06}$ & $.985_{\pm .012}$ \\
\midrule
\textsc{Creativity-enhanced} & $5.15_{\pm 0.16}$ & $17.15_{\pm 0.40}$ & $.503_{\pm .056}$ &  $3.19_{\pm 0.09}$ & $3.38_{\pm 0.07}$ & $.968_{\pm .011}$ \\
\textsc{DMAD} & $2.79_{\pm 0.22}$ & $13.90_{\pm 0.49}$ & $.372_{\pm .025}$ & $2.89_{\pm 0.07}$ & $4.17_{\pm 0.11}$ & $.977_{\pm .009}$ \\
\addlinespace
\textsc{Evolution + creativity} & $\mathbf{6.11}_{\pm 0.22}$ & $\mathbf{17.65}_{\pm 0.14}$ & $\mathbf{.580}_{\pm .046}$ & $\mathbf{3.39}_{\pm 0.07}$ & $3.38_{\pm 0.16}$ & $.979_{\pm .020}$ \\
\bottomrule
\end{tabular}%
}
\vspace{-2mm}
\end{table*}

\paragraph{RQ4: Do the task-agnostic gains transfer across tasks and compose with prompt engineering?}

Yes. Task-agnostic persona diversification transfers robustly across both open-ended and constrained task formats. On Infinity-Chat, \textsc{Dispersion} improves upon \textsc{Random} in all three diversity measures, while \textsc{Evolution} extends these gains further (\Cref{tab:crossbench_creativity_summary}). Relative to standard task prompt, \textsc{Evolution} reduces homogeneity ($0.946\to0.798$) and more than doubles flexibility ($1.39\to2.81$). Its between-persona separation gain also rises from $0.068$ of \textsc{Random} to $0.126$, indicating significantly more distinct responses across personas relative to variation within a single persona. Similarly, DAT exhibits the same ordering of escalation across the task-only \textsc{Baseline Task}, \textsc{Random}, \textsc{Dispersion} and \textsc{Evolution} conditions for both divergence and flexibility. On DAT, task-agnostic \textsc{Evolution} also exceeds \textsc{Creative} on these measures. Crucially, despite improved creativity, \textsc{Evolution} maintains near-perfect validity on both benchmarks ($99.2\%$ on Infinity-Chat and $99.6\%$ on DAT), proving that task adherence is not sacrificed.
Furthermore, these task-agnostic gains compose seamlessly with advanced reasoning scaffolds. Adding \textsc{Evolution} personas to \textsc{DMAD} further reduces Infinity-Chat homogeneity ($0.929\to0.900$) and raises both Infinity-Chat flexibility ($1.56\to1.81$) and DAT score ($90.62\to90.93$), while maintaining validity at $100\%$ and $99.7\%$, respectively, demonstrating the robust compositional benefits.

\begin{table*}[t]
\centering\small
\setlength{\tabcolsep}{3pt}
\caption{\textbf{Response creativity evaluation on Infinity-Chat (IC) and DAT.}
\emph{Homogeneity} is mean pairwise cosine similarity, \emph{Separation} is mean between-persona distance minus mean within-persona distance, \emph{Flexibility} is the Vendi score, measuring effective diversity under the embedding similarity kernel, and \emph{Divergence} is the DAT semantic score. Subscripts give 95\% confidence-interval half-widths across five random seeds.
\textbf{Bold} = best results within each group. Control (validity) and undefined combinations (--) are not bolded.}
\label{tab:crossbench_creativity_summary}
\vspace{-2mm}
\resizebox{\textwidth}{!}{%
\begin{tabular}{@{}lccccccc@{}}
\toprule
& \multicolumn{4}{c}{Infinity-Chat}
& \multicolumn{3}{c}{DAT} \\
\cmidrule(lr){2-5} \cmidrule(lr){6-8}
Variant & Homog. $\downarrow$ & Sep. $\uparrow$ & Flex. $\uparrow$ & Val. & Div. $\uparrow$ & Flex. $\uparrow$ & Val. \\
\midrule
\textsc{Baseline Task} & $0.946_{\pm 0.002}$ & -- & $1.39_{\pm 0.02}$ & $.950_{\pm .015}$ & $89.13_{\pm 0.49}$ & $1.79_{\pm 0.01}$ & $.996_{\pm .005}$ \\
\addlinespace
\textsc{Random} & $0.864_{\pm 0.005}$ & $0.068_{\pm 0.003}$ & $2.17_{\pm 0.05}$ & $1.000_{\pm .000}$ & $89.15_{\pm 1.24}$ & $1.83_{\pm 0.04}$ & $.993_{\pm .011}$ \\
\addlinespace
\textsc{Dispersion} (selection) & $0.840_{\pm 0.004}$ & $0.085_{\pm 0.006}$ & $2.40_{\pm 0.05}$ & $.986_{\pm .007}$ & $89.39_{\pm 0.26}$ & $1.84_{\pm 0.01}$ & $.995_{\pm .006}$ \\
\textsc{Evolution} (generation) & $\mathbf{0.798}_{\pm 0.006}$ & $\mathbf{0.126}_{\pm 0.006}$ & $\mathbf{2.81}_{\pm 0.08}$ & $.992_{\pm .005}$ & $\mathbf{89.48}_{\pm 0.32}$ & $\mathbf{1.85}_{\pm 0.01}$ & $.996_{\pm .004}$ \\
\midrule
\textsc{Creative} & -- & -- & -- & -- & $86.05_{\pm 0.43}$ & $1.69_{\pm 0.01}$ & $.977_{\pm .010}$ \\
\textsc{DMAD} & $0.929_{\pm 0.002}$ & -- & $1.56_{\pm 0.01}$ & $1.000_{\pm .000}$ & $90.62_{\pm 0.25}$ & $1.89_{\pm 0.01}$ & $.994_{\pm .006}$ \\
\addlinespace
\textsc{Evolution + DMAD} & $\mathbf{0.900}_{\pm 0.001}$ & $0.021_{\pm 0.003}$ & $\mathbf{1.81}_{\pm 0.01}$ & $1.000_{\pm .000}$ & $\mathbf{90.93}_{\pm 0.29}$ & $1.89_{\pm 0.01}$ & $.997_{\pm .008}$ \\
\bottomrule
\end{tabular}%
}
\vspace{-5mm}
\end{table*}

\vspace{-2mm}
\section{Conclusion}
\vspace{-2mm}
This work framed persona-diversity induction as a set-level design problem, distinguishing space-filling from frontier-seeking diversity and selection from generation. The resulting framework provides principled ways to control persona-set geometry and shows that representational diversity and ideational diversity are not the same objective: the persona set that best represents a population need not be the one that elicits the widest range of ideas. More broadly, persona geometry offers a task-agnostic mechanism for shaping the range of LLM outputs, complementary to prompt optimization and inference-time scaffolding, and points toward systems that can deliberately control not only the quality of an answer, but the diversity of perspectives from which answers are generated.

\subsection*{AI use statement}

In this work, we used generative AI tools to 
help develop theoretical models or conceptual frameworks, formulate mathematical claims, provide critical ingredients for proving mathematical claims, assist in the writing of proofs, design or provide feedback on research methodology or experiments, implement methods, support qualitative and thematic data analysis, and interpret results.
We have not used generative AI tools to
propose or refine hypotheses
and 
generate synthetic data sets, assist with translation, and clean and reformat dataset
are not applicable to this work.
Additionally, we used generative AI tools to 
create or modify scientific figures or images, create or edit software code, creation of artifacts, summarize or analyse existing literature, discover research topics or identify gaps, brainstorming, sourcing/searching for information, edit a research paper to improve readability, identify relevant literature, format references, and propose a title or keywords for a research paper.
We have reviewed all AI-assisted work. We reviewed and tested AI-generated experiment code to run as intended, and all conversations with AI are verified before applied to the paper.
We take responsibility for the final content of this work, including text, claims or artifacts produced with the aid of generative AI.

\subsection*{Reproducibility Statement}

Sections~\ref{sec:problem} and~\ref{sec:method} define the persona representations, diversification objectives, and algorithms. Appendices~\ref{app:selection} and~\ref{app:generation} provide detailed algorithmic procedures, the assumptions underlying the theoretical guarantees, and their proofs. Appendix~\ref{app:experiment_details} documents the benchmark data sources, baseline construction, generation and preprocessing protocols, evaluation metrics, annotation procedures, and LLM-judge validation. Appendix~\ref{app:results} provides the full experimental comparisons, with performance estimates and confidence intervals across five generation seeds. Extensive examples of representative personas and prompt templates for persona construction, response generation, and evaluation are provided in Appendices~\ref{app:persona} and~\ref{app:prompt}, respectively.

\bibliography{iclr2027_conference}
\bibliographystyle{iclr2027_conference}

\clearpage
\appendix
\etocdepthtag.toc{appendices}

\begingroup
\small
\hypersetup{hidelinks}
\etocsettagdepth{mainpaper}{none}
\etocsettagdepth{appendices}{subsection}
\etocsetnexttocdepth{subsection}
\etocsettocstyle{\section*{Appendix Contents}}{}
\tableofcontents
\endgroup

\clearpage

\section{Related Works}

\paragraph{Persona-conditioned and personalized language modeling.}
Role prompting is a widely used interface for steering large language models, but recent work suggests that a persona can be treated as a more structured object than a generic expert role. PersonaHub scales persona-conditioned data synthesis to a billion synthetic personas, showing that persona descriptions can serve as reusable carriers of perspective and task variation \citep{ge2024scaling}. In parallel, work on personalized response generation has emphasized that user profiles are dynamic, sparse, and context-dependent: PersonaMem benchmarks LLMs on maintaining and using evolving user profiles over long interaction histories \citep{jiang2025know}. These efforts establish personas as useful conditioning variables, but they primarily focus on sampling, profiling, or evaluating personalization. Other work cautions that LLM-generated personas may encode systematic biases or brittle assumptions, especially when used for downstream simulation or evaluation \citep{li2026llm}. In contrast, our work studies personas as \emph{optimizable} latent controls: rather than relying on generic roles or fixed persona pools, we extract, adapt, select, and contrast personas to improve the novelty and diversity of model generations.

\paragraph{Diversity-aware generation and optimization.}
A long line of work has studied diversity in generation and subset selection. Determinantal point processes provide a principled probabilistic model for selecting diverse subsets while balancing quality and diversity \citep{kulesza2012dpp}. Submodular facility-location objectives similarly formalize coverage and redundancy reduction, with classical greedy algorithms providing approximation guarantees for monotone submodular maximization under cardinality constraints \citep{nemhauser1978analysis}. In evolutionary computation, quality-diversity methods such as MAP-Elites search for collections of high-performing but behaviorally distinct solutions, rather than optimizing a single best point \citep{mouret2015mapelites}. Recent LLM work has brought related ideas into prompt and decoding optimization: Promptbreeder evolves task prompts through population-based self-improvement \citep{fernando2024promptbreeder}, Plug-and-Play Language Models steer generation through gradient-based control without retraining the base model \citep{dathathri2020pplm}, and Contrastive Search improves open-ended generation by discouraging degeneration while preserving coherence \citep{su2023contrastivesearch}. Our approach builds on these ideas but changes the optimization target: we optimize \emph{persona populations} rather than only prompts, logits, or samples, with the goal of producing useful creative diversity rather than only improving task accuracy or surface-level lexical variation.

\paragraph{Pluralism, coverage, and model homogeneity.}
The motivation for optimizing persona diversity is closely related to recent concerns about output-space collapse. Work on pluralistic alignment argues that many tasks admit a spectrum of valid responses, and that alignment should preserve calibrated diversity rather than collapse toward a single normative answer \citep{sorensen2024pluralism}. Spectrum Tuning further operationalizes this idea through distributional coverage and in-context steerability, showing that post-training can affect how well models cover diverse valid responses \citep{sorensen2025spectrum}. Related evidence on artificial hiveminds suggests that LLMs can exhibit substantial homogeneity on open-ended tasks, even when many distinct responses would be reasonable \citep{jiang2026artificial}. These findings motivate our central hypothesis: increasing structured diversity in the persona or prompt space can increase diversity in the response space, provided that the personas remain task-relevant and high quality.

\paragraph{Creativity benchmarks and metrics.}
Evaluating creativity in LLMs is challenging because novelty, usefulness, surprise, fluency, and diversity can diverge. Classic divergent-thinking paradigms such as the Alternative Uses Test have been adapted to evaluate whether LLMs can produce semantically distant but plausible uses for everyday objects \citep{stevenson2022putting}, while the Divergent Association Task probes whether models can generate remote semantic associations \citep{chen2023probing}. For creative writing, TTCW shows that LLM-generated stories often fall short of professional human writing under expert evaluation \citep{chakrabarty2024art}, and CS4 studies creativity under increasing numbers of story-writing constraints, highlighting trade-offs between originality, coherence, and instruction following \citep{atmakuru2024cs4}. Complementary metric work such as the Creativity Index quantifies linguistic novelty through attribution against web text \citep{lu2025ai}, while CreativityPrism evaluates creativity holistically across tasks, domains, and metrics \citep{hou2026creativityprism}. These benchmarks motivate our evaluation design: to show that persona-conditioned generation improves creativity, it is not sufficient to produce stylistically different outputs; the outputs must be measurably more novel, diverse, coherent, and useful across multiple creativity settings.

\section{Selection Algorithms and Proofs}
\label{app:selection}
\subsection{Adaptive Thresholded Coverage}
\label{app:coverage}
\subsubsection{Exact ILP formulation of the coverage objective}
As outlined in \Cref{sec:selection}, we formulate persona selection as a maximal-covering location problem rather than an additive maximum-similarity objective. While an additive model maximizes the summed similarity of every persona to its nearest exemplar \citep{wei2015submodularity}, its marginal gain weights every demand point's similarity improvement, allowing dense regions of the persona space to disproportionately dominate the objective. By contrast, under our thresholded coverage model, a candidate exemplar's marginal gain is strictly the number of previously uncovered demand personas inside its radius-$r$ neighborhood; already-covered points contribute zero additional gain. (Under cosine dissimilarity, this effectively thresholds the cosine similarity at $\tau_{\mathrm{cov}}=1-r$, yielding a $\tau_{\mathrm{cov}}$-coverage model).

Recall the coverage objective for a selected subset $\mathcal{S}$ of cardinality $k$ at a fixed radius $r$: $F_{\mathrm{cov},m}(\mathcal{S};\mathcal{P},r)$ (\eqref{eq:coverage}). We adopt the convention $\min\varnothing=+\infty$, ensuring that no persona is covered by the empty set and $F_{\mathrm{cov},m}(\varnothing;\mathcal{P},r)=0$.

At a fixed radius $r$, the maximum-coverage persona-selection problem seeks a subset $\mathcal S^\star_{\mathrm{cov},m}(\mathcal P;r)$ of cardinality $k$ that leaves as little of the candidate population uncovered as possible
\begin{equation}
    \mathcal S^\star_{\mathrm{cov},m}(\mathcal P;r)
    \in
    \arg\max_{\substack{\mathcal S\subseteq\mathcal P\\|\mathcal S|=k}} \dcov{m}(\mathcal S;\mathcal P,r).
    \label{eq:coverage-argmax}
\end{equation}

Because the inner $\min$ operator and the indicator function render this objective (\eqref{eq:coverage}) non-linear, we linearize it exactly. First, we define the coverage neighborhood of persona $p_i$, which always contains $i$ itself
\begin{equation}
    \mathcal J^{\mathrm{cov}}_m(i;r)
    =
    \bigl\{j\in[N_{\mathcal P}]:\dper{m}(p_i,p_j)\leq r\bigr\},
    \label{eq:coverage-neighborhood}
\end{equation}

We then introduce two sets of binary decision variables
\begin{itemize}
    \item $\chi_j\in\{0,1\}$: selection indicator for exemplar $p_j$;
    \item $z_i^{\mathrm{cov}}\in\{0,1\}$: covered indicator for candidate persona
          $p_i$.
\end{itemize}

Writing $\boldsymbol{\chi}=(\chi_1,\ldots,\chi_{N_{\mathcal{P}}})^\top$ and $\mathbf{z}^{\mathrm{cov}} =(z_1^{\mathrm{cov}},\ldots,z_{N_{\mathcal{P}}}^{\mathrm{cov}})^\top$, our goal is to select exactly $k$ exemplars to maximize the unnormalized covered count. Structuring the objective to be integer-valued allows us to pose the exact maximal covering location model as the following Integer Linear Program (ILP)
\begin{align}
    \underset{\boldsymbol\chi,\mathbf z^{\mathrm{cov}}}{\text{maximize}}
    \quad &
    \sum_{i=1}^{N_{\mathcal P}}z_i^{\mathrm{cov}}
    \label{eq:coverage-mip-objective}\\
    \text{subject to}\quad
    & z_i^{\mathrm{cov}}
      \leq\sum_{j\in\mathcal J^{\mathrm{cov}}_m(i;r)}\chi_j
      && \forall i\in[N_{\mathcal P}],
      \label{eq:coverage-mip-cover}\\
    & \sum_{j=1}^{N_{\mathcal P}} \chi_j=k,
      \label{eq:coverage-mip-cardinality}\\
    & \chi_j\in\{0,1\},\quad z_i^{\mathrm{cov}}\in\{0,1\}
      && \forall i,j\in[N_{\mathcal P}].
      \label{eq:coverage-mip-domains}
\end{align}
Dividing the optimum of this solver objective by $N_{\mathcal{P}}$ perfectly recovers $F_{\mathrm{cov},m}$ (\eqref{eq:coverage}).

\begin{remark}[Model Size]
This program requires only $2N_{\mathcal P}$ binary variables and $N_{\mathcal P}+1$ constraints. This is highly efficient compared to the $N_{\mathcal P}^2$ continuous assignment variables and $N_{\mathcal P}^2$ linking constraints required by the strong $k$-median formulation of the additive objective. Furthermore, the dissimilarities enter the coverage constraint (\eqref{eq:coverage-mip-cover}) only through the Boolean level sets $\mathcal{J}^{\mathrm{cov}}_m(i;r)$ (\eqref{eq:coverage-neighborhood}), meaning no similarity normalization, frozen reference scale, or integer rescaling of distances is needed to pose the model.
\end{remark}

\subsubsection{Formulation properties}
The thresholded coverage objective avoids the complex scale bookkeeping required by additive formulations. Its key structural properties stem from the fact that the optimization depends on the dissimilarity metric strictly through its Boolean level sets.

\begin{theorem}[Ordinal Invariance]
\label{thm:coverage-ordinal}
Let $g:\mathbb R_{\geq0}\rightarrow\mathbb R_{\geq0}$ be strictly increasing function with $g(0)=0$. Replacing the dissimilarity metric $\dper{m}$ with $g\circ\dper{m}$ and the radius $r$ by $g(r)$ leaves every coverage neighborhood $\mathcal{J}^{\mathrm{cov}}_m(i;r)$ (\eqref{eq:coverage-neighborhood}) unchanged. Consequently, the feasible set, the objective function, and the optimizers of the ILP formulation (\eqref{eq:coverage-mip-objective}--\eqref{eq:coverage-mip-domains}) remain strictly identical.
\end{theorem}
\begin{proof}
The program depends on the dissimilarity metric solely through the sets $\mathcal J^{\mathrm{cov}}_m(i;r)$. Because $g$ is strictly increasing, the inequality $\dper{m}(p_i,p_j)\leq r$ holds if and only if $g(\dper{m}(p_i,p_j))\leq g(r)$. Therefore, every level set, and the entire optimization program by extension, is perfectly preserved.
\end{proof}

In practice, this means that formulating the cosine variant as a similarity threshold $\tau_{\mathrm{cov}}$ or as a distance radius $r=1-\tau_{\mathrm{cov}}$ selects the exact same personas. No frozen reference scale is required to ensure the objective is non-negative or strictly comparable across different candidate populations.

\begin{theorem}[Integrality of Relaxed Coverage Indicators]
If only the covered indicators are relaxed to continuous bounds $z_i^{\mathrm{cov}}\in[0,1]$, then for any fixed, feasible binary selection vector $\boldsymbol\chi\in\{0,1\}^{N_{\mathcal P}}$, the relaxed program is maximized by
$$(z_i^{\mathrm{cov}})^\star = \min\Bigl\{1,\sum_{j\in\mathcal J^{\mathrm{cov}}_m(i;r)}\chi_j\Bigr\}\in\{0,1\}.$$
Thus, this relaxation is exact: integrality is inherently guaranteed for the covered-indicator block whenever $\boldsymbol\chi$ is binary. (Note: This does not assert integrality of the full LP relaxation where $\boldsymbol\chi$ is also relaxed.)
\end{theorem}

\begin{proof}
The objective is non-decreasing in each $z_i^{\mathrm{cov}}$, and distinct rows are coupled only through $\boldsymbol\chi$. Therefore, the solver independently raises each $z_i^{\mathrm{cov}}$ to the largest value permitted by the coverage constraint (\eqref{eq:coverage-mip-cover}) and the upper bound $z_i^{\mathrm{cov}}\leq1$. That value is the minimum of $1$ and a non-negative integer, which is inherently binary.
\end{proof}

\begin{remark}[Demand Aggregation]
Personas sharing identical coverage neighborhoods at radius $r$ can be merged into a single demand row with an integer weight without altering the optimum. While near-duplicate personas may share their neighborhoods and can thus be aggregated, spatial proximity alone does not guarantee identical neighborhood rows.
\end{remark}

\begin{lemma}[Monotone Submodularity]
\label{lem:coverage-submodular}
The unnormalized covered count $\mathcal S\mapsto N_{\mathcal P}\,\dcov{m}(\mathcal S;\mathcal P,r)$ is normalized, monotone, and submodular. Therefore, a standard greedy algorithm
attains at least a $(1-1/e)$ fraction of the optimal covered count \citep{nemhauser1978analysis}.
\end{lemma}
\begin{proof}
The covered count evaluates to $\bigl|\bigcup_{p_j\in\mathcal S}\{i\in[N_{\mathcal P}]:j\in\mathcal J^{\mathrm{cov}}_m(i;r)\}\bigr|$, which is the cardinality of a union of finite sets indexed by the selected exemplars. This constitutes a classical coverage function, and all coverage functions are normalized, monotone, and submodular.
\end{proof}
As previewed in \Cref{sec:selection}, we utilize Lemma~\ref{lem:coverage-submodular} not as an approximation guarantee (since our final selection algorithm below is exact), but as the mathematical engine for the cheap, one-sided greedy certificates that keep the exact solver out of most of the vast majority of the radius search space.

\subsubsection{Adaptive radius calibration}
As outlined in \Cref{sec:selection}, manually fixing the radius $r$ would reintroduce exactly the dissimilarity-specific scaling that Theorem~\ref{thm:coverage-ordinal} explicitly eliminates. We therefore calibrate it from a scale-free parameter: a target coverage level $c_{\mathrm{cov}}\in(0,1]$, mapping to an absolute covered-count target $n_{\mathrm{cov}}=\lceil c_{\mathrm{cov}}N_{\mathcal P}\rceil$. 

The calibrated radius is defined as the smallest possible radius at which some size-$k$ subset successfully meets the target
\begin{equation}
    r^\star_m(\mathcal P;c_{\mathrm{cov}}) =
    \min\Bigl\{r\geq0: \max_{\substack{\mathcal S\subseteq\mathcal P\\|\mathcal S|=k}} N_{\mathcal P}\,\dcov{m}(\mathcal S;\mathcal P,r) \geq n_{\mathrm{cov}}\Bigr\},
    \label{eq:adaptive-radius}
\end{equation}
In a metric space, this formulation specialized to the discrete generalized $k$-center-with-outliers objective \citep{charikar2001algorithms}, accommodating $N_{\mathcal P}-n_{\mathrm{cov}}$ outliers. (Note: Our exact finite search operates strictly on pairwise dissimilarities and does not assume a triangle inequality.) The final selection algorithm is lexicographic: it first discovers the minimal radius $r^\star_m(\mathcal P;c_{\mathrm{cov}})$ (\eqref{eq:adaptive-radius}), and then break ties among
radius-optimal subsets by maximizing the covered count (\eqref{eq:coverage-mip-objective}) at that exact radius. We denote this calibrated selector by
\begin{equation}
    \mathcal S^\star_{\mathrm{acov},m}(\mathcal P;c_{\mathrm{cov}})
    \in
    \arg\max_{\substack{\mathcal S\subseteq\mathcal P\\|\mathcal S|=k}} \dcov{m}\!\left(\mathcal S;\mathcal P, r^\star_m(\mathcal P;c_{\mathrm{cov}})\right).
    \label{eq:adaptive-coverage-selector}
\end{equation}

\begin{lemma}[Finite Candidate Radii]
\label{lem:candidate-radii}
For the finite pairwise spectrum $\Lambda_m(\mathcal P)$ defined in \eqref{eq:persona-pairwise-spectrum}, every fixed subset $\mathcal S$ has a covered count that is non-decreasing in $r$ and constant between consecutive values of $\Lambda_m(\mathcal P)$. Consequently, the same holds for its maximum over all size-$k$ subsets. The feasible set of the adaptive radius objective (\eqref{eq:adaptive-radius}) is therefore the half-line $[r^\star_m(\mathcal P;c_{\mathrm{cov}}),\infty)$, and the exact minimum $r^\star_m(\mathcal P;c_{\mathrm{cov}})
\in\Lambda_m(\mathcal P)$.
\end{lemma}
\begin{proof}
As $r$ grows, the indicator function in the coverage objective (\eqref{eq:coverage}) can change only when $r$ crosses a realized dissimilarity in the matrix, as the coverage events
$\{\min_{p_j\in\mathcal S}\dper{m}(p_i,p_j)\leq r\}$ are nested in $r$. Hence, every covered count is a non-decreasing step function whose jumps lie exactly in $\Lambda_m(\mathcal P)$. Since the pointwise maximum over the finitely many size-$k$ subsets is also a step function with jumps in $\Lambda_m(\mathcal P)$, the smallest feasible radius is inherently attained at a jump point.
\end{proof}

By Lemma~\ref{lem:candidate-radii}, calibration reduces to an exact discrete search over the sorted distinct values of $\Lambda_m(\mathcal P)$ \citep{hochbaum1986unified}. To avoid invoking the computationally heavy ILP solver for every probe, we evaluate the monotone decision problem of whether $k$ exemplars can cover at least $n_{\mathrm{cov}}$ personas at radius $r$, using a standard size-$k$ greedy maximum-coverage construction. Let $G_m(r)$ denote the covered count returned by this greedy heuristic. Three certificate mechanisms resolve the vast majority of search probes instantly:
\begin{enumerate}[noitemsep, leftmargin=*, topsep=0pt, partopsep=0pt]
    \item \emph{Greedy feasibility certificate.}  If $G_m(r)\geq n_{\mathrm{cov}}$, the probed radius is immediately certified as feasible.
    \item \emph{Greedy infeasibility certificate.}  By the monotone submodularity established in Lemma~\ref{lem:coverage-submodular}, the optimal covered count is strictly bounded above by $\lfloor G_m(r)/(1-1/e)\rfloor$. If this proven upper bound falls below $n_{\mathrm{cov}}$, the probed radius is mathematically certified as infeasible.
    \item \emph{Certificate tightening.}  Any feasible selection $\mathcal S$ remains feasible down to its own minimal radius (the $n_{\mathrm{cov}}$-th smallest nearest-exemplar dissimilarity under $\mathcal S$). This allows the feasible bracket to jump far below the initially probed radius in a single computational step.
\end{enumerate}

Probes falling into the narrow residual window are resolved via bounded optimization. The exact ILP solver attempts to maximize the covered count but is instructed to terminate early the moment an incumbent reaches $n_{\mathrm{cov}}$ (a feasibility certificate) or its proven upper bound drops below $n_{\mathrm{cov}}$ (an infeasibility certificate that remains mathematically valid even if cut off by a time limit, because the objective is strictly integer-valued). By monotonicity, the search terminates as soon as it certifies that the immediate candidate predecessor of $r^\star_m(\mathcal P;c_{\mathrm{cov}})$ is infeasible, whenever such a predecessor exists. That final certificate may come from the greedy upper bound or the exact solver; no lower-candidate certificate is needed when the optimum is $r_{m,(1)}$, the absolute minimum dissimilarity in the spectrum.

\subsubsection{Exact algorithm}
Maximum coverage is NP-hard \citep{feige1998threshold} when the selection budget is part of the input, and exact optimization can be computationally expensive in the worst-case. However, our compact ILP formulation (\eqref{eq:coverage-mip-objective}--\eqref{eq:coverage-mip-domains}) is highly tractable in practice. To resolve the exact probes during discrete search, the formulation is warm-started with the greedy incumbent and evaluated by an exact integer or constraint solver (e.g., CP-SAT). 

We accept a final solution as exact only when it achieves certified optimal status. Because the unnormalized covered count is integral, an incumbent solution is automatically certified as optimal once its proven mathematical upper bound drops less than one unit above it.

\begin{algorithm}[htbp]
\caption{Exact Adaptive Thresholded-Coverage Selection}
\label{alg:exact_coverage}
\begin{algorithmic}[1]
    \State \textbf{Input:} Candidate population $\mathcal P$ of size $N_{\mathcal P}$, dissimilarity index $m$, target size $k$, target coverage $c_{\mathrm{cov}}$
    \State Embed each $p_i\in\mathcal P$
        \Comment{$\mathcal{O}(N_{\mathcal P})$ encoder calls}
    \State Compute $D^{\mathrm P}_{m,ij}\gets\dper{m}(p_i,p_j)$ and $\mathbf D_m^{\mathrm P}\gets(D^{\mathrm P}_{m,ij})_{i,j=1}^{N_{\mathcal P}}$
        \Comment{$\mathcal{O}(N_{\mathcal P}^2)$ pair evaluations}
    \State $n_{\mathrm{cov}}\gets\lceil c_{\mathrm{cov}}N_{\mathcal P}\rceil$
        \Comment{$\mathcal{O}(1)$}
    \State Form $\Lambda_m(\mathcal P)$ and sort it as $r_{m,(1)}<\cdots<r_{m,(L_m)}$
        \Comment{naively $\mathcal{O}(N_{\mathcal P}^2\log N_{\mathcal P})$}
    \State $\ell_{\mathrm{hi}}\gets L_m$ with any size-$k$ witness feasible at $r_{m,(L_m)}$; $\ell_{\mathrm{lo}}\gets0$
        \Comment{$0$ is a sentinel, never a radius index}
    \State Apply greedy certificates and witness tightening to the bracket, preserving feasible $\ell_{\mathrm{hi}}$ and either $\ell_{\mathrm{lo}}=0$ or certified-infeasible $\ell_{\mathrm{lo}}$
    \While{$\ell_{\mathrm{lo}}<\ell_{\mathrm{hi}}-1$}
        \State Probe $\ell\gets\ell_{\mathrm{hi}}-1$ at radius $r_{m,(\ell)}$: greedy certificates, else the bounded-optimization oracle
        \If{feasible, with witness cover $\mathcal S$}
            \State $\ell_{\mathrm{hi}}\gets$ index of the minimal radius of $\mathcal S$
                \Comment{certificate tightening}
        \Else
            \State $\ell_{\mathrm{lo}}\gets\ell$
                \Comment{certified infeasible}
        \EndIf
    \EndWhile
    \State $r^\star\gets r_{m,(\ell_{\mathrm{hi}})}$
    \State Solve \eqref{eq:coverage-mip-objective}--\eqref{eq:coverage-mip-domains} at $r=r^\star$ to certified optimality, warm-started with the incumbent cover
    \State \textbf{Output:} Globally optimal coverage subset $\mathcal S^\star_{\mathrm{acov},m}(\mathcal P;c_{\mathrm{cov}})$ and calibrated radius $r^\star_m(\mathcal P;c_{\mathrm{cov}})=r^\star$
\end{algorithmic}
\end{algorithm}

\begin{remark}[Probe Schedule and Time Limits]
Any probe schedule that terminates with the tight bracket $\ell_{\mathrm{lo}}=\ell_{\mathrm{hi}}-1$ mathematically certifies the same minimal radius. 
Alternative schedules can change both the number and difficulty of radius decisions. For example, sequential probing can require $O(L_m)$ decisions, whereas binary search requires $O(\log L_m)$ certified decisions.
If a bounded-optimization probe is cut off by an imposed solver time limit and is therefore conservatively treated as infeasible, the returned selection is still a valid optimal cover at the returned radius. However, that radius acts only as an upper bound on the true minimal radius $r^\star_m(\mathcal P;c_{\mathrm{cov}})$. The exact optimality guarantee established in \Cref{thm:exact_coverage} applies only when every decisive infeasibility result is formally certified.
\end{remark}

\subsubsection{Optimality guarantee}
\begin{proof}
The proof of global optimality relies on three sequential guarantees:

\emph{Termination.} The discrete distance spectrum $\Lambda_m(\mathcal P)$ is finite. Since every iteration of the while-loop in \Cref{alg:exact_coverage} must either strictly lower the upper bound $\ell_{\mathrm{hi}}$ by at least one index or raise the lower bound $\ell_{\mathrm{lo}}$ to $\ell_{\mathrm{hi}}-1$, the search space strictly shrinks at every step. Therefore, the algorithm is guaranteed to terminate in finite time.

\emph{Radius optimality.}  Two invariants are maintained throughout the execution of the search loop. First, the radius $r_{m,(\ell_{\mathrm{hi}})}$ is always feasible, guaranteed by a stored size-$k$ witness cover. (Certificate tightening preserves this invariant because any cover is naturally feasible down to its own internal minimal radius by construction.) Second, either $\ell_{\mathrm{lo}}=0$ (acting as the initialized sentinel) or $r_{m,(\ell_{\mathrm{lo}})}$ is certified infeasible. This refutation is valid whether it comes from the greedy certificate (validated by the submodularity established in Lemma~\ref{lem:coverage-submodular}) or from the exact solver (whose proven bound acts as a strict upper bound on the true integral optimum.)

Upon termination, the bracket perfectly tightens. Either $\ell_{\mathrm{hi}}=1$, meaning the smallest candidate radius is in the spectrum is feasible, or $\ell_{\mathrm{lo}}=\ell_{\mathrm{hi}}-1\geq1$, meaning the immediate predecessor to the upper bound is infeasible. By monotonicity and the discrete step-function property established in Lemma~\ref{lem:candidate-radii}, this guarantees that the returned radius is the global minimum
$$r^\star=r_{m,(\ell_{\mathrm{hi}})}=r^\star_m(\mathcal P;c_{\mathrm{cov}}).$$

\emph{Coverage optimality at $r^\star$.}  The final step of the algorithm executes a bounded optimization over the finite set of selection vectors $\boldsymbol\chi\in\{0,1\}^{N_{\mathcal P}}$ subject to the exact cardinality constraint $\sum_j\chi_j=k$. The ILP solver evaluates this space and yields a feasible incumbent cover alongside a proven upper bound on the maximum possible covered count. Achieving a certified optimal status guarantees that these two values are equal. Furthermore, because the unnormalized covered count is integer-valued, proving an absolute gap of less than one is sufficient to rule out the existence of a better feasible selection.
\end{proof}

\subsection{Max--Min Dispersion Persona Selection}
\label{app:dispersion}
\subsubsection{Structural properties and thresholded equivalence}
Unlike greedy heuristics, the exact incremental-clique method relies on the discrete topology of the solution space rather than the continous geometric properties of the embedding space. 

\begin{assumption}[Pairwise dissimilarity]
The chosen dissimilarity function $\dper{m}(\cdot,\cdot)$ is symmetric and non-negative. However, a triangle inequality is not required for the exact clique reduction to hold.
\end{assumption}

The theoretical guarantee of this exact search is grounded in the finite, discrete nature of the candidate pool:
\begin{lemma}[Bottleneck Distance, \citep{erkut1990discrete}]
Because the candidate population $\mathcal P$ is finite, the optimal value $\ddisp{m}(\mathcal S^\star_{\mathrm{disp},m}(\mathcal P))$ equals the weight of at least one edge in the complete pairwise-dissimilarity graph on constructed on $\mathcal P$.
\label{lem:bottleneck}
\end{lemma}

\subsubsection{Exact incremental-clique algorithm}
\begin{algorithm}[htbp]
\caption{Exact Incremental-Clique Max--Min Dispersion}
\label{alg:exact_max_min_dispersion}
\begin{algorithmic}[1]
    \State \textbf{Input:} Candidate population $\mathcal P$, target size $k$, persona dissimilarity $\dper{m}$
    \State Compute the distinct pairwise-dissimilarity levels \(\Lambda_m(\mathcal P)\) and order them as \(r_{m,(1)}<\cdots<r_{m,(L_m)}\).
    \State Initialize $\mathcal G=(\mathcal P,\mathcal E)$ with $\mathcal E\gets\emptyset$
    \For{$\ell=L_m,L_m-1,\ldots,1$}
        \State $\mathcal E_\ell\gets \{\{p_a,p_b\}:a<b,\ \dper{m}(p_a,p_b)=r_{m,(\ell)}\}$
        \For{each edge $\{p_a,p_b\}\in\mathcal E_\ell$}
            \State $\mathcal E\gets \mathcal E\cup\{\{p_a,p_b\}\}$
            \State $\mathcal N_{a,b}\gets \{p_c\in\mathcal P: \{p_a,p_c\}\in\mathcal E \land\{p_b,p_c\}\in\mathcal E\}$
            \If{$|\mathcal N_{a,b}|\geq k-2$}
                \State Test whether $\mathcal G[\mathcal N_{a,b}]$ contains a $(k-2)$-clique
                \If{a $(k-2)$-clique $\mathcal Q$ exists}
                    \State \textbf{Return}
                        $\mathcal S^\star_{\mathrm{disp},m}(\mathcal P) \gets\mathcal Q\cup\{p_a,p_b\}$
                \EndIf
            \EndIf
        \EndFor
    \EndFor
\end{algorithmic}
\end{algorithm}

\subsubsection{Optimality guarantee}
\begin{proof}
Let $\delta^\star$ denote the optimal max--min dispersion. By Lemma~\ref{lem:bottleneck} (the bottleneck lemma), $\delta^\star=r_{m,(\ell^\star)}$ for some threshold in the sorted distance spectrum.  After all edges at threshold $r_{m,(\ell)}$ have been inserted, \Cref{alg:exact_max_min_dispersion} has effectively constructed the threshold graph
$$\mathcal G_\ell = \bigl(\mathcal P, \{\{p,p'\}:p\neq p',\ \dper{m}(p,p')\geq r_{m,(\ell)}\}\bigr).$$
A size-$k$ subset has a minimum internal dispersion of at least $r_{m,(\ell)}$ if and only if it forms a full $k$-clique in $\mathcal G_\ell$. The algorithm processes thresholds from largest to smallest and returns as soon as the first $k$-clique is completed. If a subset with a strictly larger dispersion existed, its corresponding clique would have been fully connected at an earlier (larger) threshold and the algorithm would have already terminated. Hence the returned clique inherently possesses the bottleneck distance $r_{m,(\ell^\star)}=\delta^\star$, proving it is globally optimal.
\end{proof}

\begin{remark}[Computational Efficiency via Sparsity]
Although finding a size-$k$ clique is NP-hard in general, the threshold graphs evaluated during the earliest stages of the algorithm (the largest distances) are sparse. For our experiment setting of $k=5$, each local feasibility test reduces to triangle detection in the mutual neighborhood $\mathcal N_{a,b}$. A direct implementation has a coarse polynomial bound of $O(N_{\mathcal{P}}^5)$, while this sparsity reduce actual search costs.
\end{remark}

\section{Generation Algorithms and Proofs}
\label{app:generation}
\subsection{Uniform-Coverage MCMC Persona Generation}
\label{app:mcmc}

We propose \emph{Uniform-Coverage MCMC} (UC-MCMC), a persona-generation sampler whose operational target assigns equal probability mass to the reachable cells of a fixed, finite-resolution partition of semantic directions. Within each cell, a frozen language model supplies a reference base law favoring linguistically plausible personas; valid proposals that land in previously unoccupied cells act to expand the active target in hindsight. The construction successfully pursues coverage through its sampling law, rather than by first generating a pool and then retrospectively equalizing its cell counts.

\subsubsection{Semantic Geometry and Target Distributions}
\paragraph{State Space and Validity.}
We specialize $\Omega$ to the countable persona state space of canonically serialized, EOS-terminated persona token sequences up to a fixed maximum length. A nonterminating generation is represented by a failure symbol $\bot\notin\Omega$. Let $\operatorname{val}:\Omega\rightarrow\{0,1\}$ be a fixed validity function. This function may incorporate a cached LLM judgment, provided that the model, prompt, decoding rule, and cache are frozen so that $\operatorname{val}(p)$ is deterministic. We extend this function with the convention $\operatorname{val}(\bot)=0$.

\paragraph{Semantic Geometry and Cell Partitions.}
Let $\widetilde\phi$ be the MRL-truncated text embedding, whose restriction to the persona state space maps $\Omega$ to $\mathbb R^d_{\mathrm M}$. We freeze a center $\widehat\mu$ and a nonsingular whitening map $W$ (computed using only the baseline population or a separate pilot sample). The semantic direction of a persona is defined as its normalized projection
\begin{equation}
    \mathbf z_W(p) = \frac{W\{\widetilde\phi(p)-\widehat\mu\}}{\left\|W\{\widetilde\phi(p)-\widehat\mu\}\right\|_2} \in\mathbb S^{d_{\mathrm M}-1}.
    \label{eq:ucmcmc-direction}
\end{equation}
We assume the denominator is nonzero for every valid persona. We then partition the sphere $\mathbb S^{d_{\mathrm M}-1}$ into $M$ measurable, equal-area cells $\mathcal C_1,\ldots,\mathcal C_M$, with deterministic tie-breaking on cell boundaries. (Equal-area, small-diameter sphere partitions can be constructed using, for example, the recursive zonal method of \citet{leopardi2006sphere}.) We define the cell-assignment function $c(p)=j$ when $\mathbf z_W(p)\in\mathcal C_j$, allowing use to formally define the valid persona set within cell $j$ as
\begin{equation}
    \Omega_j = \{p\in\Omega:\operatorname{val}(p)=1,\ c(p)=j\}.
    \label{eq:ucmcmc-cell-state}
\end{equation}
Note that equal area is required for the geometric interpretation formalized below, but the underlying MCMC validity results require only a fixed measurable partition.

\paragraph{Base Law and Cell-Conditioned Targets.}
Let $G_{\theta_{\mathrm P}}$ be the frozen persona-generation model and let $x_{\mathrm{gen}}$ be its fully serialized, fixed prompt. For persona $p$, let $w_{1:L(p)}$ be its exact output token sequence, excluding terminal EOS. We define the base probability mass of $p$ as
\begin{equation}
    b(p) = \left\{\prod_{\ell=1}^{L(p)} G_{\theta_{\mathrm P}}(w_\ell\mid x_{\mathrm{gen}},w_{<\ell})\right\} G_{\theta_{\mathrm P}}(\mathrm{EOS}\mid x_{\mathrm{gen}},w_{\leq L(p)}).
    \label{eq:ucmcmc-base}
\end{equation}
Every generation outcome outside $\Omega$ is mapped to the failure symbol $\bot$ without resampling. Thus, we set
$$b(\bot)=1-\sum_{p\in\Omega}b(p)$$
as the probability law on $\Omega_{\bot}:=\Omega\cup\{\bot\}$, while $\bot$ carries zero target weight because $\operatorname{val}(\bot)=0$. We assume $b(p)>0$ for every valid persona, which naturally holds for raw softmax sampling when all canonical output tokens remain in support.

In our primary configuration, the global independence proposal is exactly this completion law
\begin{equation}
    q_0(p'\mid p)=q_0(p')=b(p').
    \label{eq:ucmcmc-global-proposal}
\end{equation}

The base mass of cell $j$, accounting for validity, is therefore
\begin{equation}
    B_j = \sum_{p\in\Omega} b(p)\operatorname{val}(p)\mathbbm{1}\{c(p)=j\}, \qquad \mathcal R=\{j:B_j>0\},
    \label{eq:ucmcmc-cell-mass}
\end{equation}
where $\mathcal R$ is defined as the set of reachable cells. For any $j\in\mathcal R$, the within-cell target distribution is defined as
\begin{equation}
    \pi_j(p) = \frac{b(p)\operatorname{val}(p)\mathbbm{1}\{c(p)=j\}}{B_j}.
    \label{eq:ucmcmc-cell-target}
\end{equation}
For an active-cell index set $\mathcal I\subseteq\mathcal R$, the target distribution of a single emitted persona is the equal-cell mixture
\begin{equation}
    \Pi_{\mathcal I} = \frac{1}{|\mathcal I|} \sum_{j\in\mathcal I}\pi_j, \qquad \Pi_{\mathcal I}(\Omega_j)=\frac{1}{|\mathcal I|} \quad (j\in\mathcal I).
    \label{eq:ucmcmc-emission-target}
\end{equation}
In contrast, the joint target distribution for the full population containing exactly one chain per cell is the product measure
\begin{equation}
    \Gamma_{\mathcal I} = \bigotimes_{j\in\mathcal I}\pi_j.
    \label{eq:ucmcmc-product-target}
\end{equation}

\begin{remark}[Distributional Distinctions]
The distinction between $\pi_j$, $\Pi_{\mathcal I}$, and $\Gamma_{\mathcal I}$ is critical: $\pi_j$ is a localized within-cell persona law, $\Pi_{\mathcal I}$ is the desired
marginal law of a single randomly emitted persona, and $\Gamma_{\mathcal I}$ is the stationary joint law of the entire population of parallel cell chains. The unknown masses $B_j$ never need to be estimated, as they cancel from all within-cell Metropolis--Hastings ratios, while the parallel cell chains and the uniform scheduler inherently supply the equal mixture weights required by \eqref{eq:ucmcmc-emission-target}.
\end{remark}

\paragraph{Resolution-Limited Directional Uniformity.}
The following lemma formalizes the finite-resolution connection to true directional uniformity across the continuous semantic sphere.

\begin{lemma}[Wasserstein Bound on Uniformity]
\label{lem:ucmcmc-wasserstein}
Let $\sigma_{\mathcal I}$ be the normalized spherical surface measure on $\mathcal C_{\mathcal I}=\bigcup_{j\in\mathcal I}\mathcal C_j$, let $\nu_{\mathcal I}=(\mathbf z_W)_{\#}\Pi_{\mathcal I}$ be the pushforward of $\Pi_{\mathcal I}$ to the sphere, and let $\Delta_j$ be the geodesic diameter of cell $\mathcal C_j$. For every Wasserstein order $q_{\mathrm W}\geq1$
\begin{equation}
    W_{q_{\mathrm W}}(\nu_{\mathcal I},\sigma_{\mathcal I}) \leq \left(\frac{1}{|\mathcal I|} \sum_{j\in\mathcal I}\Delta_j^{q_{\mathrm W}}\right)^{1/q_{\mathrm W}} \leq \max_{j\in\mathcal I}\Delta_j .
    \label{eq:ucmcmc-wasserstein}
\end{equation}
\end{lemma}
\begin{proof}
Since both measures assign a probability mass $1/|\mathcal I|$ to every individual cell $\mathcal C_j$, we can couple their conditional distributions separately within each respective cell. The physical geodesic distance between any two points coupled inside $\mathcal C_j$ is naturally bounded by its diameter $\Delta_j$.
\end{proof}
Lemma~\ref{lem:ucmcmc-wasserstein} connects equal cell masses to directional uniformity when cell diameters are small. Our implementation uses ten orthonormal sign cuts in 128 dimensions, to the shown bound does not provide a nontrivial approximation guarantee to uniform surface measure. Rather, the operational guarantee is equal allocation across reachable cells.

\subsubsection{Within-Cell Metropolis Kernels and Proposal Requirements}
\paragraph{Component Kernels and Acceptance Probabilities}
In addition to the global independence proposal $q_0$, let $q_1,\ldots,q_{L_q}$ be a set of fixed, exactly evaluable local-edit or block proposals. At each step, the algorithm first draws a proposal type $\ell$ with a fixed probability $\omega_\ell>0$, such that $\sum_{\ell=0}^{L_q}\omega_\ell=1$, and then applies a component-specific Metropolis--Hastings (MH) correction.

For a chain assigned to an active cell $j$, the acceptance probability for transitioning from an incumbent persona $p$ to a proposed persona $p'$ is
\begin{equation}
    \alpha_{j,\ell}(p,p') = \mathbbm{1}\{p'\in\Omega_j\} \min\left\{1,\,\frac{b(p')q_\ell(p\mid p')}{b(p)q_\ell(p'\mid p)}\right\}.
    \label{eq:ucmcmc-acceptance}
\end{equation}
A zero reverse probability implies a zero acceptance probability. Importantly, if the proposal type or edited block is selected with a state-dependent probability, that probability must
be explicitly included in the forward--reverse ratio. Equation \eqref{eq:ucmcmc-acceptance} describes a random mixture of separately corrected kernels; it does not represent the acceptance ratio for a single marginalized mixture proposal.

Let $K_{j,\ell}$ denote the resulting accept--reject kernel for proposal $\ell$, which includes the self-transition upon rejection, and define the full mixture kernel as
\begin{equation}
    K_j=\sum_{\ell=0}^{L_q}\omega_\ell K_{j,\ell}.
    \label{eq:ucmcmc-mixture-kernel}
\end{equation}

\begin{theorem}[Cell-Kernel Correctness and Matched Refresh]
\label{thm:ucmcmc-kernel}
For every reachable cell $j$, each component kernel $K_{j,\ell}$ is reversible with respect to $\pi_j$; hence the mixture $K_j$ leaves $\pi_j$ invariant. Moreover, under the matched global proposal ($q_0=b$), for every $p\in\Omega_j$, proposal value $p'\in\Omega_{\bot}$, and measurable subset $E\subseteq\Omega_j$
\begin{align}
    \alpha_{j,0}(p,p') &= \mathbbm{1}\{p'\in\Omega_j\},    \label{eq:ucmcmc-global-acceptance}\\
    K_{j,0}(p,E) &= B_j\pi_j(E)+(1-B_j)\delta_p(E),    \label{eq:ucmcmc-refresh-kernel}\\
    K_j(p,E) &\geq \omega_0B_j\pi_j(E).    \label{eq:ucmcmc-minorization}
\end{align}
where $\delta_p$ denotes the Dirac probability measure at $p$. Consequently, for every arbitrary initialization law $\zeta$ supported on $\Omega_j$ and every integer $n\geq0$
\begin{equation}
    \left\| \zeta K_j^n-\pi_j \right\|_{\mathrm{TV}} \leq (1-\omega_0B_j)^n .
    \label{eq:ucmcmc-cell-tv}
\end{equation}
\end{theorem}
\begin{proof}
For any two valid same-cell states $p,p'\in\Omega_j$, the two directional probability flows between them under kernel $\ell$ equal
$$\min\{ \pi_j(p)q_\ell(p'\mid p), \pi_j(p')q_\ell(p\mid p') \}.$$
Since this flow is perfectly symmetric, detailed balance is satisfied by the standard Metropolis--Hastings argument \citep{hastings1970,tierney1994}. 

When evaluating the global proposal ($q_0=b$), the proposal probabilities and base-weight terms cancel out for any valid same-cell move. A global proposal therefore hits $\Omega_j$ with total probability $B_j$ and, conditional on successfully doing so, is distributed according to law $\pi_j$, yielding the independent refresh kernel $K_{j,0}$ (\eqref{eq:ucmcmc-refresh-kernel}). Because $K_{j,0}$ is drawn with probability $\omega_0$, the fixed mixture satisfies the minorization condition (\eqref{eq:ucmcmc-minorization}). Finally, the resulting total-variation convergence bound follows from the classical Doeblin coupling argument \citep{mengersen1996rates,rosenthal1995minorization}.
\end{proof}

\begin{remark}[Global Cancellation and Strict Density Requirements]
The cancellation that occurs for global proposals is mathematically beneficial, not degenerate: a same-cell hit results in an exact, independent refresh from the target $\pi_j$. (The overall global acceptance rate may still appear small, because invalid and foreign-cell proposals remain as source-chain self-transitions.) 

However, for this exactness to hold, the forward and reverse proposal probabilities must describe the actual sampling law, including the EOS token and any block-selection probabilities. Therefore, the primary LLM implementation must sample raw categorical probabilities without top-$k_{\mathrm{dec}}$, top-$p_{\mathrm{dec}}$, retry-until-valid, or unaccounted many-to-one text normalization. If perfectly reversible local-edit blocks cannot be specified for free-form prose, the local component should be implemented exclusively as full-persona global MH proposal, ensuring the population-level random scan in \eqref{eq:ucmcmc-population-kernel} remains valid.
\end{remark}

\subsubsection{The Hindsight-Spawning Population Algorithm.}
Let $\mathcal P_0$ be the baseline persona population. We initialize the active-cell index set as the set of valid cells discovered in the baseline
$$\mathcal I_0 = \{c(p):p\in\mathcal P_0,\ \operatorname{val}(p)=1\}.$$
For each $j\in\mathcal I_0$, we choose one arbitrary valid seed $P^{(j)}\in\Omega_j$, assuming its base weight $b(P^{(j)})>0$. This seed may instead be drawn uniformly from the baseline
personas residing in that cell; its specific initialization law does not affect the anytime cell-level guarantees established below.

\begin{algorithm}[t]
\caption{Uniform-Coverage MCMC Persona Generation}
\label{alg:ucmcmc}
\begin{algorithmic}[1]
\Require Baseline population $\mathcal P_0$; frozen evaluation components $(\operatorname{val},\mathbf z_W,c,b)$; proposals $\{q_\ell\}_{\ell=0}^{L_q}$ and their probabilities $\{\omega_\ell\}_{\ell=0}^{L_q}$; total iterations $T_{\mathrm M}$
\State Initialize $\mathcal I_0$ and one starting state $P^{(j)}\in\Omega_j$ for every $j\in\mathcal I_0$
\For{$t=1,\ldots,T_{\mathrm M}$}
    \State Draw a target cell $J_t\mid\mathcal F_{t-1} \sim\operatorname{Uniform}(\mathcal I_{t-1})$
    \State Draw a proposal type $H_t\sim\operatorname{Categorical}(\omega_0,\ldots,\omega_{L_q})$
    \State Set $p\gets P^{(J_t)}$ and draw a candidate $P'_t\sim q_{H_t}(\cdot\mid p)$
    \State Set $\mathcal I_t\gets\mathcal I_{t-1}$ and leave the source chain $P^{(J_t)}$ unchanged by default
    \If{$\operatorname{val}(P'_t)=1$}
        \State $j'\gets c(P'_t)$
        \If{$j'=J_t$}
            \State Draw $U_t\sim\operatorname{Uniform}(0,1)$
            \If{$U_t\leq\alpha_{J_t,H_t}(p,P'_t)$}
                \State $P^{(J_t)}\gets P'_t$
            \EndIf
        \ElsIf{$j'\notin\mathcal I_{t-1}$}
            \State Spawn a new chain $P^{(j')}\gets P'_t$ and expand the active set $\mathcal I_t\gets\mathcal I_{t-1}\cup\{j'\}$
                \Comment{hindsight seed; not an accepted source move}
        \EndIf
    \EndIf
    \State Emit the post-transition source state $P_t^{\mathrm{emit}}\gets P^{(J_t)}$
\EndFor
\Ensure MCMC trace $(P_1^{\mathrm{emit}},\ldots,P_{T_{\mathrm M}}^{\mathrm{emit}})$, final active population, and a separately labeled discovery archive
\end{algorithmic}
\end{algorithm}

\paragraph{Hindsight Mechanics and Rejection.}
A valid foreign-cell proposal is automatically rejected by the source chain because its source-cell target probability is zero. If its destination cell was previously inactive, however, it serves as a valid initialization for a new chain. This seed is not emitted immediately, nor is it recorded as an accepted MCMC transition for the source chain. A global hindsight seed entering a new cell $j'$ is distributed exactly according to $\pi_{j'}$; a seed produced by a local-edit proposal, however, may have an arbitrary initialization law $\zeta_{j'}$ supported on $\Omega_{j'}$.

Crucially, a proposal entering an already active foreign cell must not be submitted directly to the destination cell's MH test, because the proposal was generated conditional on the source-chain's state rather than the destination-chain's state. Such cross-cell discoveries are instead retained in the separately labeled discovery archive. Finally, every scheduled proposal attempt advances the selected chain by one local MCMC iteration, even upon rejection. After a rejection, the emitted sample is the repeated current state, not the rejected candidate. (Discarding repeats or running until a fixed number of acceptances would improperly produce an accepted-state jump chain, violating the target measure.

\paragraph{Fixed-Active-Set MCMC Interpretation.}
For a temporarily fixed active-cell set $\mathcal I$ and a joint population state $\mathbf p=(p_j)_{j\in\mathcal I}$, we can define the full random-scan kernel as
\begin{equation}
    \mathsf K_{\mathcal I}(\mathbf p,d\mathbf p') = \frac{1}{|\mathcal I|} \sum_{j\in\mathcal I} K_j(p_j,dp'_j) \delta_{\mathbf p_{-j}}(d\mathbf p'_{-j}).
    \label{eq:ucmcmc-population-kernel}
\end{equation}
Because each localized $K_j$ preserves $\pi_j$, the global kernel $\mathsf K_{\mathcal I}$ preserves the product target $\Gamma_{\mathcal I}$ in \eqref{eq:ucmcmc-product-target}. This maps to a standard random-scan, component-wise MCMC construction \citep{johnson2013componentwise}: selecting a cell acts as the Gibbs coordinate step, and $K_j$ acts as the within-coordinate MH update. At stationarity, emitting the selected coordinate yields the exact marginal law $\Pi_{\mathcal I}$.

Hindsight spawning increases the number of active chains. With frozen components, however, the augmented process consisting of the active set $\mathcal I_t$ and its chain states is time-homogeneous on the disjoint union of population spaces. Before complete discovery, it is not the fixed-dimension random-scan chain associated with a fixed active-set product target. The theoretical guarantees in the following subsection are specifically designed to hold for this expansive target distribution.

\subsubsection{Theoretical Guarantees}
\paragraph{Standing Assumptions.}
To guarantee valid Markovian dynamics, all representation, validity, embedding, whitening, partition, target, and proposal components must be selected during a baseline or pilot phase and entirely frozen before the reported run. If the growing archive is allowed to change a proposal prompt, alter a validity rule, or update a target weight, the process becomes adaptive MCMC and; under such conditions, the fixed-kernel results established below no longer apply without imposing additional conditions \citep{roberts2007adaptive}.

\begin{theorem}[Anytime Active-Cell Uniformity]
\label{thm:ucmcmc-anytime}
Let $\mathcal F_{t-1}$ contain the complete history, the active set, and the individual chain states immediately prior to iteration $t$. For every initialization, $t\geq 1$, and $j\in[M]$, the spatial allocation of the emitted persona is strictly uniform over the active set
\begin{equation}
    \Pr\{c(P_t^{\mathrm{emit}})=j\mid\mathcal F_{t-1}\} = \frac{\mathbbm{1}\{j\in\mathcal I_{t-1}\}}{|\mathcal I_{t-1}|}.
    \label{eq:ucmcmc-anytime-uniformity}
\end{equation}
Moreover, define the predictable step exposure $\eta_{j,t}$, the cumulative realized count $N_j(T_{\mathrm M})$, and the cumulative expected exposure $E_j(T_{\mathrm M})$ as follows
\begin{align}
    \eta_{j,t} &= \frac{\mathbbm{1}\{j\in\mathcal I_{t-1}\}}{|\mathcal I_{t-1}|}, 
    & N_j(T_{\mathrm M}) &= \sum_{t=1}^{T_{\mathrm M}} \mathbbm{1}\{c(P_t^{\mathrm{emit}})=j\}, 
    & E_j(T_{\mathrm M}) &= \sum_{t=1}^{T_{\mathrm M}} \eta_{j,t}.
    \label{eq:ucmcmc-exposure}
\end{align}
Then $N_j(T_{\mathrm M})-E_j(T_{\mathrm M})$ is a martingale. For every fixed horizon $T_{\mathrm M}$ and $\delta\in(0,1)$, the deviation is strictly bounded
\begin{equation}
    \Pr\left\{\max_{1\leq j\leq M} |N_j(T_{\mathrm M})-E_j(T_{\mathrm M})| \geq \sqrt{\frac{T_{\mathrm M}}{2} \log\frac{2M}{\delta}} \right\} \leq \delta .
    \label{eq:ucmcmc-allocation-bound}
\end{equation}
\end{theorem}
\begin{proof}
The scheduler draws $J_t$ uniformly from $\mathcal I_{t-1}$. Since the source chain either accepts a same-cell state or defaults to a self-transition, the relation $c(P_t^{\mathrm{emit}})=J_t$ almost surely, proving \eqref{eq:ucmcmc-anytime-uniformity}. Consequently, $\mathbb E[\mathbbm{1}\{c(P_t^{\mathrm{emit}})=j\}\mid\mathcal F_{t-1}]=\eta_{j,t}$, which gives the martingale. Applying the bounded-increment Hoeffding--Azuma inequality \citep{azuma1967weighted} to each cell and taking union-bound over the fixed $M$ cells yields \eqref{eq:ucmcmc-allocation-bound}.
\end{proof}

\begin{remark}[Dynamic Allocation and Burn-In]
The predictable exposure $E_j(T_{\mathrm M})$ is the only mathematically appropriate benchmark while the active set continues to grow, as a newly discovered cell cannot receive allocations retroactively. Thus, while the accumulated trace may not be uniform over the \emph{final} set $\mathcal I_{T_{\mathrm M}}$, it remains concentrated around the exact dynamic allocation induced by the uniform scheduler. Variance-adaptive confidence sequences can strengthen \eqref{eq:ucmcmc-allocation-bound} to simultaneous validity over all stopping times \citep{howard2021confidence}; notably, no reset is required when a new cell is discovered.

Furthermore, Theorem~\ref{thm:ucmcmc-anytime} acts as a precise no-burn-in guarantee: cell marginal uniformity is exact from the first emitted iteration. Burn-in remains relevant only
if early personas are intended to be interpreted as draws from the asymptotic \emph{within-cell} law $\pi_j$, with that specific discrepancy bounded by \eqref{eq:ucmcmc-cell-tv}.
\end{remark}

\begin{theorem}[Eventual Discovery and Full-Target Convergence]
\label{thm:ucmcmc-discovery}
Let $\tau_j=\inf\{t:j\in\mathcal I_t\}$ be the discovery time of a reachable cell $j\notin\mathcal I_0$. Under the matched global proposal and assuming $\omega_0>0$
\begin{align}
    \Pr(\tau_j>T_{\mathrm M}) &\leq (1-\omega_0B_j)^{T_{\mathrm M}},    \label{eq:ucmcmc-cell-discovery}\\
    \Pr(\mathcal I_{T_{\mathrm M}}\neq\mathcal R) &\leq \sum_{j\in\mathcal R\setminus\mathcal I_0} (1-\omega_0B_j)^{T_{\mathrm M}},    \label{eq:ucmcmc-all-discovery}\\
    \mathbb E|\mathcal R\setminus\mathcal I_{T_{\mathrm M}}| &\leq \sum_{j\in\mathcal R\setminus\mathcal I_0} (1-\omega_0B_j)^{T_{\mathrm M}}.    \label{eq:ucmcmc-expected-undiscovered}
\end{align}
Since $\mathcal R$ is finite, every reachable cell is discovered in finite time almost surely. Let $\tau=\inf\{t:\mathcal I_t=\mathcal R\}$ and $M_{\mathcal R}=|\mathcal R|$. Conditional on the population at time $\tau$, for every integer $s\geq0$ the subsequent fixed-active-set chain obeys
\begin{equation}
    \left\| \mathsf K_{\mathcal R}^{\,s}(\mathbf p,\cdot) - \Gamma_{\mathcal R} \right\|_{\mathrm{TV}}
    \leq
    \min\left\{1,\,\sum_{j\in\mathcal R} \left(1-\frac{\omega_0B_j}{M_{\mathcal R}}\right)^s\right\}.
    \label{eq:ucmcmc-product-tv}
\end{equation}
In particular, the emitted-persona law converges to the equal-cell mixture
\begin{equation}
    \Pi_{\mathcal R} = \frac{1}{|\mathcal R|} \sum_{j\in\mathcal R}\pi_j.    \label{eq:ucmcmc-full-target}
\end{equation}
\end{theorem}
\begin{proof}
At each global iteration, the global proposal component is selected with probability $\omega_0$ and independently draws from the base law $b$. It therefore proposes a valid persona in cell $j$ with probability $\omega_0B_j$, independently of the selected source cell. Repeated failure gives \eqref{eq:ucmcmc-cell-discovery}. Then, applying a union bound and the linearity of expectation give \eqref{eq:ucmcmc-all-discovery} and \eqref{eq:ucmcmc-expected-undiscovered}. Note that local proposals can only discover a cell earlier. The right-hand side of \eqref{eq:ucmcmc-all-discovery} tends to zero for finite $\mathcal R$, proving almost-sure eventual discovery.

Following time $\tau$, coordinate $j$ is actively selected and undergoes an exact $\pi_j$ refresh with probability $\omega_0B_j/M_{\mathcal R}$ per population iteration. By coupling a chain initiated from $\mathbf p$ with a stationary chain using the same scan, proposal, and accept--reject randomness, coordinate $j$ coalesces at its first exact refresh and remains coupled thereafter. Union-bounding the probability that any coordinate has not refreshed by time $s$ directly proves \eqref{eq:ucmcmc-product-tv}.
To obtain convergence at deterministic times, let $d(s)$ denote the right-hand side of \eqref{eq:ucmcmc-product-tv}. The emission kernel $Q_{\mathrm{emit}}$ acting on the pre-transition population satisfies $\Gamma_{\mathcal R}Q_{\mathrm{emit}}=\Pi_{\mathcal R}$. For any deterministic integer $0\leq r<t$, conditioning on whether discovery is complete by time $r$ and applying total-variation contraction gives
$$\left\|\mathcal L(P_t^{\mathrm{emit}})-\Pi_{\mathcal R}\right\|_{\mathrm{TV}} \leq \Pr(\tau>r)+d(t-r-1).$$
Taking $r=\lfloor t/2\rfloor$ makes both terms vanish as $t\to\infty$, proving \eqref{eq:ucmcmc-full-target}.
\end{proof}

\begin{remark}[The Resolution Trade-off and Interpretive Scope]
The active coverage $|\mathcal I_t|$ is nondecreasing, so each hindsight discovery is a deterministic coverage gain even though it also expands the target against which future allocation is measured. Equations \eqref{eq:ucmcmc-cell-discovery}--\eqref{eq:ucmcmc-expected-undiscovered} quantify this gain. 
The discovery bounds also highlight the resolution trade-off: refinements to reduce cell diameters can tighten the geometric bound in \eqref{eq:ucmcmc-wasserstein}. Under nested refinement, each child cell has a base mass no greater than its parent, which can slow discovery and refresh.
Without a lower bound on the reachable cell masses, no distribution-free finite-time discovery guarantee can be operationally strong.

Finally, we stress that the complete sequential trace, including all repeated states, is the object governed by these MCMC guarantees. A deduplicated archive composed of baseline states, accepted same-cell states, and valid cross-cell proposals is useful for measuring structural novelty and contrasting against empirical evolutionary generation, but it must be formally analyzed and reported as a visited-state archive rather than as an unweighted sample drawn from $\Pi_{\mathcal I}$.
\end{remark}

\subsection{Evolutionary TextGrad Persona Generation}
\label{app:evolution}

Evolutionary TextGrad implements frontier generation by expanding the baseline
population toward valid, low-density regions of persona space.  Unlike UC-MCMC's
distributional target, this is an empirical search procedure; the resulting expanded
pool is passed to the same max--min selector used for fixed-pool selection.

\paragraph{Objective and textual feedback.}
We formulate persona generation as evolutionary constrained optimization.
At iteration $t$, let $\mathcal P_t$ be the current persona population.  For
any persona $p$, whether an incumbent or a newly proposed candidate, the
comparison set is $\mathcal P_t\setminus\{p\}$; set subtraction has no effect
when $p\notin\mathcal P_t$.  Write
$\overline{\mathbf e}(p)=\phi(p)/\|\phi(p)\|_2$ for the unit-normalized
persona embedding.  Let
$m_{\mathrm E}\in\{\mathrm{cos},2,\mathrm{Mah}\}$ be the persona
dissimilarity fixed for an evolutionary run; the reported run uses
$m_{\mathrm E}=\mathrm{Mah}$.  Define
\begin{align}
    \operatorname{Gap}_t(p)
    &=
    \min_{\widetilde p\in\mathcal P_t\setminus\{p\}}
    \dper{m_{\mathrm E}}(p,\widetilde p),
    \label{eq:evolution-gap}\\
    \widehat\rho_t(p)
    &=
    \frac{1}{|\mathcal P_t\setminus\{p\}|}
    \sum_{\widetilde p\in\mathcal P_t\setminus\{p\}}
    \exp\!\left(
        \beta_{\mathrm{KDE}}\,
        \overline{\mathbf e}(p)^\top
        \overline{\mathbf e}(\widetilde p)
    \right),
    \label{eq:evolution-density}\\
    \operatorname{Fit}_t(p)
    &=
    \lambda_{\mathrm{rel}}\operatorname{Val}(p)
    +\lambda_{\mathrm{gap}}\operatorname{Gap}_t(p)
    -\lambda_{\mathrm{den}}\log\widehat\rho_t(p).
    \label{eq:evolution-fitness}
\end{align}
Here $\operatorname{Val}(p)\in[0,1]$ measures satisfaction of the persona schema and task-independent validity constraints, $\operatorname{Gap}_t(p)$ is nearest-neighbor novelty under the same $\dper{m_{\mathrm E}}$ used for downstream persona selection, and $\widehat\rho_t(p)$ is an unnormalized von Mises-Fisher kernel-density score in persona embedding space. 
Parent-path comparisons use an identical reference set for the parent and offspring. Thus, extend the definitions as $\operatorname{Gap}_t(x;R)$ and $\widehat\rho_t(x;R)$ for a nonempty reference set $R$, where $R=\mathcal P_t\setminus\{x\}$ without explicit reference set.

Because personas are discrete text rather than differentiable vectors, we use
\emph{textual gradient} only for natural-language feedback, following
\citet{yuksekgonul2025textgrad}
\begin{equation}
    g_t^{\mathrm{text}}(p)
    =
    \operatorname{TextGrad}\!\left(
        p;\operatorname{Val}(p),\operatorname{Gap}_t(p),
        \widehat\rho_t(p),x_{\mathrm{grad}}
    \right),
    \label{eq:evolution-textgrad}
\end{equation}
where $x_{\mathrm{grad}}$ is the fixed feedback instruction.  The editor then
proposes
\begin{equation}
    p'=\operatorname{TGD.step}\!\left(p,g_t^{\mathrm{text}}(p)\right).
    \label{eq:evolution-update}
\end{equation}
This notation avoids treating LLM feedback as a literal derivative with
respect to a token sequence.

\paragraph{Evolutionary algorithm.}
 
The loop grows the population by mutating incumbents under textual-gradient feedback and admitting only offspring that demonstrably increase spread. Two design choices matter. First, selection and acceptance operate on the two diversity axes \emph{separately} rather than on the scalar \eqref{eq:evolution-fitness}: a scalar trade-off would let a large novelty gain purchase an equally large density loss, which is exactly the substitution that produces a few far-flung outliers around an unchanged core. We therefore report $\operatorname{Fit}_t$ as a summary statistic, while the loop itself is weight-free. Second, every admission score is evaluated against the frozen population $\mathcal P_{t-1}$, so individual scores do not depend on the order in which candidates are evaluated.
 
\paragraph{Parent selection.}
Each tournament draws $m_{\mathrm{tour}}=5$ incumbents uniformly without
replacement and admits two winners: the entrant with the largest
$\operatorname{Gap}_{t-1}$ and, when distinct, the entrant with the smallest
$\log\widehat\rho_{t-1}$. Tournaments repeat until the parent set reaches
$\beta_{\mathrm{par}}|\mathcal P_{t-1}|$ with $\beta_{\mathrm{par}}=0.05$, so
the batch grows with the population. Admitting one winner per axis applies
equal pressure to isolation and to sparsity without committing to an exchange
rate between them.
 
\paragraph{Admission.}
Let $p$ be a parent and $p'$ its offspring, and let
\begin{equation}
    m_{\mathrm{gap}} = c_m\operatorname{sd}\{\operatorname{Gap}_{t-1}(q)\}_{q\in\mathcal P_{t-1}},
    \qquad
    m_{\mathrm{den}} = c_m\operatorname{sd} \{\log\widehat\rho_{t-1}(q)\}_{q\in\mathcal P_{t-1}},
    \label{eq:evolution-margins}
\end{equation}
with $c_m=0.10$, be per-axis margins recomputed each iteration. The offspring is admitted when it passes the deterministic and relevance gates and satisfies either
\begin{equation}
    \begin{aligned}
    \text{(parent path)}\quad
    &\operatorname{Gap}_{t-1}(p';R_p)>\operatorname{Gap}_{t-1}(p;R_p)+m_{\mathrm{gap}}
    \\
    &\log\widehat\rho_{t-1}(p';R_p)<\log\widehat\rho_{t-1}(p;R_p)-m_{\mathrm{den}},
    \end{aligned}
    \label{eq:evolution-parent-path}
\end{equation}
\begin{equation}
    \begin{aligned}
        \text{(absolute path)}\quad
        &\operatorname{Gap}_{t-1}(p')\geq Q_{\tau}
        \\
        &\log\widehat\rho_{t-1}(p')\leq Q_{100-\tau}',
        \end{aligned}
        \label{eq:evolution-tau-path}
\end{equation}
where $Q_\tau$ and $Q'_{100-\tau}$ are the $\tau$-th and $(100-\tau)$-th population percentiles of the two axes, $\tau=75$.
The parent path requires improvement on both axes by explicit margins, discouraging admission based on negligible score differences.
The absolute path admits a candidate that fails to beat its parent but is nonetheless competitive with the population, which prevents an already-isolated parent from blocking progress. Scoring is deliberately asymmetric between the two: the parent path scores $p'$ against $R_p=\mathcal P_{t-1}\setminus\{p\}$ so that parent and offspring are compared leave-one-out on identical reference sets, whereas the absolute path scores against the full population so that a near-copy of an elite parent cannot inherit that parent's isolation.
 
\paragraph{Sibling deduplication.}
Because offspring are scored against a frozen population, their admission scores do no account for other offspring of the same iteration. When two admitted offspring are closer than $m_{\mathrm{gap}}$, we treat them as near-duplicates and retain the one with the larger full-population gap. This step provides a heuristic control on redundancy within the admitted batch.
 
\begin{algorithm}[t]
\caption{Evolutionary \textsc{TextGrad} Persona Generation}
\label{alg:evolution}
\begin{algorithmic}[1]
\Require Baseline population $\mathcal P_0$; iteration cap $T_{\mathrm E}$;
    population cap $N_{\max}$; batch fraction $\beta_{\mathrm{par}}$;
    tournament size $m_{\mathrm{tour}}$; margin fraction $c_m$; percentile
    $\tau$; frozen feedback instruction $x_{\mathrm{grad}}$
\For{$t=1,\ldots,T_{\mathrm E}$}
    \State Compute leave-one-out $\operatorname{Gap}_{t-1}(q)$ and
        $\log\widehat\rho_{t-1}(q)$ for every $q\in\mathcal P_{t-1}$
    \State Set the margins \eqref{eq:evolution-margins} and the percentile
        thresholds $Q_\tau,Q'_{100-\tau}$
    \State $\mathcal B_t\gets\emptyset$
    \While{$|\mathcal B_t|<\beta_{\mathrm{par}}|\mathcal P_{t-1}|$}
        \State Sample $m_{\mathrm{tour}}$ incumbents uniformly without
            replacement
        \State Add the largest-$\operatorname{Gap}$ entrant and, if distinct,
            the smallest-$\log\widehat\rho$ entrant
    \EndWhile
    \State $\mathcal A_t\gets\emptyset$
    \For{$p\in\mathcal B_t$}
        \State $p'\gets\operatorname{TGD.step}(p,g_{t-1}^{\mathrm{text}}(p))$
            \Comment{\eqref{eq:evolution-textgrad}--\eqref{eq:evolution-update}}
        \If{$p'$ fails the English or length gate \textbf{or}
            $\operatorname{Val}(p')\neq1$}
            \State \textbf{reject}
            \Comment{gates precede embedding and judging}
        \ElsIf{\eqref{eq:evolution-parent-path} \textbf{or}
            \eqref{eq:evolution-tau-path} holds}
            \State $\mathcal A_t\gets\mathcal A_t\cup\{p'\}$
        \EndIf
    \EndFor
    \State Remove from $\mathcal A_t$ the smaller-gap member of every pair
        within $m_{\mathrm{gap}}$
        \Comment{sibling deduplication}
    \State $\mathcal P_t\gets\mathcal P_{t-1}\cup\mathcal A_t$
    \If{$|\mathcal P_t|\geq N_{\max}$} \State \textbf{break} \EndIf
\EndFor
\Ensure Evolved pool $\mathcal P_{T_{\mathrm E}}$; the max--min selector of
    \Cref{sec:selection} is applied downstream to return
    $\mathcal S^\star_{\mathrm{disp},m_{\mathrm E}}
    (\mathcal P_{T_{\mathrm E}})$
\end{algorithmic}
\end{algorithm}
 
The population is grown by union and never shrinks, so coverage of the persona
space is monotone in $t$. Selection by max--min dispersion is deliberately kept
downstream of the loop rather than folded into it: holding the selection
operator fixed and varying only the candidate pool is what isolates the effect
of generation from the effect of selection.

\section{Experiment Details}
\label{app:experiment_details}
The number of personas $k=5$: The Vendi score using cosine Gram matrix is measured as $10.4$ on the baseline personas. This translates to an effective number of distinct personas in the baseline pool is $10$, and any number $k\le10$ should be able to accommodate $k$ semantically distinct personas. We chose $k=5$ for convenience and the feasibility of human clustering and judgment with multiple variations and baselines.

Performance tables report means over five response-generation seeds (42--46) and two-sided 95\% Student's $t$ confidence intervals. For seed-level metric values $x_1,\ldots,x_5$, computed using the benchmark-specific aggregation, we report
$$\overline{x}\pm t_{0.975,4}\frac{s}{\sqrt{5}}, \qquad s^2=\frac{1}{4}\sum_{r=1}^{5}(x_r-\overline{x})^2,$$
where $t_{0.975,4}\approx2.776$. These intervals summarize variability across response-generation seeds for the evaluated persona sets. ICC intervals in the judge-validation analysis use bootstrap resampling instead.

\subsection{Common Baselines and Variants}
\paragraph{Persona control baselines.}
\textsc{Gibberish} acts as a length-matched pseudoword control to verify that performance gains stem from semantic content rather than mere prompt length. Prompted with Jabberwocky personas, LLMs generate responses containing pseudowords (``shield from \emph{nasnorn} rain'' or ``press flowai glurt sheevu''). Such out-of-vocabulary tokens embed in low-density regions, inflating every distance-based metric \eqref{eq:originality} without any corresponding remoteness of the underlying idea. We therefore delete them from each use before computing embeddings. The pseudowords originates from gibberish input: the gibberish generation process replaces content words one-for-one in place from an existing persona, so aligning a persona with its gibberish counterpart recovers a list fo non-words. A response word is deleted if it matches one of them, ignoring case and inflectional endings. The responses are evaluated after deletion (``shield from rain'' or ``press'').

\textsc{Task-conditioned} persona generation \citep{jin2025multi} (MPAQ) serves as a baseline that explicitly generates task-specific personas. Task-specific generation inherently incurs expensive, per-task inference costs; by contrast, our generic diversified personas incur only a one-time optimization cost and can be deployed universally across any task. Due to this overhead, task-conditioned personas are not useful unless they are better for creative responses than task-agnostic personas.

\paragraph{Baseline persona pool.}
We utilize the PersonaMem-v2 dataset \citep{jiang2025know,ge2024scaling} as our fixed candidate pool. To extract $k=5$ personas from this massive population, we evaluate two standard baselines: \textsc{Random} (uniform sampling) to test arbitrary individualization, and \textsc{Typical} (the $5$ densest-mode personas via Gaussian KDE) to test representative personas.

\paragraph{Diversity-selected personas (ours).}
\textsc{Selection} extracts $k{=}5$ personas from the baseline pool by maximizing min dispersion (\Cref{sec:selection}) under the Mahalanobis metric. Contrasting this with \textsc{Random} and \textsc{Typical} isolates the effect of \emph{diverse selection} over standard individualization.

\paragraph{Generated personas (ours).}
To break the performance ceiling imposed by a finite pool, we activaly expand the persona population using two algorithms prior to applying the identical max--min dispersion selector. \textsc{MCMC} acts as a space-filling generator targeting an equal-cell mixture over reachable semantic regions (\Cref{sec:mcmc}). \textsc{Evolution} actively push personas toward low-density extremes by applying evolutionary \textsc{TextGrad} (\Cref{sec:evolution}).

\subsection{Alternative Uses Test (AUT)}
\label{sec:aut}
The Alternative Uses Test (AUT) is a widely used divergent-thinking
benchmark \citep{guilford1967nature}.  Given a common object (e.g., a fork,
book, or wallet), respondents produce alternative uses.  Its open-ended,
object-conditioned form has become a standard testbed for LLM creativity
\citep{stevenson2022putting,goes2023pushing,rabeyah2025llms} and is well
suited to testing whether persona diversification broadens ideational range.
For object $o$ and experimental condition $h$, let
$\widetilde{\mathcal U}_{o,h}$ denote the raw multiset of
$n_{\mathrm{raw}}=25$ generated uses and let $\mathcal U_{o,h}$ denote the
canonicalized, deduplicated, valid set used by downstream metrics.  Its actual
size $n_{o,h}=|\mathcal U_{o,h}|$ can therefore be smaller than $25$.

\subsubsection{Baselines and Variants}
\label{sec:variants}

We compare a graded ladder of prompting and persona conditions, each adding one ingredient over the last so that its marginal effect is isolated. All conditions share the generation protocol below; they differ only in the prompt and in how (or whether) a persona is supplied.

\paragraph{Standard persona baselines.}
Three persona-free prompts from the ladder of \citet{goes2023pushing} isolate what prompting alone achieves: \textsc{Common-Use} (their \emph{nn} prompt, eliciting conventional uses) as a lower anchor; \textsc{Alternative-Use} (their \emph{nc} ``creative'' prompt); and \textsc{Creativity-Enhanced} (their \emph{bs} expert prompt with an explicit definition of creative use). \textsc{Alternative-Use} is the baseline for persona injection: every persona condition below supplies a persona to that same \emph{nc} prompt, so any gain over a standard, non-persona prompt is attributable to the persona, not to prompt optimization.

\paragraph{Evolutionary personas (ours).}
\textsc{AUT-Evolution} produces the pool of personas by the same evolutionary generation algorithm (\Cref{alg:evolution}), but with AUT-specific fitness function. Fix the response generator and the benchmark prompt. For persona $p$ and object
$o$, let $\mathcal U_o(p)$ denote the valid uses that $p$ produces for $o$ under
the protocol of \Cref{sec:aut}, and let $m_{\mathrm R}$ be the fixed response
dissimilarity. A persona is scored on four axes, each averaged over the
evaluated objects,
\begin{align}
    \operatorname{Util}_o(p)
    &=
    \frac{1}{|\mathcal U_o(p)|}
    \sum_{u\in\mathcal U_o(p)} s_o(u),
    \label{eq:autevo-utility}\\
    \operatorname{Nov}_o(p)
    &=
    \frac{1}{|\mathcal U_o(p)|}
    \sum_{u\in\mathcal U_o(p)}
    \min_{v\in\mathcal U_o^{\mathrm{comm}}}
    d_{m_{\mathrm R}}^{\mathrm R}(u,v),
    \label{eq:autevo-novelty}\\
    \operatorname{Div}_o(p)
    &=
    \binom{|\mathcal U_o(p)|}{2}^{-1}
    \sum_{\{u,v\}\subseteq\mathcal U_o(p)}
    d_{m_{\mathrm R}}^{\mathrm R}(u,v),
    \label{eq:autevo-diversity}\\
    \operatorname{Flex}_o(p)
    &=
    \operatorname{VS}(\mathcal U_o(p)),
    \label{eq:autevo-flexibility}
\end{align}
so that novelty is a directed Chamfer distance to the frozen common-use
reference set, diversity is within-persona spread, and flexibility is the
effective number of distinct use categories \eqref{eq:vendi}. The utility
strength $s_o(u)$ is obtained by judging each use pairwise against $K=3$
strength-stratified anchors drawn from a frozen per-object ladder and fitting a
one-dimensional Bradley--Terry model against the fixed anchor strengths; each
ladder is sum-zero identified, so $\operatorname{Util}_o$ reads as an average
advantage over baseline utility and is not comparable across objects. Validity
is a gate rather than an axis: an offspring is discarded unless every retained
use passes the parse, length, language, and semantic-admissibility checks of
\Cref{sec:aut}.
 
Write $A(p)=(\operatorname{Util},\operatorname{Nov},\operatorname{Div},
\operatorname{Flex})(p)\in\mathbb R^4$ for the object-averaged axis vector.
Admission replaces \eqref{eq:evolution-parent-path}--\eqref{eq:evolution-tau-path}
by a majority rule on the four axes: the offspring is accepted when it strictly
beats its parent on at least $w=3$ axes, or matches the $\tau$-th percentile
member of the population on at least $w$ axes. Parent selection uses the same
tournament, with the deciding axis rotated through a random permutation of the
four so that each receives equal pressure over an iteration. Because all four
axes are frozen within a metric epoch, the common-use pool, the per-object
covariances, and the anchor ladder are snapshotted at run start, a later
re-evaluation of the benchmark cannot retroactively move the objective a run was
optimized against.
 
Two properties of this substitution are worth stating. It is strictly more
expensive, since each candidate must generate and have judged a full set of
responses before it can be scored, whereas the persona-text fitness is computed
from embeddings alone. And it is task-bound: the resulting pool is optimized for
the AUT and carries no guarantee of transfer, which is precisely the trade the
task-agnostic fitness avoids. We report both pools and compare them directly.

\paragraph{Human reference.}
For the three objects with released human norms (book, fork, tin can), we
include the \citet{stevenson2022putting} human responses as a reference for
the judge-scored dimensions.  We sample
$n_{\mathrm{comp}}=5$ human respondents, matching the number of independent
LLM completions per condition.

\paragraph{Generation and sampling protocol.}
All responses are generated by Gemma-4-31B (temperature $0.7$) under a
uniform output constraint, five words per use, no adjectives, and an
unbounded list, so conditions differ only in prompt and persona.  The
baseline population is
$\mathcal P_0=\{p_i\}_{i=1}^{N_0}$ from PersonaMem
\citep{jiang2025know}, with $N_0=997$ after excluding malformed or
missing entries.  If $x_o$ is the prompt instantiated for object $o$, a
persona-conditioned response follows
$y\sim G_{\theta_{\mathrm R}}(\,\cdot\mid x_o,p)$ with
$p\in\Omega_{\varnothing}:=\Omega\cup\{\varnothing\}$; persona-free
conditions use the conditioning sentinel $p=\varnothing$, which is distinct
from the failed persona-generation symbol $\bot$.  Non-persona conditions use
$n_{\mathrm{comp}}=5$ independent completions per object, while persona
conditions use one completion for each of the $k=5$ selected personas.  We
retain the first $n_{\mathrm{use}}=5$ uses from each completion, giving
$n_{\mathrm{raw}}=n_{\mathrm{comp}}n_{\mathrm{use}}=25$ raw uses per
object--condition pair before the validity gate.

\begin{table}[t]
\centering\small
\begin{tabular}{@{}lll@{}}
\toprule
Variant & Persona & Description \\
\midrule
\textsc{Common-Use}        & --  & \cite{goes2023pushing} common-use (nn) \\
\textsc{Alternative-Use}   & --  & \cite{goes2023pushing} creative (nc) \\
\textsc{Creativity-Enhanced}& -- & \cite{goes2023pushing} expert (bs) \\
\textsc{Random}                 & \checkmark & $5$ chosen uniformly at random \\
\textsc{Typical}                & \checkmark & $5$ density-ranked modes \\
\textsc{Selection} (ours)      & \checkmark & maximize dispersion on base pool \\
\textsc{Evolution} (ours) & \checkmark & maximize dispersion on evolved pool \\
\textsc{Human}                  & --  & $5$ chosen from \cite{stevenson2022putting} \\
\bottomrule
\end{tabular}
\caption{Variants, from persona-free baselines to our diversity-aware selection and generation. Persona variants all inject into the \textsc{Alternative-Use} prompt; \textsc{Selection} and \textsc{Evolution} share one selection operator and differ only in the candidate pool.}
\label{tab:variants}
\end{table}

\subsubsection{Metrics}
\paragraph{Notation.}
Let $\mathcal Y\subseteq\mathcal T$ denote the response-text space.  The
fixed encoder $\phi:\mathcal T\to\mathbb R^d$ maps any persona, prompt,
object label, or response text to an embedding; write
$\mathbf e(\cdot)=\phi(\cdot)$.  
The regularized covariance $\widehat\Sigma_{\mathrm R}$ is estimated once on
a frozen reference split and is not refit by condition.  We call
$d_{\mathrm{cos}}^{\mathrm R}$ a dissimilarity because it need not satisfy
the triangle inequality.  Let $\mathcal U_o^{\mathrm{comm}}$ be the frozen
common-use reference set for object $o$, drawn from no-persona baseline
generations \citep{goes2023pushing}.

For a valid response set
$\mathcal U=\{u_i\}_{i=1}^{n}$, let
$\overline{\mathbf e}_i=\mathbf e(u_i)/\|\mathbf e(u_i)\|_2$,
let $\mathbf E_{\mathcal U}$ stack these unit vectors by row, and let
$\mathbf G_{\mathcal U}
=\mathbf E_{\mathcal U}\mathbf E_{\mathcal U}^{\top}$ be their cosine Gram
matrix.  If $\xi_1,\ldots,\xi_n$ are the eigenvalues of
$\mathbf G_{\mathcal U}/n$, the Vendi Score is
\begin{equation}
    \operatorname{VS}(\mathcal U)
    =
    \exp\!\left(-\sum_{i=1}^{n}\xi_i\log\xi_i\right),
    \label{eq:vendi}
\end{equation}
with $0\log0=0$.  It is the effective number of distinct responses
\citep{friedman2023vendi,pasarkar2024cousins}.

\paragraph{Validity.}
We map each raw use through a deterministic canonicalizer
$\operatorname{can}(\cdot)$ (lowercasing, filler and punctuation removal, and
lemmatization).  A separate, human-validated LLM judge supplies the
admissibility indicator
$\operatorname{val}_{\mathrm R}:\mathcal Y\rightarrow\{0,1\}$.
For
$\widetilde{\mathcal U}
=\{\widetilde u_i\}_{i=1}^{n_{\mathrm{raw}}}$,
\begin{equation}
    \operatorname{Validity}(\widetilde{\mathcal U})
    =
    \frac{1}{n_{\mathrm{raw}}}
    \sum_{i=1}^{n_{\mathrm{raw}}}
    \operatorname{val}_{\mathrm R}
    \!\left(\operatorname{can}(\widetilde u_i)\right).
    \label{eq:validity}
\end{equation}
The downstream set is
$$\mathcal U = \left\{\operatorname{can}(\widetilde u): \widetilde u\in\widetilde{\mathcal U},\ \operatorname{val}_{\mathrm R}(\operatorname{can}(\widetilde u))=1\right\},$$
where set semantics merge exact canonical duplicates.  Thus validity uses the
raw denominator, whereas all following metrics use the post-gate denominator
$n=|\mathcal U|$.  Gating prevents invalid outliers from receiving
spuriously high originality scores \citep{chen2023probing}.

\paragraph{Diversity.}
The first diversity metric is the mean pairwise response dissimilarity (defined for $n\geq2$)
\begin{equation}
    \operatorname{Div}_{m}(\mathcal U) = \frac{2}{n(n-1)} \sum_{i<j}d_m^{\mathrm R}(u_i,u_j).   \label{eq:diversity}
\end{equation}

Another diversity metric is the volume of convex hull on reduced PCA dimensions. Let $e_m(u)\in\mathbb R^{D_m}$ denote the response embedding realized in space $m$: the L2-normalized embedding for $m=\mathrm{cos}$, the raw embedding for $m=2$, and the whitened truncated embedding $L_o^{\top}\tilde e(u)$ for $m=\mathrm{Mah}$, where $\Sigma_o^{-1}=L_oL_o^{\top}$ is the Cholesky factorization of the object-specific inverse covariance. Let $\Pi_{m,o}\colon\mathbb R^{D_m}\to\mathbb R^{k}$ ($k=5$) be the PCA projection fit on the pooled embeddings ${e_m(u)}$ of all variants' responses for object $o$. Then the volume is defined for $n \geq k+1$ as
\begin{equation}
    \operatorname{Vol}_{m}(\mathcal U)
    = \operatorname{vol}_{k}\!\Big(
        \operatorname{conv}\big\{\Pi_{m,o}\, e_m(u_1),\ldots,\Pi_{m,o}\, e_m(u_n)\big\}
      \Big),
    \label{eq:volume}
\end{equation}
where $\operatorname{conv}(\cdot)$ is the convex hull and $\operatorname{vol}_k$ the $k$-dimensional Lebesgue volume.

\paragraph{Originality.}
We reserve \emph{originality} for reference-relative response distance,
distinguishing it from the persona novelty gap in
\eqref{eq:evolution-gap}.  The object-relative and common-use-relative
versions are
\begin{equation}
\begin{aligned}
    \operatorname{Orig}_{m}(\mathcal U;o)
    &=
    \frac{1}{n}\sum_{i=1}^n d_m^{\mathrm R}(u_i,o),
    \\
    \operatorname{Orig}_{m}
        (\mathcal U;\mathcal U_o^{\mathrm{comm}})
    &=
    \frac{1}{n}\sum_{i=1}^n
    \min_{v\in\mathcal U_o^{\mathrm{comm}}}
    d_m^{\mathrm R}(u_i,v).
\end{aligned}
    \label{eq:originality}
\end{equation}
The second expression is the directed Chamfer distance from the generated
set to the frozen common-use reference set \citep{chen2023probing}.

\paragraph{Category Metrics (Flexibility and Surprise).}
Let $\gamma(u)$ be the unique category assigned to valid use $u$ by human
clustering, based on action class
\citep{fillmore1968case,levin1993english} and affordance
\citep{gibson1979ecological,norman1988psychology}, and define
$\operatorname{Cat}(\mathcal U)=\{\gamma(u):u\in\mathcal U\}$.
The size-normalized manual-category metrics are
\begin{equation}
    \operatorname{Flex}_{\gamma}(\mathcal U)
    =
    \frac{|\operatorname{Cat}(\mathcal U)|}{n},
    \label{eq:flexibility_gamma}
\end{equation}
\begin{equation}
    \operatorname{Sur}_{\gamma}
        (\mathcal U;\mathcal U_o^{\mathrm{comm}})
    =
    \frac{1}{n}\sum_{i=1}^n
    \mathbbm{1}\!\left[
        \gamma(u_i)\notin
        \operatorname{Cat}(\mathcal U_o^{\mathrm{comm}})
    \right].
    \label{eq:surprise-gamma}
\end{equation}
The clustering-free flexibility proxy is
$\operatorname{Flex}_{\mathrm{VS}}(\mathcal U)
=\operatorname{VS}(\mathcal U)/n$.

For the proposed originality-weighted Vendi proxy, set
$$\upsilon_i = \min_{v\in\mathcal U_o^{\mathrm{comm}}} d_{\mathrm{cos}}^{\mathrm R}(u_i,v), \qquad \overline\upsilon_i = \frac{\upsilon_i}{\sum_{r=1}^{n}\upsilon_r}, \qquad \overline{\boldsymbol\upsilon} = (\overline\upsilon_1,\ldots,\overline\upsilon_n)^\top,$$
when $\sum_r\upsilon_r>0$, and define
$$\mathbf G_{\mathcal U}^{(\upsilon)} = \operatorname{diag}(\sqrt{\overline{\boldsymbol\upsilon}})\, \mathbf G_{\mathcal U}\, \operatorname{diag}(\sqrt{\overline{\boldsymbol\upsilon}}).$$
If $\xi_i^{(\upsilon)}$ are its eigenvalues, then
\begin{equation}
    \operatorname{Sur}_{\mathrm{VS}}
        (\mathcal U;\mathcal U_o^{\mathrm{comm}})
    =
    \frac{1}{n}
    \exp\!\left(
        -\sum_i\xi_i^{(\upsilon)}\log\xi_i^{(\upsilon)}
    \right).
    \label{eq:surprise-vendi}
\end{equation}
We set this proxy to $0$ when all $\upsilon_i=0$.  Unlike
$\operatorname{Sur}_{\gamma}$, it is a proposed composite of
reference-relative originality and within-set diversity, not a literal
category share; its weighted spectral construction follows the
quality-weighted Vendi framework \citep{nguyen2024quality}.

\paragraph{Creativity Scores by LLM Judge.}
The creativity, originality, surprise, and utility of each use are rated by an LLM judge. Following \citet{goes2023pushing}, all valid canonical uses for object $o$ are pooled across conditions, randomly partitioned into batches, ranked and scored, and repeated for $R_{\mathrm{judge}}=5$ rounds with fresh partitions.  The rank score is the mean within-batch percentile, and the judge score is the mean raw rating, scaled to $[1,5]$.

With the same judge, we obtain absolute per-use ratings of originality, surprise, and utility using faithful renderings of \citet{stevenson2022putting}'s human protocols. Utility is reported separately but remains the effectiveness component of creativity, rather than an unrelated quantity \citep{runco2012standard}.

\begin{table}[t]
\centering\small
\begin{tabular}{@{}lll@{}}
\toprule
Metric & Basis & Description \\
\midrule
Validity                 & Human / LLM judge        & Proportion of admissible, well-formed uses (gate; control) \\
Diversity                & Embedding        & Mean pairwise distance among a variant's uses ($\uparrow$) \\
Originality              & Embedding        & Mean semantic or directed Chamfer distance from uses \\
                         &                  & to the object concept or common uses ($\uparrow$) \\
Flexibility              & Category         & (Effective) number (ratio) of distinct use categories ($\uparrow$) \\
Surprise                 & Category         & (Effective) share of uses in categories beyond common uses ($\uparrow$) \\
Creativity               & LLM judge        & Pooled rank-and-score creativity, \citet{goes2023pushing} ($\uparrow$) \\
Originality              & LLM judge        & Deviation from typical use, \citet{stevenson2022putting} ($\uparrow$) \\
Surprise                 & LLM judge        & Unexpectedness, \citet{stevenson2022putting} ($\uparrow$) \\
Utility                  & LLM judge        & Feasibility/usefulness (control), \citet{stevenson2022putting} \\
\bottomrule
\end{tabular}
\caption{Creativity and diversity metrics. Embedding metrics use the fixed
encoder $\phi$ (EmbeddingGemma) with cosine, L2, or Mahalanobis
dissimilarity; category metrics use both human clustering and Vendi-based
effective counts; LLM-judge metrics use an independent Qwen3.6-27B judge.
$\uparrow$: higher indicates a larger value. Validity is a gate. Utility is
the effectiveness dimension required alongside originality for creativity
\citep{runco2012standard} and is reported separately to expose trade-offs.}
\label{tab:metrics}
\end{table}

We use Qwen3.6-27B, an open-weight model from a family distinct from the Gemma-4-31B generator. Qwen3.6-27B shows moderate agreement with the two-rater mean on Stevenson-style ratings (Spearman ($\rho=.566$)–($.635$); ICC($(C,1)=.407$)–($.568$)), below the human–human agreement reference (($\rho=.669$)–($.842$); ICC ($=.677$)–($.790$)). We therefore interpret these scores as noisy ordinal indicators, with strongest support for originality, and rely on convergence with embedding- and category-based measures rather than treating them as interchangeable with human ratings.

\begin{table*}[t]
\centering
\scriptsize
\setlength{\tabcolsep}{3.5pt}
\begin{tabular}{ll r rr rrrrr}
\toprule
Dim. & Pair & $n$ & ICC$(C,1)$ & 95\% CI & QWK & $\rho$ & Exact & Within-1 & Paper ICC \\
\midrule
Orig. & LLM vs.\ rater~1 & 139 & 0.480 & $[0.341,0.609]$ & 0.447 & 0.553 & 0.547 & 0.878 & 0.78 \\
Orig. & LLM vs.\ rater~2 & 139 & 0.575 & $[0.452,0.693]$ & 0.574 & 0.599 & 0.504 & 0.942 & 0.78 \\
Orig. & LLM vs.\ mean     & 139 & 0.568 & $[0.443,0.682]$ & 0.481 & 0.635 & 0.302 & 0.878 & 0.78 \\
Orig. & Human ceiling     & 139 & 0.677 & $[0.577,0.754]$ & 0.609 & 0.669 & 0.475 & 0.993 & 0.78 \\
\midrule
Util. & LLM vs.\ rater~1 & 139 & 0.355 & $[0.211,0.504]$ & 0.344 & 0.509 & 0.475 & 0.813 & 0.70 \\
Util. & LLM vs.\ rater~2 & 139 & 0.430 & $[0.292,0.568]$ & 0.406 & 0.597 & 0.532 & 0.806 & 0.70 \\
Util. & LLM vs.\ mean     & 139 & 0.407 & $[0.281,0.531]$ & 0.377 & 0.603 & 0.396 & 0.784 & 0.70 \\
Util. & Human ceiling     & 139 & 0.701 & $[0.493,0.851]$ & 0.696 & 0.745 & 0.712 & 0.978 & 0.70 \\
\midrule
Surp. & LLM vs.\ rater~1 & 139 & 0.432 & $[0.229,0.608]$ & 0.423 & 0.500 & 0.612 & 0.935 & 0.79 \\
Surp. & LLM vs.\ rater~2 & 139 & 0.544 & $[0.407,0.659]$ & 0.490 & 0.571 & 0.525 & 0.885 & 0.79 \\
Surp. & LLM vs.\ mean     & 139 & 0.523 & $[0.354,0.667]$ & 0.468 & 0.566 & 0.446 & 0.906 & 0.79 \\
Surp. & Human ceiling     & 139 & 0.790 & $[0.728,0.841]$ & 0.756 & 0.842 & 0.698 & 0.971 & 0.79 \\
\bottomrule
\end{tabular}
\caption{AUT LLM-as-a-judge agreement on the \emph{valid} Stevenson subset
($n{=}139$). Paper ICC is the human--human ICC reported by
Stevenson et~al.; Human ceiling is the reproduced
\texttt{rater01}--\texttt{rater02} agreement on the released file.
LLM scores are from Qwen3.6-27B with the faithful Stevenson prompts
(including 0/99 codes).}
\label{tab:aut-judge-valid}
\end{table*}

\begin{table*}[t]
\centering
\scriptsize
\setlength{\tabcolsep}{3.5pt}
\begin{tabular}{ll r rr rrrrr}
\toprule
Dim. & Pair & $n$ & ICC$(C,1)$ & 95\% CI & QWK & $\rho$ & Exact & Within-1 & Paper ICC \\
\midrule
Orig. & LLM vs.\ rater~1 & 147 & 0.386 & $[0.248,0.525]$ & 0.378 & 0.488 & 0.524 & 0.844 & 0.57 \\
Orig. & LLM vs.\ rater~2 & 147 & 0.479 & $[0.341,0.608]$ & 0.474 & 0.533 & 0.476 & 0.905 & 0.57 \\
Orig. & LLM vs.\ mean     & 147 & 0.467 & $[0.332,0.601]$ & 0.410 & 0.555 & 0.286 & 0.837 & 0.57 \\
Orig. & Human ceiling     & 147 & 0.630 & $[0.498,0.734]$ & 0.579 & 0.642 & 0.476 & 0.986 & 0.57 \\
\midrule
Util. & LLM vs.\ rater~1 & 147 & 0.336 & $[0.199,0.472]$ & 0.316 & 0.499 & 0.449 & 0.769 & 0.68 \\
Util. & LLM vs.\ rater~2 & 147 & 0.400 & $[0.264,0.534]$ & 0.367 & 0.572 & 0.503 & 0.769 & 0.68 \\
Util. & LLM vs.\ mean     & 147 & 0.382 & $[0.262,0.501]$ & 0.362 & 0.584 & 0.374 & 0.748 & 0.68 \\
Util. & Human ceiling     & 147 & 0.684 & $[0.505,0.830]$ & 0.680 & 0.750 & 0.694 & 0.973 & 0.68 \\
\midrule
Surp. & LLM vs.\ rater~1 & 147 & 0.391 & $[0.210,0.566]$ & 0.374 & 0.455 & 0.585 & 0.905 & 0.67 \\
Surp. & LLM vs.\ rater~2 & 147 & 0.480 & $[0.344,0.602]$ & 0.423 & 0.501 & 0.497 & 0.850 & 0.67 \\
Surp. & LLM vs.\ mean     & 147 & 0.467 & $[0.313,0.613]$ & 0.411 & 0.502 & 0.422 & 0.878 & 0.67 \\
Surp. & Human ceiling     & 147 & 0.772 & $[0.704,0.832]$ & 0.743 & 0.831 & 0.694 & 0.966 & 0.67 \\
\bottomrule
\end{tabular}
\caption{AUT LLM-as-a-judge agreement on \emph{all} doubly rated Stevenson
responses ($n{=}147$), including rows with invalid codes. Metrics as in
Table~\ref{tab:aut-judge-valid}.}
\label{tab:aut-judge-all}
\end{table*}

\paragraph{Interpretation.}
On the production-relevant valid subset, the LLM tracks the human mean with
Spearman $\rho \in [0.57,0.64]$ and within-one rates of $78$--$91\%$: most
disagreements are off-by-one on a five-point scale rather than ordinal
inversions.
Relative to the reproduced human--human ceiling, the LLM recovers about
$84\%$ of the ceiling ICC and $95\%$ of the ceiling Spearman on
\textbf{originality}, and still $58$--$81\%$ of the ceiling on utility and
surprise.
Exact agreement with the continuous rater mean is lower by construction
(integer LLM scores vs.\ half-point means), so we emphasize ICC, QWK,
Spearman, and within-one.
Taken together, the judge is below, but in the same regime as trained
human reliability, which is the appropriate bar for LLM-as-a-judge
deployment.
We therefore treat the Stevenson originality, utility, and surprise scores
in the main AUT results as credible automated ratings, with originality the
best-aligned dimension and utility the most conservative.

\subsection{Infinity-Chat 100}
We evaluate five manually selected queries from Infinity-Chat 100 \cite{jiang2026artificial}: 1, 2, 46, 60, and 78. For each query and condition, persona-free prompting produces 50 responses, while persona conditions produce ten responses from each of five personas. This yields 250 responses per condition per generation seed. Generation uses temperature $1.0$, top-$p=0.9$, and a maximum of 512 new tokens. Before computing embedding and lexical metrics, we exclude responses explicitly labeled invalid.

\subsubsection{Baselines and Variants}
The persona-free baseline uses the original query. We also evaluate the shared ZS-CoT, Step-Back, and DMAD prompt variants, together with the persona controls and selection and generation variants defined above. The composition experiment adds the selected personas to the DMAD prompt.

\paragraph{Task-specific evolutionary baseline.}
\textsc{IC-Evolution} replaces persona-space fitness with three response-space axes: mean pairwise cosine distance, Vendi score, and one minus mean word-trigram Jaccard overlap. These are computed within each query and averaged across queries after validity screening and mechanical removal of persona framing. The resulting population is supplied to the downstream Mahalanobis max--min dispersion selector.

\subsubsection{Metrics}
\paragraph{Homogeneity.}
For a query and condition, let $\mathcal U=(u_1,\ldots,u_n)$ denote the response collection, and write $d_{ij}=d_{\mathrm{cos}}^{\mathrm R}(u_i,u_j)$. Homogeneity is mean pairwise cosine similarity,
$$\operatorname{Hom}(\mathcal U)=1-\frac{2}{n(n-1)}\sum_{i<j}d_{ij}.$$

\paragraph{Persona Separation.}
Let $p_i$ be the persona generating $u_i$. Define $\mathcal B=\{(i,j):i<j,\ p_i\ne p_j\}$ and $\mathcal I=\{(i,j):i<j,\ p_i=p_j\}$. Persona separation is
$$\operatorname{Sep}(\mathcal U) = \frac{1}{|\mathcal B|}\sum_{(i,j)\in\mathcal B}d_{ij}-\frac{1}{|\mathcal I|}\sum_{(i,j)\in\mathcal I}d_{ij}.$$
Separation is undefined when either pair set is empty. Both metrics are computed within each query and averaged across queries.

\paragraph{Flexibility.}
We apply the cosine-Gram Vendi score in \eqref{eq:vendi} to the response embeddings for each query and average across queries.

\paragraph{Validity and Quality.}
Validity is the fraction of responses admitted by the validity screen. Qwen3.6-27B independently rates the quality of each response on a five-point scale for execution and fit to the request and the scores are averaged across responses and queries.

\paragraph{LLM Judge on Response Quality and Validity}

Infinity-Chat releases human \emph{ratings} of model responses without human-written response pools \cite{jiang2026artificial}. We use the ratings to validate Qwen3.6-27B judge employed in evaluation: $750$ (query,~response) pairs from $25$ external LMs, each with approximately $25$ independent absolute ratings on a 1--5 scale (\texttt{human\_absolute}). The quality head is scored against the $25$-rater mean (the paper's comparison target for LM judges); the validity head is checked with a tripwire that responses humans rate $\ge 4$ must almost never be labeled invalid. Table~\ref{tab:ic-judge-quality} reports rank and absolute agreement for quality; Table~\ref{tab:ic-judge-validity} reports the validity screen.

\begin{table}[t]
\centering
\small
\begin{tabular}{lr}
\toprule
Metric & Value \\
\midrule
$n$ scored pairs & 750 \\
Quality prompt & \texttt{hivemind\_quality\_absolute\_v1} \\
\midrule
Spearman $\rho$ (judge vs.\ human mean) & 0.158 \\
Pearson $r$ (judge vs.\ human mean) & 0.278 \\
ICC$(C,1)$ & 0.225 \\
ICC 95\% CI (bootstrap) & $[0.146,\,0.296]$ \\
MAE & 1.021 \\
RMSE & 1.156 \\
Mean bias (judge $-$ human mean) & $+0.691$ \\
\midrule
Exact agree (vs.\ rounded human mean) & 0.136 \\
Within-one (vs.\ rounded human mean) & 0.840 \\
QWK (vs.\ rounded human mean) & 0.132 \\
\midrule
Human split-half Spearman ceiling & 0.624 \\
$\rho$ / ceiling & 0.253 \\
\bottomrule
\end{tabular}
\caption{Infinity-Chat quality-judge agreement against the $25$-rater human
mean on the Artificial Hivemind dense absolute-rating set.
Exact / within-one / QWK compare the integer LLM score to the human mean
rounded to the nearest integer.}
\label{tab:ic-judge-quality}
\end{table}

\begin{table}[t]
\centering
\small
\begin{tabular}{lrrr}
\toprule
Human mean bin & $n$ & Invalid rate & Prompt \\
\midrule
Good ($\ge 4$) & 405 & 0.002 & \texttt{hivemind\_validity\_screen\_v1} \\
Mid ($(2,4)$) & 341 & 0.023 & \\
Bad ($\le 2$) & 4 & 0.250 & \\
\midrule
Overall & 750 & 0.013 & \\
\bottomrule
\end{tabular}
\caption{Infinity-Chat validity-gate tripwire on the same $750$ pairs.
The automated gate fails the protocol if it marks ${>}5\%$ of human-good
responses invalid; the observed rate is $0.25\%$.}
\label{tab:ic-judge-validity}
\end{table}

The validity screen is \emph{highly conservative} with respect to human judgments: only $0.25\%$ of responses that humans rate as good ($\ge 4$) are discarded, well under the $5\%$ tripwire, so the LLM gate does not systematically delete content that human raters endorse. On quality, the judge is \emph{positively associated} with the human consensus (Pearson $r{=}0.28$, ICC${=}0.23$) and lands within one scale point of the rounded human mean on $84\%$ of items, while remaining below the human split-half Spearman ceiling ($\rho{=}0.16$ vs.\ $0.62$). The residual gap is consistent with a mild high-score bias (mean bias $+0.69$) on a distribution where human means already cluster near the top of the scale. We therefore treat the Infinity-Chat \textbf{validity} labels as a credible hard filter for all downstream LLM and embedding metrics, and treat \textbf{quality} scores as a coarsely human-aligned complementary axis, which is useful for relative quality, diversity tradeoffs across variants, and interpreted alongside embedding-based homogeneity and diversity rather than as a calibrated substitute for the full $25$-rater panel.

\subsection{Divergent Association Task (DAT)}
\subsubsection{Baselines and Variants}
\paragraph{Standard persona baselines.}
We use the original DAT prompt \citep{olson2021naming,chen2023probing} as the baseline persona-free prompt: \textsc{DAT}. \cite{schapiro2026creativityneuro} tests various prompts to contrast creative and non-creative outputs, and we also compare their task-optimized prompts as our baseline comparisons: \textsc{CN-Base} (their base instruction prompt that contrasts with random instruction prompt), \textsc{CN-Random} (their random instruction prompt), \textsc{Creativity-enhanced} (their creative prompt that serves as a contrast to their non-creative prompt), \textsc{Non-creative} (their non-creative prompt that serves the similar role as common use prompt of AUT benchmark). 

\paragraph{Other baselines and variants.}
Other reasoning baselines, persona controls (gibberish, \cite{jin2025multi}, and random/typical personas from \cite{jiang2025know,ge2024scaling}) are similarly compared for DAT as well. Our selection and generation (MCMC and Evo) variants use the same set of selected personas used for AUT. (Note that the personas, as well as selection and generation algorithms are agnostic of a specific task). \textsc{DAT-Evolution} is the task-conditioned variant using DAT-specific fitness function
\begin{equation}
    \operatorname{Fit}(p)=\lambda_{\operatorname{DAT}} \operatorname{DAT}_m(\mathcal W_{1:7}) + \lambda_{\operatorname{Flex}} \operatorname{VS}(\mathcal W) + \lambda_H H(\mathcal W) + \lambda_C (1-C_{10}(\mathcal W)),
\end{equation}
where $\mathcal W$ denotes the set of valid responses produced by the persona $p$. 

\paragraph{Generation and sampling protocol.}
Non-persona conditions use $n_{\mathrm{comp}}=35$ independent completions, while persona conditions use seven completions for each of the $k=5$ selected personas. We retain the first $n_{\mathrm{word}}=10$ words from each completion, giving $n_{\mathrm{raw}}=350$ raw words per condition before validity gate.

\subsubsection{Metrics}
For DAT, we report four creativity metrics along with validity ratio: DAT score as suggested by the original literature \citep{olson2021naming}, fluency measured by Shannon entropy, flexibility measured by Vendi score \eqref{eq:vendi}, and top-10 concentration.

\paragraph{Notation.}
For DAT, we denote response as $w$ and a valid response set as $\mathcal W$, as they represent noun words. Let $W_j=(w_{j1},\ldots,w_{jn_j})$ contain the usable validated word occurrences from completion $j$ in response order. Unlike AUT, the relationship between $10$ words in a completion is also an important evaluation target, because the task is to name $10$ irrelevant nouns. Thus, we deal with the pooled multiset $\mathcal W=\biguplus_j W_j$. Let $\mathcal W^{\mathrm{human}}$ be the human reference set of words, drawn from human responses collected by \cite{olson2021naming}.

\paragraph{DAT Score.}
\cite{olson2021naming} invented Divergent Association Task (DAT) along with a creativity metric called DAT score. It is adopted in its original form by following works using DAT \citep{chen2023probing,schapiro2026creativityneuro}. Let $\overline g(w)$ denote its unit-normalized GloVe word embedding. For a scorable completion ($n_j\ge7$),
$$\operatorname{DAT}(W_j) = \frac{100}{\binom{n_j}{2}}\sum_{1\leq a<b\leq n_j}\left[1-\overline g(w_{ja})^\top\overline g(w_{jb})\right].$$
Human-reference percentiles use the same calculation applied to the raw word lists of 8,572 human completions.

\paragraph{Flency.}
Shannon entropy captures fluency on how flat the generated words are distributed: higher entropy, lower duplicate words and thus higher fluency. Therefore, the metric measures how creative/divergent LLMs are
\begin{equation}
    H(\mathcal W) = -\sum_{w\in\mathcal V} \frac{c_w}{\sum_{w\in\mathcal V}c_w}\log{\frac{c_w}{\sum_{w\in\mathcal V}c_w}},
\end{equation}
where $\mathcal V$ is vocabulary or the set of unique words in $\mathcal W$ and $c_w$ is the count of appearance of $w$ in $\mathcal W$.

Top-10 Concentration measures fluency on how much proportion the top-10 most frequent words appear over and over again: lower top-10 share, lower concentration and thus higher fluency. Therefore, the metric measures how lexically sophisticated LLMs are
\begin{equation}
    C_{10}(\mathcal W) = \frac{\sum_{i=1}^{10}c_{(i)}}{\sum_{w\in\mathcal V}c_w},
\end{equation}
where $c_{(i)}$ is the count of appearance $i-$th most frequently appearing word in $\mathcal W$.

\paragraph{Flexibility.}
For each nonempty completion $W_j$, we apply \eqref{eq:vendi} to the unit-normalized EmbeddingGemma embeddings of its individual usable word occurrences, retaining repeated words as repeated embedding rows. The reported Vendi score is the mean of these within-completion scores over nonempty completions. It measures effective semantic diversity within a generated word list, whereas type token ratio (TTR) is the ratio of unique valid words to all valid word occurrences
\begin{equation}
    TTR(\mathcal W)=\frac{|\mathcal V|}{|\mathcal W|}.
\end{equation}

\section{Experiment Results}
\label{app:results}
\subsection{Baselines}
\subsubsection{Common Baselines}
Here we evaluate and compare various metrics on 12 categories of variants: 3 standard persona baselines using reasoning prompts, 2 persona control baselines, 2 PersonaMem-v2 \citep{jiang2025know} baselines, 2 selection algorithm variants, and 3 generation algorithm variants. These are existing baselines using reasoning prompts:
\begin{itemize}[noitemsep, leftmargin=*, topsep=0pt, partopsep=0pt, label={\tiny\raisebox{0.5ex}{$\blacktriangleright$}}]
    \item ZS-CoT: Zero-Shot Chain-of-Thought (ZS-CoT) reasoning prompt adapted from \cite{kojima2022large}
    \item Step-Back: Step-Back Prompting (SBP) adapted from \cite{zheng2024take}
    \item DMAD: Diverse Multi-Agent Debate prompt combining reasoning strategies (ZS-CoT + SBP) adapted from \cite{liu2025breaking}
\end{itemize}

We also compare our methods against control personas:
\begin{itemize}[noitemsep, leftmargin=*, topsep=0pt, partopsep=0pt, label={\tiny\raisebox{0.5ex}{$\blacktriangleright$}}]
    \item Gibberish: gibberish of equivalent length as personas our methods utilize
    \item Task-conditioned: task-conditioned personas generated by \cite{jin2025multi}
\end{itemize}

We compare against existing persona benchmark PersonaMem-v2 dataset \citep{jiang2025know}, which randomly selected from PersonaHub \citep{ge2024scaling}, using two sampling methods:
\begin{itemize}[noitemsep, leftmargin=*, topsep=0pt, partopsep=0pt, label={\tiny\raisebox{0.5ex}{$\blacktriangleright$}}]
    \item Random: randomly selected personas from the base population.
    \item Typical: typical personas selected from the most dense regions.
\end{itemize}

Finally, we apply our diversity-aware selection methods on persona population of \cite{jiang2025know}:
\begin{itemize}[noitemsep, leftmargin=*, topsep=0pt, partopsep=0pt, label={\tiny\raisebox{0.5ex}{$\blacktriangleright$}}]
    \item Coverage: selected personas from the base population to maximize coverage defined by cosine, L2 and Mahalanobis distances.
    \item Dispersion: selected personas from the base population to maximize minimum dispersion defined by cosine, L2 and Mahalanobis distances; Mahalnobis distance-based dispersion selection is chosen as a default for all persona generation algorithms.
\end{itemize}

Then, expand the baseline population using our MCMC and evolutionary persona generation algorithms:
\begin{itemize}[noitemsep, leftmargin=*, topsep=0pt, partopsep=0pt, label={\tiny\raisebox{0.5ex}{$\blacktriangleright$}}]
    \item MCMC: extended population using Uniform-Coverage MCMC persona generation algorithm.
    \item Evolution: extended population using Evolutionary TextGrad algorithm on persona fitness function.
    \item AUT Evolution: extended population using Evolutionary TextGrad algorithm on AUT response fitness function.
\end{itemize}

\subsubsection{AUT Baselines}
These are the baseline and task-optimized prompts from \cite{goes2023pushing}:
\begin{itemize}[noitemsep, leftmargin=*, topsep=0pt, partopsep=0pt, label={\tiny\raisebox{0.5ex}{$\blacktriangleright$}}]
    \item Common Use: common uses of objects
    \item Alternative Use: alternative uses of objects generated by generic AUT prompt
    \item Expert: alternative uses generated by AUT-optimized expert prompt
    \item Creativity-enhanced: alternative uses generated by creativity-enhanced prompt adapted from \cite{goes2023pushing}
\end{itemize}

\subsubsection{DAT Baselines}
These are the prompt conditions for the DAT, following \cite{chen2023probing} and \cite{schapiro2026creativityneuro}:
\begin{itemize}[noitemsep, leftmargin=*, topsep=0pt, partopsep=0pt, label={\tiny\raisebox{0.5ex}{$\blacktriangleright$}}]
    \item \textsc{Divergent Association}: zero-shot task prompt in the original form of \cite{olson2021naming}, adapted by \cite{chen2023probing}.
    \item \textsc{Base-Instruction}: base control prompt \citep{chen2023probing}, which asks for ten nouns with no divergence instruction.
    \item \textsc{Random-Instruction}: random control prompt \citep{chen2023probing}, which asks for ten random nouns.
    \item \textsc{Creative} / Non-Creative (\textsc{Non-Divergent Association}): the contrastive prompt pair of \cite{schapiro2026creativityneuro}, which instruct the model to answer creatively and uncreatively respectively, bracketing the achievable range.
\end{itemize}
Every condition beyond DAT, the reasoning prompts, the control personas, and all of our selection and generation variants, injects into the \textsc{Divergent Association} prompt, which is therefore the reference condition throughout the DAT results. Selection under cosine and $L_2$ dissimilarity returns identical persona sets at every pool, so the two are reported jointly.

\subsection{Persona Diversity}
\label{app:persona_diversity}
\begin{table}[htbp]
\centering\small
\caption{Persona diversity of the variants measured in Dispersion (\eqref{eq:min-dispersion}) and Vendi score (\eqref{eq:vendi}) of each persona set ($k=5$), computed on persona embeddings under cosine, L2, and Mahalanobis dissimilarity. Dispersion represents how personas are spread far from each other (minimum pairwise distance within the set; higher is more spread). Vendi score represents the effective number of distinct clusters of personas (higher is more distinct). Methods are grouped by their algorithm categories: MPAQ is the baseline \citep{jin2025multi} that generated persona for output diversity, Individualized personas represent baseline selection methods from the base pool of size 997, Selected personas represent our methods of persona selection according to various objective functions, and Generated personas represent our methods of persona generation, each selected from its expanded pool by exact Mahalanobis max--min dispersion. \textbf{Bold} = best per column within the selected and generated groups; baseline groups are not bolded.}
\label{tab:persona_diversity_full}
\setlength{\tabcolsep}{5pt}
\begin{tabular}{@{}l ccc ccc@{}}
\toprule
\multirow{2}{*}{Persona set}
  & \multicolumn{3}{c}{Dispersion}
  & \multicolumn{3}{c}{Vendi score} \\
\cmidrule(lr){2-4}\cmidrule(lr){5-7}
  & Cosine & L2 & Mahal. & Cosine & L2 & Mahal. \\
\midrule
\addlinespace
\multicolumn{7}{@{}l}{\textit{Baseline personas}}\\
\textsc{Task-conditioned}                & $0.130$ & $0.510$ & $13.42$ & $1.81$ & $2.21$ & $1.21$ \\
\addlinespace
\multicolumn{7}{@{}l}{\textit{Individualized personas} (base pool $=997$)}\\
\textsc{Typical}             & $0.270$ & $0.734$ & $12.04$ & $2.34$ & $2.73$ & $1.21$ \\
\textsc{Random}              & $0.248$ & $0.703$ & $13.67$ & $2.60$ & $2.92$ & $1.25$ \\
\addlinespace
\multicolumn{7}{@{}l}{\textit{Selected personas} (base pool $=997$)}\\
\quad Coverage, cosine        & $0.231$ & $0.679$ & $12.25$ & $2.22$ & $2.63$ & $1.19$ \\
\quad Coverage, Mahalanobis   & $0.191$ & $0.619$ & $10.52$ & $2.25$ & $2.64$ & $1.15$ \\
\quad Dispersion, cosine      & $\mathbf{0.485}$ & $\mathbf{0.984}$ & $16.05$ & $\mathbf{3.41}$ & $\mathbf{3.41}$ & $1.38$ \\
\quad Dispersion, Mahalanobis  & $0.366$ & $0.856$ & $\mathbf{20.98}$ & $3.07$ & $3.23$ & $\mathbf{1.46}$ \\
\addlinespace
\multicolumn{7}{@{}l}{\textit{Generated personas} (expanded pool size)} \\
\textsc{MCMC} (8,553)           & $0.352$ & $0.839$ & $20.98$ & $3.04$ & $3.21$ & $1.46$ \\
\textsc{Evolution} (1,869)      & $0.393$ & $0.887$ & $\mathbf{24.38}$ & $3.19$ & $3.30$ & $\mathbf{1.64}$ \\
\textsc{AUT-Evolution} (3,397)  & $0.373$ & $0.864$ & $22.34$ & $3.12$ & $3.26$ & $1.51$ \\
\textsc{DAT-Evolution} (2,439) & $\mathbf{0.436}$ & $\mathbf{0.934}$ & $21.61$ & $\mathbf{3.28}$ & $\mathbf{3.34}$ & $1.49$ \\
\bottomrule
\end{tabular}
\end{table}

\begin{figure*}[t]
    \centering
    \includegraphics[width=\textwidth]{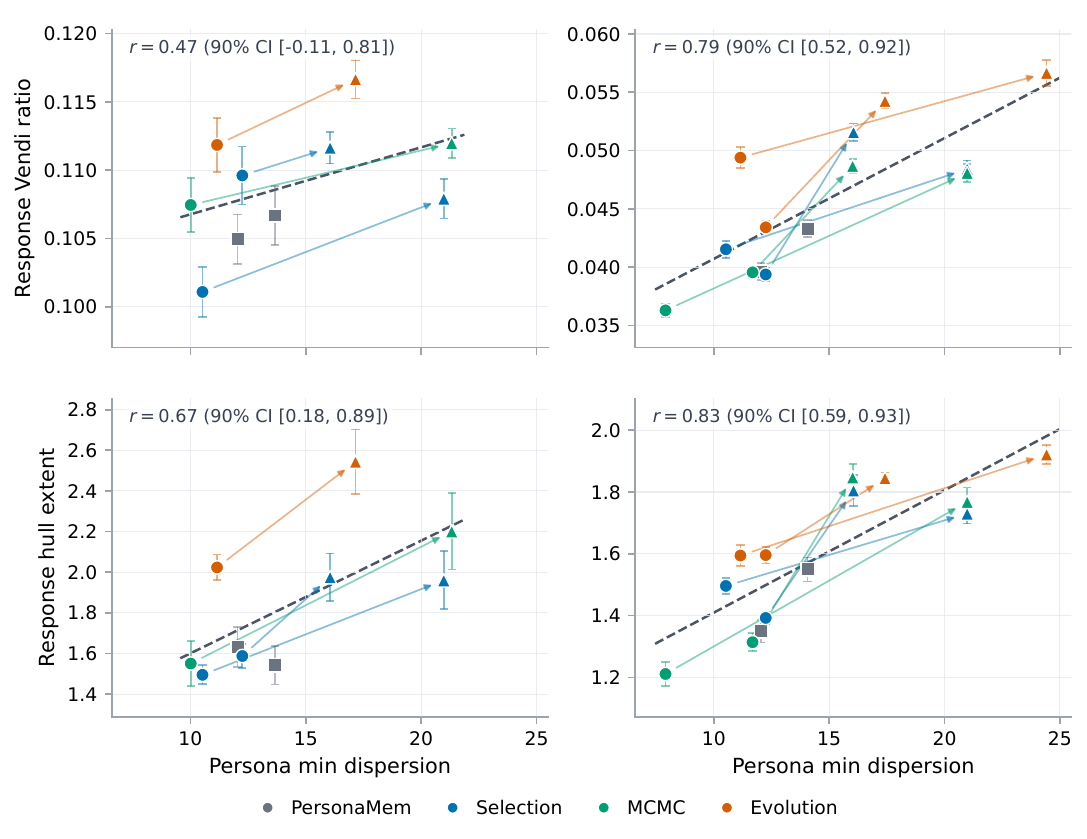}
    \caption{Persona-to-response relationship on (a) the Alternative Uses Task and (b) Infinity-Chat. Points represent 14 analogous task-agnostic method configurations per benchmark; response hull extent and vendi score are averaged over seven AUT objects and five Infinity-Chat queries. Arrows connect six matched Coverage$\rightarrow$Dispersion contrasts while holding candidate-pool family and selector metric fixed. Dashed lines show configuration-level linear fits; annotations report Pearson correlations and descriptive 95\% intervals.}
    \label{fig:persona-response-hull-vendi}
\end{figure*}

\clearpage
\subsection{AUT Benchmark}

\begin{table}[htbp]
\centering\small
\caption{Response validity: the proportion of generated uses admitted by the validity gate of \Cref{sec:aut}, reported over all seven AUT objects and, for comparability with \textsc{Stevenson}, over the three objects with human reference data ($175$ and $75$ generated uses per variant, respectively). Every variants beyond \textsc{Alternative-Use} build prompts upon the generic \textsc{Alternative-Use} prompt unless its group states otherwise; the final two groups re-run our selection and generation methods on the \textsc{Creativity-enhanced} prompt of \citet{goes2023pushing}, showing that persona diversification composes with a stronger prompting strategy rather than competing with it. Cells give the admitted proportion averaged over objects. Subscripts denote the 95\% confidence interval across 5 independent random seeds. \textsc{Stevenson} exists only for the human-reference objects, so its seven-object cell is omitted (---). Validity is a \emph{control} rather than a target: values near $1$ are expected and the informative signal is degradation rather than rank, so no value is bolded. Generated variants report their expanded pool size in parentheses.}
\label{tab:aut_response_validity}
\begin{tabular}{lcc}
\toprule
Variant & All seven objects & Human-reference objects \\
\midrule
\multicolumn{3}{@{}l}{\textit{Standard persona}} \\
\textsc{Common-Use} & $0.971_{\pm 0.018}$ & $0.981_{\pm 0.015}$ \\
\textsc{Alternative-Use} & $0.984_{\pm 0.014}$ & $0.968_{\pm 0.025}$ \\
\textsc{Expert} & $0.997_{\pm 0.006}$ & $0.997_{\pm 0.007}$ \\
\textsc{Creativity-enhanced} & $0.968_{\pm 0.011}$ & $0.981_{\pm 0.022}$ \\
\textsc{ZS-CoT} & $0.978_{\pm 0.009}$ & $0.997_{\pm 0.007}$ \\
\textsc{Step-Back} & $0.976_{\pm 0.014}$ & $0.968_{\pm 0.015}$ \\
\textsc{DMAD} & $0.977_{\pm 0.009}$ & $0.997_{\pm 0.007}$ \\
\textsc{Gibberish} & $0.693_{\pm 0.035}$ & $0.659_{\pm 0.125}$ \\
\addlinespace
\multicolumn{3}{@{}l}{\textit{Baseline personas}} \\
\textsc{MPAQ} & $0.930_{\pm 0.024}$ & $0.947_{\pm 0.020}$ \\
\addlinespace
\multicolumn{3}{@{}l}{\textit{Individualized personas}} \\
\textsc{Random} & $0.993_{\pm 0.006}$ & $0.989_{\pm 0.014}$ \\
\textsc{Typical} & $0.995_{\pm 0.006}$ & $0.997_{\pm 0.007}$ \\
\addlinespace
\multicolumn{3}{@{}l}{\textit{Selected personas from base pool}} \\
\quad Coverage, cosine & $0.987_{\pm 0.006}$ & $1.000_{\pm 0.000}$ \\
\quad Coverage, Mahalanobis & $0.993_{\pm 0.013}$ & $1.000_{\pm 0.000}$ \\
\quad Dispersion, cosine & $0.992_{\pm 0.012}$ & $0.992_{\pm 0.015}$ \\
\quad Dispersion, Mahalanobis & $0.986_{\pm 0.010}$ & $0.995_{\pm 0.009}$ \\
\addlinespace
\multicolumn{3}{@{}l}{\textit{Generated personas} (expanded pool size)} \\
\quad MCMC (8,553) & $0.984_{\pm 0.009}$ & $0.995_{\pm 0.009}$ \\
\quad Evolution (1,869) & $0.985_{\pm 0.012}$ & $0.987_{\pm 0.017}$ \\
\quad AUT-Evolution (3,397) & $0.978_{\pm 0.017}$ & $0.989_{\pm 0.014}$ \\
\addlinespace
\multicolumn{3}{@{}l}{\textit{Selected personas} $+$ \textsc{Creativity-enhanced} prompt} \\
\quad Coverage, cosine & $0.978_{\pm 0.012}$ & $0.981_{\pm 0.015}$ \\
\quad Coverage, Mahalanobis & $0.989_{\pm 0.014}$ & $0.992_{\pm 0.009}$ \\
\quad Dispersion, cosine & $0.982_{\pm 0.008}$ & $0.979_{\pm 0.019}$ \\
\quad Dispersion, Mahalanobis & $0.979_{\pm 0.011}$ & $0.976_{\pm 0.030}$ \\
\addlinespace
\multicolumn{3}{@{}l}{\textit{Generated personas} $+$ \textsc{Creativity-enhanced} prompt} \\
\quad MCMC & $0.967_{\pm 0.014}$ & $0.968_{\pm 0.040}$ \\
\quad Evolution & $0.979_{\pm 0.020}$ & $0.973_{\pm 0.031}$ \\
\quad AUT-Evolution & $0.973_{\pm 0.006}$ & $0.968_{\pm 0.034}$ \\
\addlinespace
\multicolumn{3}{@{}l}{\textit{Human}} \\
\textsc{Stevenson} & --- & $0.811_{\pm 0.007}$ \\
\bottomrule
\end{tabular}
\end{table}



\begin{table}[htbp]
\centering\small
\setlength{\tabcolsep}{4pt}
\caption{Response diversity across all seven AUT objects, reported two ways. \emph{Mean pairwise distance} (\eqref{eq:diversity}) is the average cosine, L2 or Mahalanobis distance among the valid uses of a variant; \emph{convex hull volume} (\eqref{eq:diversity}) is the volume of the convex hull those uses span in the same geometry, so it rewards a set that occupies a region rather than merely separating in the mean. Both are averaged over objects. Subscripts denote the 95\% confidence interval across 5 independent random seeds, and higher is more diverse. Hull extent under cosine and L2 coincides to nine decimal places, the two metrics are monotone transforms of one another on unit-normalized embeddings, so it is reported once. Baseline and our method variants are grouped by the personas: standard persona represents LLMs default persona when no persona is given explicitly, baseline personas represent existing persona injection methods, individualized personas represent broad population of personas from PersonaHub \cite{ge2024scaling, jiang2025know}, while selected personas and generated personas represent our specific persona selection and generation algorithms. The final two groups re-run our selection and generation methods on the \textsc{Creativity-enhanced} prompt, showing that persona diversification composes with a stronger prompting strategy. Bold marks the best value per column within a method category, with tied values bolded jointly. All personas, including gibberish, are injected to the alternative use prompt unless their group states otherwise.}
\label{tab:aut_diversity_full}
\begin{tabular}{lccccc}
\toprule
 & \multicolumn{3}{c}{Mean pairwise distance} & \multicolumn{2}{c}{Convex hull volume} \\
\cmidrule(lr){2-4}\cmidrule(lr){5-6}
Variant & Cosine & L2 & Mahal. & Cosine/L2 & Mahal. \\
\midrule
\multicolumn{6}{@{}l}{\textit{Standard persona}} \\
\textsc{Common-Use} & $0.155_{\pm 0.004}$ & $0.526_{\pm 0.006}$ & $15.50_{\pm 0.35}$ & $0.103_{\pm 0.004}$ & $2.17_{\pm 0.23}$ \\
\textsc{Alternative-Use} & $0.188_{\pm 0.003}$ & $0.590_{\pm 0.006}$ & $11.97_{\pm 0.28}$ & $0.172_{\pm 0.014}$ & $1.42_{\pm 0.09}$ \\
\textsc{Expert} & $0.179_{\pm 0.004}$ & $0.580_{\pm 0.006}$ & $14.57_{\pm 0.29}$ & $0.161_{\pm 0.004}$ & $2.36_{\pm 0.13}$ \\
\textsc{Creativity-enhanced} & $0.179_{\pm 0.002}$ & $0.592_{\pm 0.004}$ & $\mathbf{20.24}_{\pm 0.56}$ & $0.144_{\pm 0.010}$ & $\mathbf{5.15}_{\pm 0.16}$ \\
\textsc{ZS-CoT} & $0.165_{\pm 0.003}$ & $0.563_{\pm 0.006}$ & $14.73_{\pm 0.38}$ & $0.161_{\pm 0.006}$ & $2.43_{\pm 0.11}$ \\
\textsc{Step-Back} & $0.173_{\pm 0.006}$ & $0.574_{\pm 0.014}$ & $13.66_{\pm 0.52}$ & $0.168_{\pm 0.011}$ & $2.03_{\pm 0.29}$ \\
\textsc{DMAD} & $0.178_{\pm 0.005}$ & $0.588_{\pm 0.009}$ & $16.23_{\pm 0.46}$ & $0.167_{\pm 0.007}$ & $2.79_{\pm 0.22}$ \\
\textsc{Gibberish} & $\mathbf{0.233}_{\pm 0.005}$ & $\mathbf{0.672}_{\pm 0.006}$ & $19.83_{\pm 0.82}$ & $\mathbf{0.184}_{\pm 0.008}$ & $3.14_{\pm 0.34}$ \\
\addlinespace
\multicolumn{6}{@{}l}{\textit{Baseline personas}} \\
\textsc{MPAQ} & $0.216_{\pm 0.002}$ & $0.652_{\pm 0.003}$ & $19.05_{\pm 0.49}$ & $0.186_{\pm 0.003}$ & $3.46_{\pm 0.27}$ \\
\addlinespace
\multicolumn{6}{@{}l}{\textit{Individualized personas}} \\
\textsc{Random} & $0.202_{\pm 0.004}$ & $0.621_{\pm 0.007}$ & $12.59_{\pm 0.63}$ & $0.202_{\pm 0.004}$ & $1.54_{\pm 0.12}$ \\
\textsc{Typical} & $0.197_{\pm 0.006}$ & $0.615_{\pm 0.010}$ & $12.78_{\pm 0.26}$ & $0.204_{\pm 0.003}$ & $1.63_{\pm 0.12}$ \\
\addlinespace
\multicolumn{6}{@{}l}{\textit{Selected personas from base pool}} \\
\quad Coverage, cosine/L2 & $0.207_{\pm 0.006}$ & $0.629_{\pm 0.008}$ & $12.71_{\pm 0.28}$ & $0.203_{\pm 0.007}$ & $1.58_{\pm 0.08}$ \\
\quad Coverage, Mahalanobis & $0.193_{\pm 0.003}$ & $0.604_{\pm 0.005}$ & $11.77_{\pm 0.20}$ & $0.200_{\pm 0.007}$ & $1.50_{\pm 0.06}$ \\
\quad Dispersion, cosine/L2 & $\mathbf{0.208}_{\pm 0.003}$ & $\mathbf{0.633}_{\pm 0.004}$ & $\mathbf{13.36}_{\pm 0.18}$ & $\mathbf{0.213}_{\pm 0.005}$ & $\mathbf{1.98}_{\pm 0.14}$ \\
\quad Dispersion, Mahalanobis & $0.202_{\pm 0.004}$ & $0.622_{\pm 0.005}$ & $12.95_{\pm 0.35}$ & $0.210_{\pm 0.005}$ & $1.96_{\pm 0.17}$ \\
\addlinespace
\multicolumn{6}{@{}l}{\textit{Generated personas} (expanded pool size)} \\
\quad MCMC (8,553) & $0.207_{\pm 0.004}$ & $0.631_{\pm 0.006}$ & $14.04_{\pm 0.59}$ & $0.211_{\pm 0.001}$ & $2.21_{\pm 0.23}$ \\
\quad Evolution (1,869) & $0.214_{\pm 0.005}$ & $0.644_{\pm 0.007}$ & $15.47_{\pm 0.39}$ & $0.210_{\pm 0.005}$ & $2.53_{\pm 0.20}$ \\
\quad AUT-Evolution (3,397) & $\mathbf{0.216}_{\pm 0.006}$ & $\mathbf{0.646}_{\pm 0.009}$ & $\mathbf{15.66}_{\pm 0.21}$ & $\mathbf{0.218}_{\pm 0.007}$ & $\mathbf{2.98}_{\pm 0.16}$ \\
\addlinespace
\multicolumn{6}{@{}l}{\textit{Selected personas} $+$ \textsc{Creativity-enhanced} prompt} \\
\quad Coverage, cosine/L2 & $0.190_{\pm 0.005}$ & $0.610_{\pm 0.009}$ & $20.44_{\pm 0.32}$ & $0.150_{\pm 0.006}$ & $5.17_{\pm 0.21}$ \\
\quad Coverage, Mahalanobis & $0.190_{\pm 0.003}$ & $0.609_{\pm 0.004}$ & $20.45_{\pm 0.17}$ & $0.154_{\pm 0.002}$ & $5.22_{\pm 0.15}$ \\
\quad Dispersion, cosine/L2 & $\mathbf{0.195}_{\pm 0.005}$ & $\mathbf{0.618}_{\pm 0.009}$ & $20.57_{\pm 0.20}$ & $\mathbf{0.159}_{\pm 0.007}$ & $6.43_{\pm 0.22}$ \\
\quad Dispersion, Mahalanobis & $0.191_{\pm 0.004}$ & $0.611_{\pm 0.008}$ & $\mathbf{20.57}_{\pm 0.51}$ & $0.158_{\pm 0.007}$ & $\mathbf{6.73}_{\pm 0.29}$ \\
\addlinespace
\multicolumn{6}{@{}l}{\textit{Generated personas} $+$ \textsc{Creativity-enhanced} prompt} \\
\quad MCMC & $0.193_{\pm 0.004}$ & $0.615_{\pm 0.008}$ & $21.16_{\pm 0.41}$ & $0.156_{\pm 0.005}$ & $6.75_{\pm 0.21}$ \\
\quad Evolution & $0.197_{\pm 0.005}$ & $0.621_{\pm 0.009}$ & $21.20_{\pm 0.25}$ & $\mathbf{0.157}_{\pm 0.005}$ & $6.11_{\pm 0.22}$ \\
\quad AUT-Evolution & $\mathbf{0.199}_{\pm 0.005}$ & $\mathbf{0.625}_{\pm 0.008}$ & $\mathbf{21.86}_{\pm 0.65}$ & $0.156_{\pm 0.005}$ & $\mathbf{7.00}_{\pm 0.26}$ \\
\bottomrule
\end{tabular}
\end{table}

\begin{table}[t]
\centering\small
\setlength{\tabcolsep}{4pt}
\caption{Response diversity restricted to the three objects with human reference data (book, fork, tin can), so that all variants are directly comparable to \textsc{Stevenson}. Cells, bolding, and column conventions are as in Table~\ref{tab:aut_diversity_full}.}
\label{tab:aut_diversity_human}
\begin{tabular}{lccccc}
\toprule
 & \multicolumn{3}{c}{Mean pairwise distance} & \multicolumn{2}{c}{Convex hull extent} \\
\cmidrule(lr){2-4}\cmidrule(lr){5-6}
Variant & Cosine & L2 & Mahal. & Cosine/L2 & Mahal. \\
\midrule
\multicolumn{6}{@{}l}{\textit{Standard persona}} \\
\textsc{Common-Use} & $0.182_{\pm 0.005}$ & $0.583_{\pm 0.010}$ & $15.49_{\pm 0.66}$ & $0.124_{\pm 0.009}$ & $2.16_{\pm 0.38}$ \\
\textsc{Alternative-Use} & $\mathbf{0.195}_{\pm 0.003}$ & $\mathbf{0.601}_{\pm 0.005}$ & $11.38_{\pm 0.39}$ & $0.173_{\pm 0.012}$ & $1.20_{\pm 0.13}$ \\
\textsc{Expert} & $0.185_{\pm 0.005}$ & $0.590_{\pm 0.005}$ & $13.76_{\pm 0.32}$ & $\mathbf{0.170}_{\pm 0.012}$ & $2.01_{\pm 0.26}$ \\
\textsc{Creativity-enhanced} & $0.174_{\pm 0.006}$ & $0.583_{\pm 0.009}$ & $\mathbf{19.54}_{\pm 0.31}$ & $0.138_{\pm 0.006}$ & $\mathbf{5.06}_{\pm 0.40}$ \\
\textsc{ZS-CoT} & $0.162_{\pm 0.006}$ & $0.556_{\pm 0.011}$ & $13.62_{\pm 0.46}$ & $0.159_{\pm 0.014}$ & $1.99_{\pm 0.14}$ \\
\textsc{Step-Back} & $0.180_{\pm 0.004}$ & $0.584_{\pm 0.009}$ & $12.79_{\pm 0.20}$ & $0.162_{\pm 0.012}$ & $1.77_{\pm 0.18}$ \\
\textsc{DMAD} & $0.178_{\pm 0.009}$ & $0.585_{\pm 0.016}$ & $14.98_{\pm 0.77}$ & $0.168_{\pm 0.008}$ & $2.28_{\pm 0.19}$ \\
\textsc{Gibberish} & $0.235_{\pm 0.009}$ & $0.675_{\pm 0.016}$ & $20.50_{\pm 2.58}$ & $0.178_{\pm 0.013}$ & $3.27_{\pm 0.63}$ \\
\addlinespace
\multicolumn{6}{@{}l}{\textit{Baseline personas}} \\
\textsc{MPAQ} & $0.216_{\pm 0.007}$ & $0.651_{\pm 0.011}$ & $17.70_{\pm 0.59}$ & $0.198_{\pm 0.008}$ & $2.92_{\pm 0.38}$ \\
\addlinespace
\multicolumn{6}{@{}l}{\textit{Individualized personas}} \\
\textsc{Random} & $0.206_{\pm 0.004}$ & $0.629_{\pm 0.007}$ & $12.52_{\pm 0.66}$ & $0.213_{\pm 0.008}$ & $1.48_{\pm 0.15}$ \\
\textsc{Typical} & $0.206_{\pm 0.006}$ & $0.629_{\pm 0.010}$ & $12.70_{\pm 0.31}$ & $0.216_{\pm 0.013}$ & $1.58_{\pm 0.21}$ \\
\addlinespace
\multicolumn{6}{@{}l}{\textit{Selected personas from base pool}} \\
\quad Coverage, cosine/L2 & $\mathbf{0.214}_{\pm 0.008}$ & $\mathbf{0.639}_{\pm 0.013}$ & $12.66_{\pm 0.63}$ & $0.214_{\pm 0.008}$ & $1.57_{\pm 0.16}$ \\
\quad Coverage, Mahalanobis & $0.205_{\pm 0.008}$ & $0.624_{\pm 0.016}$ & $11.90_{\pm 0.49}$ & $0.214_{\pm 0.009}$ & $1.52_{\pm 0.15}$ \\
\quad Dispersion, cosine/L2 & $0.209_{\pm 0.002}$ & $0.638_{\pm 0.003}$ & $\mathbf{13.11}_{\pm 0.45}$ & $0.220_{\pm 0.005}$ & $1.87_{\pm 0.33}$ \\
\quad Dispersion, Mahalanobis & $0.208_{\pm 0.003}$ & $0.634_{\pm 0.004}$ & $12.87_{\pm 0.39}$ & $\mathbf{0.224}_{\pm 0.003}$ & $\mathbf{1.95}_{\pm 0.32}$ \\
\addlinespace
\multicolumn{6}{@{}l}{\textit{Generated personas} (expanded pool size)} \\
\quad MCMC (8,553) & $0.211_{\pm 0.009}$ & $0.639_{\pm 0.014}$ & $13.92_{\pm 1.35}$ & $0.223_{\pm 0.007}$ & $2.21_{\pm 0.48}$ \\
\quad Evolution (1,869) & $0.214_{\pm 0.008}$ & $0.646_{\pm 0.012}$ & $15.41_{\pm 0.61}$ & $0.212_{\pm 0.003}$ & $2.56_{\pm 0.41}$ \\
\quad AUT-Evolution (3,397) & $\mathbf{0.222}_{\pm 0.006}$ & $\mathbf{0.658}_{\pm 0.011}$ & $\mathbf{16.03}_{\pm 0.56}$ & $\mathbf{0.227}_{\pm 0.007}$ & $\mathbf{3.08}_{\pm 0.34}$ \\
\addlinespace
\multicolumn{6}{@{}l}{\textit{Selected personas} $+$ \textsc{Creativity-enhanced} prompt} \\
\quad Coverage, cosine/L2 & $0.186_{\pm 0.005}$ & $0.602_{\pm 0.009}$ & $19.84_{\pm 0.60}$ & $0.145_{\pm 0.006}$ & $5.08_{\pm 0.32}$ \\
\quad Coverage, Mahalanobis & $0.186_{\pm 0.005}$ & $0.604_{\pm 0.008}$ & $19.94_{\pm 0.39}$ & $0.144_{\pm 0.005}$ & $5.22_{\pm 0.25}$ \\
\quad Dispersion, cosine/L2 & $\mathbf{0.190}_{\pm 0.005}$ & $\mathbf{0.608}_{\pm 0.010}$ & $19.85_{\pm 0.46}$ & $\mathbf{0.153}_{\pm 0.010}$ & $6.17_{\pm 0.52}$ \\
\quad Dispersion, Mahalanobis & $0.184_{\pm 0.002}$ & $0.600_{\pm 0.004}$ & $\mathbf{20.18}_{\pm 0.60}$ & $0.148_{\pm 0.005}$ & $\mathbf{6.59}_{\pm 0.38}$ \\
\addlinespace
\multicolumn{6}{@{}l}{\textit{Generated personas} $+$ \textsc{Creativity-enhanced} prompt} \\
\quad MCMC & $0.186_{\pm 0.003}$ & $0.604_{\pm 0.006}$ & $20.62_{\pm 0.24}$ & $0.147_{\pm 0.004}$ & $6.58_{\pm 0.28}$ \\
\quad Evolution & $0.190_{\pm 0.005}$ & $0.610_{\pm 0.009}$ & $20.66_{\pm 0.73}$ & $0.148_{\pm 0.011}$ & $5.98_{\pm 0.49}$ \\
\quad AUT-Evolution & $\mathbf{0.196}_{\pm 0.006}$ & $\mathbf{0.620}_{\pm 0.011}$ & $\mathbf{21.48}_{\pm 0.75}$ & $\mathbf{0.148}_{\pm 0.003}$ & $\mathbf{6.92}_{\pm 0.40}$ \\
\addlinespace
\multicolumn{6}{@{}l}{\textit{Human}} \\
\textsc{Stevenson} & $0.149_{\pm 0.009}$ & $0.539_{\pm 0.015}$ & $17.80_{\pm 0.93}$ & $0.101_{\pm 0.006}$ & $2.26_{\pm 0.34}$ \\
\bottomrule
\end{tabular}
\end{table}

\begin{table}[htbp]
\centering\small
\setlength{\tabcolsep}{4pt}
\caption{Response originality across all seven AUT objects, measured both as distance to the prompted object and as directed distance to the common-use set, following \eqref{eq:originality}. Cells give the mean over valid uses, averaged over objects; higher is more original. \textsc{Common-Use} defines the common-use reference set $C_o$, so its distance to that set is zero by construction and is omitted (---). Subscripts denote the 95\% confidence interval across 5 independent random seeds. The final two groups re-run our selection and generation methods on the \textsc{Creativity-enhanced} prompt, showing that persona diversification composes with a stronger prompting strategy. Bold marks the best value per column within the standard, selected and generated groups separately, with tied values bolded jointly; baseline and individualized personas are not bolded. Generated variants report their expanded pool size in parentheses.}
\label{tab:aut_response_originality_full}
\resizebox{\textwidth}{!}{%
\begin{tabular}{lcccccc}
\toprule
 & \multicolumn{3}{c}{To object} & \multicolumn{3}{c}{To common uses} \\
\cmidrule(lr){2-4}\cmidrule(lr){5-7}
Variant & Cosine & L2 & Mahal. & Cosine & L2 & Mahal. \\
\midrule
\multicolumn{7}{@{}l}{\textit{Standard persona}} \\
\textsc{Common-Use} & $0.142_{\pm 0.002}$ & $0.525_{\pm 0.004}$ & $15.28_{\pm 0.16}$ & --- & --- & --- \\
\textsc{Alternative-Use} & $0.170_{\pm 0.001}$ & $0.580_{\pm 0.002}$ & $13.60_{\pm 0.14}$ & $0.114_{\pm 0.008}$ & $0.439_{\pm 0.019}$ & $10.79_{\pm 0.32}$ \\
\textsc{Expert} & $0.164_{\pm 0.002}$ & $0.568_{\pm 0.004}$ & $14.72_{\pm 0.17}$ & $0.104_{\pm 0.003}$ & $0.417_{\pm 0.010}$ & $11.25_{\pm 0.40}$ \\
\textsc{Creativity-enhanced} & $0.173_{\pm 0.004}$ & $0.586_{\pm 0.007}$ & $\mathbf{18.13}_{\pm 0.35}$ & $\mathbf{0.174}_{\pm 0.006}$ & $\mathbf{0.585}_{\pm 0.011}$ & $\mathbf{17.15}_{\pm 0.40}$ \\
\textsc{ZS-CoT} & $0.146_{\pm 0.002}$ & $0.536_{\pm 0.004}$ & $14.67_{\pm 0.10}$ & $0.125_{\pm 0.003}$ & $0.483_{\pm 0.007}$ & $13.02_{\pm 0.43}$ \\
\textsc{Step-Back} & $0.165_{\pm 0.002}$ & $0.571_{\pm 0.003}$ & $14.91_{\pm 0.18}$ & $0.129_{\pm 0.002}$ & $0.488_{\pm 0.007}$ & $12.86_{\pm 0.20}$ \\
\textsc{DMAD} & $0.151_{\pm 0.004}$ & $0.547_{\pm 0.007}$ & $15.19_{\pm 0.19}$ & $0.135_{\pm 0.006}$ & $0.504_{\pm 0.017}$ & $13.90_{\pm 0.49}$ \\
\textsc{Gibberish} & $\mathbf{0.193}_{\pm 0.002}$ & $\mathbf{0.618}_{\pm 0.004}$ & $17.72_{\pm 0.52}$ & $0.156_{\pm 0.009}$ & $0.541_{\pm 0.018}$ & $15.89_{\pm 0.68}$ \\
\addlinespace
\multicolumn{7}{@{}l}{\textit{Baseline personas}} \\
\textsc{MPAQ} & $0.180_{\pm 0.001}$ & $0.596_{\pm 0.002}$ & $16.95_{\pm 0.28}$ & $0.162_{\pm 0.002}$ & $0.560_{\pm 0.004}$ & $15.46_{\pm 0.34}$ \\
\addlinespace
\multicolumn{7}{@{}l}{\textit{Individualized personas}} \\
\textsc{Random} & $0.184_{\pm 0.004}$ & $0.603_{\pm 0.007}$ & $13.90_{\pm 0.31}$ & $0.143_{\pm 0.007}$ & $0.516_{\pm 0.015}$ & $11.89_{\pm 0.48}$ \\
\textsc{Typical} & $0.179_{\pm 0.003}$ & $0.595_{\pm 0.005}$ & $13.91_{\pm 0.13}$ & $0.139_{\pm 0.005}$ & $0.509_{\pm 0.011}$ & $11.96_{\pm 0.30}$ \\
\addlinespace
\multicolumn{7}{@{}l}{\textit{Selected personas from base pool}} \\
\quad Coverage, cosine & $\mathbf{0.188}_{\pm 0.005}$ & $\mathbf{0.609}_{\pm 0.008}$ & $13.91_{\pm 0.14}$ & $\mathbf{0.145}_{\pm 0.004}$ & $\mathbf{0.520}_{\pm 0.008}$ & $12.03_{\pm 0.37}$ \\
\quad Coverage, Mahalanobis & $0.179_{\pm 0.001}$ & $0.594_{\pm 0.002}$ & $13.51_{\pm 0.11}$ & $0.133_{\pm 0.006}$ & $0.495_{\pm 0.012}$ & $11.46_{\pm 0.28}$ \\
\quad Dispersion, cosine & $0.184_{\pm 0.002}$ & $0.603_{\pm 0.003}$ & $\mathbf{14.20}_{\pm 0.05}$ & $0.146_{\pm 0.004}$ & $0.524_{\pm 0.010}$ & $\mathbf{12.34}_{\pm 0.27}$ \\
\quad Dispersion, Mahalanobis & $0.181_{\pm 0.003}$ & $0.598_{\pm 0.004}$ & $13.98_{\pm 0.18}$ & $0.138_{\pm 0.005}$ & $0.506_{\pm 0.009}$ & $12.00_{\pm 0.28}$ \\
\addlinespace
\multicolumn{7}{@{}l}{\textit{Generated personas} (expanded pool size)} \\
\quad MCMC (8,553) & $0.182_{\pm 0.002}$ & $0.600_{\pm 0.003}$ & $14.51_{\pm 0.30}$ & $0.144_{\pm 0.005}$ & $0.519_{\pm 0.014}$ & $12.66_{\pm 0.55}$ \\
\quad Evolution (1,869) & $\mathbf{0.190}_{\pm 0.004}$ & $\mathbf{0.612}_{\pm 0.007}$ & $15.22_{\pm 0.19}$ & $\mathbf{0.162}_{\pm 0.003}$ & $\mathbf{0.556}_{\pm 0.006}$ & $\mathbf{13.61}_{\pm 0.17}$ \\
\quad AUT-Evolution (3,397) & $0.188_{\pm 0.002}$ & $0.610_{\pm 0.003}$ & $\mathbf{15.29}_{\pm 0.12}$ & $0.156_{\pm 0.004}$ & $0.544_{\pm 0.006}$ & $13.58_{\pm 0.11}$ \\
\addlinespace
\multicolumn{7}{@{}l}{\textit{Selected personas} $+$ \textsc{Creativity-enhanced} prompt} \\
\quad Coverage, cosine & $\mathbf{0.187}_{\pm 0.003}$ & $\mathbf{0.609}_{\pm 0.005}$ & $18.45_{\pm 0.20}$ & $\mathbf{0.181}_{\pm 0.003}$ & $\mathbf{0.598}_{\pm 0.005}$ & $17.25_{\pm 0.12}$ \\
\quad Coverage, Mahalanobis & $0.184_{\pm 0.003}$ & $0.603_{\pm 0.005}$ & $18.33_{\pm 0.11}$ & $0.177_{\pm 0.004}$ & $0.590_{\pm 0.006}$ & $17.20_{\pm 0.15}$ \\
\quad Dispersion, cosine & $0.185_{\pm 0.005}$ & $0.604_{\pm 0.008}$ & $18.31_{\pm 0.27}$ & $0.179_{\pm 0.007}$ & $0.593_{\pm 0.013}$ & $\mathbf{17.24}_{\pm 0.26}$ \\
\quad Dispersion, Mahalanobis & $0.182_{\pm 0.005}$ & $0.600_{\pm 0.008}$ & $\mathbf{18.36}_{\pm 0.36}$ & $0.178_{\pm 0.007}$ & $0.592_{\pm 0.011}$ & $17.22_{\pm 0.39}$ \\
\addlinespace
\multicolumn{7}{@{}l}{\textit{Generated personas} $+$ \textsc{Creativity-enhanced} prompt} \\
\quad MCMC & $0.183_{\pm 0.005}$ & $0.602_{\pm 0.009}$ & $18.71_{\pm 0.22}$ & $0.182_{\pm 0.006}$ & $0.598_{\pm 0.009}$ & $17.64_{\pm 0.21}$ \\
\quad Evolution & $\mathbf{0.187}_{\pm 0.005}$ & $\mathbf{0.608}_{\pm 0.008}$ & $18.80_{\pm 0.17}$ & $0.183_{\pm 0.006}$ & $0.600_{\pm 0.011}$ & $17.65_{\pm 0.14}$ \\
\quad AUT-Evolution & $\mathbf{0.187}_{\pm 0.006}$ & $\mathbf{0.608}_{\pm 0.009}$ & $\mathbf{19.14}_{\pm 0.41}$ & $\mathbf{0.186}_{\pm 0.009}$ & $\mathbf{0.605}_{\pm 0.014}$ & $\mathbf{18.02}_{\pm 0.39}$ \\
\bottomrule
\end{tabular}%
}
\end{table}

\begin{table}[t]
\centering\small
\setlength{\tabcolsep}{4pt}
\caption{Response originality restricted to the three objects with human reference data (book, fork, tin can), so that all variants are directly comparable to \textsc{Stevenson}. Cells, bolding, and column conventions are as in Table~\ref{tab:aut_response_originality_full}. \textsc{Common-Use} defines the common-use reference set $C_o$, so its distance to that set is zero by construction and is omitted (---).}
\label{tab:aut_response_originality_human}
\resizebox{\textwidth}{!}{%
\begin{tabular}{lcccccc}
\toprule
 & \multicolumn{3}{c}{To object} & \multicolumn{3}{c}{To common uses} \\
\cmidrule(lr){2-4}\cmidrule(lr){5-7}
Variant & Cosine & L2 & Mahal. & Cosine & L2 & Mahal. \\
\midrule
\multicolumn{7}{@{}l}{\textit{Standard persona}} \\
\textsc{Common-Use} & $0.182_{\pm 0.005}$ & $0.583_{\pm 0.010}$ & $15.49_{\pm 0.66}$ & --- & --- & --- \\
\textsc{Alternative-Use} & $\mathbf{0.195}_{\pm 0.003}$ & $\mathbf{0.601}_{\pm 0.005}$ & $11.38_{\pm 0.39}$ & $0.173_{\pm 0.012}$ & $1.20_{\pm 0.13}$ & $10.79_{\pm 0.32}$ \\
\textsc{Expert} & $0.185_{\pm 0.005}$ & $0.590_{\pm 0.005}$ & $13.76_{\pm 0.32}$ & $\mathbf{0.170}_{\pm 0.012}$ & $2.01_{\pm 0.26}$ & $11.25_{\pm 0.40}$ \\
\textsc{Creativity-enhanced} & $0.174_{\pm 0.006}$ & $0.583_{\pm 0.009}$ & $\mathbf{19.54}_{\pm 0.31}$ & $0.138_{\pm 0.006}$ & $\mathbf{5.06}_{\pm 0.40}$ & $\mathbf{17.15}_{\pm 0.40}$ \\
\textsc{ZS-CoT} & $0.162_{\pm 0.006}$ & $0.556_{\pm 0.011}$ & $13.62_{\pm 0.46}$ & $0.159_{\pm 0.014}$ & $1.99_{\pm 0.14}$ & $13.02_{\pm 0.43}$ \\
\textsc{Step-Back} & $0.180_{\pm 0.004}$ & $0.584_{\pm 0.009}$ & $12.79_{\pm 0.20}$ & $0.162_{\pm 0.012}$ & $1.77_{\pm 0.18}$ & $12.86_{\pm 0.20}$ \\
\textsc{DMAD} & $0.178_{\pm 0.009}$ & $0.585_{\pm 0.016}$ & $14.98_{\pm 0.77}$ & $0.168_{\pm 0.008}$ & $2.28_{\pm 0.19}$ & $13.90_{\pm 0.49}$ \\
\textsc{Gibberish} & $0.235_{\pm 0.009}$ & $0.675_{\pm 0.016}$ & $20.50_{\pm 2.58}$ & $0.178_{\pm 0.013}$ & $3.27_{\pm 0.63}$ & $15.89_{\pm 0.68}$ \\
\addlinespace
\multicolumn{7}{@{}l}{\textit{Baseline personas}} \\
\textsc{MPAQ} & $0.216_{\pm 0.007}$ & $0.651_{\pm 0.011}$ & $17.70_{\pm 0.59}$ & $0.198_{\pm 0.008}$ & $2.92_{\pm 0.38}$ & $15.46_{\pm 0.34}$ \\
\addlinespace
\multicolumn{7}{@{}l}{\textit{Individualized personas}} \\
\textsc{Random} & $0.206_{\pm 0.004}$ & $0.629_{\pm 0.007}$ & $12.52_{\pm 0.66}$ & $0.213_{\pm 0.008}$ & $1.48_{\pm 0.15}$ & $11.89_{\pm 0.48}$ \\
\textsc{Typical} & $0.206_{\pm 0.006}$ & $0.629_{\pm 0.010}$ & $12.70_{\pm 0.31}$ & $0.216_{\pm 0.013}$ & $1.58_{\pm 0.21}$ & $11.96_{\pm 0.30}$ \\
\addlinespace
\multicolumn{7}{@{}l}{\textit{Selected personas from base pool}} \\
\quad Coverage, cosine & $\mathbf{0.214}_{\pm 0.008}$ & $\mathbf{0.639}_{\pm 0.013}$ & $12.66_{\pm 0.63}$ & $0.214_{\pm 0.008}$ & $1.57_{\pm 0.16}$ & $12.03_{\pm 0.37}$ \\
\quad Coverage, Mahalanobis & $0.205_{\pm 0.008}$ & $0.624_{\pm 0.016}$ & $11.90_{\pm 0.49}$ & $0.214_{\pm 0.009}$ & $1.52_{\pm 0.15}$ & $11.46_{\pm 0.28}$ \\
\quad Dispersion, cosine & $0.209_{\pm 0.002}$ & $0.638_{\pm 0.003}$ & $\mathbf{13.11}_{\pm 0.45}$ & $0.220_{\pm 0.005}$ & $1.87_{\pm 0.33}$ & $\mathbf{12.34}_{\pm 0.27}$ \\
\quad Dispersion, Mahalanobis & $0.208_{\pm 0.003}$ & $0.634_{\pm 0.004}$ & $12.87_{\pm 0.39}$ & $\mathbf{0.224}_{\pm 0.003}$ & $\mathbf{1.95}_{\pm 0.32}$ & $12.00_{\pm 0.28}$ \\
\addlinespace
\multicolumn{7}{@{}l}{\textit{Generated personas} (expanded pool size)} \\
\quad MCMC (8,553) & $0.211_{\pm 0.009}$ & $0.639_{\pm 0.014}$ & $13.92_{\pm 1.35}$ & $\mathbf{0.223}_{\pm 0.007}$ & $2.21_{\pm 0.48}$ & $12.66_{\pm 0.55}$ \\
\quad Evolution (1,869) & $0.214_{\pm 0.008}$ & $0.646_{\pm 0.012}$ & $15.41_{\pm 0.61}$ & $0.212_{\pm 0.003}$ & $2.56_{\pm 0.41}$ & $\mathbf{13.61}_{\pm 0.17}$ \\
\quad AUT-Evolution (3,397) & $\mathbf{0.222}_{\pm 0.006}$ & $\mathbf{0.658}_{\pm 0.011}$ & $\mathbf{16.03}_{\pm 0.56}$ & $0.227_{\pm 0.007}$ & $\mathbf{3.08}_{\pm 0.34}$ & $13.58_{\pm 0.11}$ \\
\addlinespace
\multicolumn{7}{@{}l}{\textit{Selected personas} $+$ \textsc{Creativity-enhanced} prompt} \\
\quad Coverage, cosine & $0.186_{\pm 0.005}$ & $0.602_{\pm 0.009}$ & $19.84_{\pm 0.60}$ & $0.145_{\pm 0.006}$ & $5.08_{\pm 0.32}$ & $17.25_{\pm 0.12}$ \\
\quad Coverage, Mahalanobis & $0.186_{\pm 0.005}$ & $0.604_{\pm 0.008}$ & $19.94_{\pm 0.39}$ & $0.144_{\pm 0.005}$ & $5.22_{\pm 0.25}$ & $17.20_{\pm 0.15}$ \\
\quad Dispersion, cosine & $\mathbf{0.190}_{\pm 0.005}$ & $\mathbf{0.608}_{\pm 0.010}$ & $19.85_{\pm 0.46}$ & $\mathbf{0.153}_{\pm 0.010}$ & $6.17_{\pm 0.52}$ & $\mathbf{17.24}_{\pm 0.26}$ \\
\quad Dispersion, Mahalanobis & $0.184_{\pm 0.002}$ & $0.600_{\pm 0.004}$ & $\mathbf{20.18}_{\pm 0.60}$ & $0.148_{\pm 0.005}$ & $\mathbf{6.59}_{\pm 0.38}$ & $17.22_{\pm 0.39}$ \\
\addlinespace
\multicolumn{7}{@{}l}{\textit{Generated personas} $+$ \textsc{Creativity-enhanced} prompt} \\
\quad MCMC & $0.186_{\pm 0.003}$ & $0.604_{\pm 0.006}$ & $20.62_{\pm 0.24}$ & $0.147_{\pm 0.004}$ & $\mathbf{6.58}_{\pm 0.28}$ & $17.64_{\pm 0.21}$ \\
\quad Evolution & $0.190_{\pm 0.005}$ & $0.610_{\pm 0.009}$ & $20.66_{\pm 0.73}$ & $0.148_{\pm 0.011}$ & $5.98_{\pm 0.49}$ & $17.65_{\pm 0.14}$ \\
\quad AUT-Evolution & $\mathbf{0.196}_{\pm 0.006}$ & $\mathbf{0.620}_{\pm 0.011}$ & $\mathbf{21.48}_{\pm 0.75}$ & $\mathbf{0.148}_{\pm 0.003}$ & $6.92_{\pm 0.40}$ & $\mathbf{18.02}_{\pm 0.39}$ \\
\addlinespace
\multicolumn{7}{@{}l}{\textit{Human}} \\
\textsc{Stevenson} & $0.119_{\pm 0.000}$ & $0.480_{\pm 0.000}$ & $15.51_{\pm 0.00}$ & $0.144_{\pm 0.000}$ & $0.533_{\pm 0.000}$ & $14.15_{\pm 0.00}$ \\
\bottomrule
\end{tabular}%
}
\end{table}

\begin{table}[t]
\centering\small
\setlength{\tabcolsep}{5pt}
\caption{Response flexibility, reported over all seven AUT objects and, for comparability with \textsc{Stevenson}, over the three objects with human reference data ($175$ and $75$ generated uses per variant, respectively). Both measures are size-controlled: \emph{Flexibility} is the number of distinct use categories and \emph{Vendi} is the Vendi score \eqref{eq:vendi}, each divided by the number of valid uses, so that variants admitting different numbers of uses remain comparable. Cells give the mean over objects. Subscripts denote the 95\% confidence interval across 5 independent random seeds; higher is more flexible. \textsc{Stevenson} exists only for the human-reference objects, so its seven-object cells are omitted (---). The final two groups re-run our selection and generation methods on the \textsc{Creativity-enhanced} prompt, showing that persona diversification composes with a stronger prompting strategy. Bold marks the best value per column within the standard, selected and generated groups separately, with tied values bolded jointly; baseline, individualized and human rows are not bolded. Generated variants report their expanded pool size in parentheses.}
\label{tab:aut_response_flexibility}
\begin{tabular}{lcccc}
\toprule
 & \multicolumn{2}{c}{All seven objects} & \multicolumn{2}{c}{Human-reference objects} \\
\cmidrule(lr){2-3}\cmidrule(lr){4-5}
Variant & Flexibility & Vendi & Flexibility & Vendi \\
\midrule
\multicolumn{5}{@{}l}{\textit{Standard persona}} \\
\textsc{Common-Use} & $0.160_{\pm 0.011}$ & $0.086_{\pm 0.003}$ & $0.214_{\pm 0.025}$ & $0.094_{\pm 0.004}$ \\
\textsc{Alternative-Use} & $0.288_{\pm 0.022}$ & $0.096_{\pm 0.002}$ & $0.284_{\pm 0.012}$ & $0.099_{\pm 0.003}$ \\
\textsc{Expert} & $0.300_{\pm 0.026}$ & $0.094_{\pm 0.001}$ & $0.320_{\pm 0.017}$ & $0.096_{\pm 0.001}$ \\
\textsc{Creativity-enhanced} & $\mathbf{0.503}_{\pm 0.056}$ & $0.104_{\pm 0.002}$ & $\mathbf{0.467}_{\pm 0.071}$ & $0.101_{\pm 0.002}$ \\
\textsc{ZS-CoT} & $0.332_{\pm 0.025}$ & $0.094_{\pm 0.001}$ & $0.323_{\pm 0.044}$ & $0.091_{\pm 0.002}$ \\
\textsc{Step-Back} & $0.303_{\pm 0.022}$ & $0.096_{\pm 0.004}$ & $0.280_{\pm 0.032}$ & $0.097_{\pm 0.003}$ \\
\textsc{DMAD} & $0.372_{\pm 0.025}$ & $0.102_{\pm 0.003}$ & $0.350_{\pm 0.034}$ & $0.098_{\pm 0.004}$ \\
\textsc{Gibberish} & $0.384_{\pm 0.061}$ & $\mathbf{0.151}_{\pm 0.013}$ & $0.359_{\pm 0.031}$ & $\mathbf{0.158}_{\pm 0.027}$ \\
\addlinespace
\multicolumn{5}{@{}l}{\textit{Baseline personas}} \\
\textsc{MPAQ} & $0.561_{\pm 0.030}$ & $0.130_{\pm 0.006}$ & $0.491_{\pm 0.057}$ & $0.129_{\pm 0.006}$ \\
\addlinespace
\multicolumn{5}{@{}l}{\textit{Individualized personas}} \\
\textsc{Random} & $0.331_{\pm 0.024}$ & $0.107_{\pm 0.003}$ & $0.333_{\pm 0.031}$ & $0.109_{\pm 0.003}$ \\
\textsc{Typical} & $0.330_{\pm 0.016}$ & $0.105_{\pm 0.002}$ & $0.323_{\pm 0.022}$ & $0.108_{\pm 0.003}$ \\
\addlinespace
\multicolumn{5}{@{}l}{\textit{Selected personas from base pool}} \\
\quad Coverage, cosine & $0.338_{\pm 0.025}$ & $0.110_{\pm 0.003}$ & $0.323_{\pm 0.025}$ & $0.111_{\pm 0.004}$ \\
\quad Coverage, Mahalanobis & $0.337_{\pm 0.011}$ & $0.101_{\pm 0.002}$ & $0.329_{\pm 0.041}$ & $0.105_{\pm 0.005}$ \\
\quad Dispersion, cosine & $\mathbf{0.379}_{\pm 0.026}$ & $\mathbf{0.112}_{\pm 0.001}$ & $\mathbf{0.356}_{\pm 0.048}$ & $\mathbf{0.113}_{\pm 0.001}$ \\
\quad Dispersion, Mahalanobis & $0.357_{\pm 0.048}$ & $0.108_{\pm 0.002}$ & $0.353_{\pm 0.051}$ & $0.111_{\pm 0.001}$ \\
\addlinespace
\multicolumn{5}{@{}l}{\textit{Generated personas} (expanded pool size)} \\
\quad MCMC (8,553) & $0.378_{\pm 0.052}$ & $0.112_{\pm 0.001}$ & $0.348_{\pm 0.029}$ & $0.113_{\pm 0.005}$ \\
\quad Evolution (1,869) & $0.431_{\pm 0.045}$ & $0.117_{\pm 0.002}$ & $0.411_{\pm 0.102}$ & $0.117_{\pm 0.003}$ \\
\quad AUT-Evolution (3,397) & $\mathbf{0.463}_{\pm 0.052}$ & $\mathbf{0.118}_{\pm 0.002}$ & $\mathbf{0.490}_{\pm 0.045}$ & $\mathbf{0.121}_{\pm 0.004}$ \\
\addlinespace
\multicolumn{5}{@{}l}{\textit{Selected personas} $+$ \textsc{Creativity-enhanced} prompt} \\
\quad Coverage, cosine & $0.508_{\pm 0.082}$ & $0.109_{\pm 0.003}$ & $0.463_{\pm 0.101}$ & $0.107_{\pm 0.002}$ \\
\quad Coverage, Mahalanobis & $0.518_{\pm 0.066}$ & $0.108_{\pm 0.001}$ & $0.496_{\pm 0.090}$ & $0.107_{\pm 0.003}$ \\
\quad Dispersion, cosine & $0.543_{\pm 0.052}$ & $\mathbf{0.111}_{\pm 0.003}$ & $0.514_{\pm 0.094}$ & $\mathbf{0.109}_{\pm 0.003}$ \\
\quad Dispersion, Mahalanobis & $\mathbf{0.560}_{\pm 0.051}$ & $0.109_{\pm 0.003}$ & $\mathbf{0.526}_{\pm 0.117}$ & $0.106_{\pm 0.001}$ \\
\addlinespace
\multicolumn{5}{@{}l}{\textit{Generated personas} $+$ \textsc{Creativity-enhanced} prompt} \\
\quad MCMC & $0.560_{\pm 0.045}$ & $0.111_{\pm 0.002}$ & $0.526_{\pm 0.108}$ & $0.108_{\pm 0.002}$ \\
\quad Evolution & $0.580_{\pm 0.046}$ & $0.112_{\pm 0.003}$ & $0.562_{\pm 0.062}$ & $0.110_{\pm 0.004}$ \\
\quad AUT-Evolution & $\mathbf{0.612}_{\pm 0.026}$ & $\mathbf{0.114}_{\pm 0.003}$ & $\mathbf{0.606}_{\pm 0.123}$ & $\mathbf{0.113}_{\pm 0.003}$ \\
\addlinespace
\multicolumn{5}{@{}l}{\textit{Human}} \\
\textsc{Stevenson} & --- & --- & $0.739_{\pm 0.043}$ & $0.104_{\pm 0.001}$ \\
\bottomrule
\end{tabular}
\end{table}

\begin{table}[t]
\centering\small
\setlength{\tabcolsep}{5pt}
\caption{Response surprise, reported over all seven AUT objects and, for comparability with \textsc{Stevenson}, over the three objects with human reference data ($175$ and $75$ generated uses per variant, respectively). Both measures are size-controlled by the number of valid uses: \emph{Surprise} is the proportion of uses judged surprising relative to the common-use reference set, and \emph{Wtd.\ Vendi} is the novelty-weighted Vendi score \eqref{eq:vendi} under the same normalization. \textsc{Common-Use} defines the reference set against which surprise is measured, so its surprise is zero by construction and is omitted (---); \textsc{Stevenson} exists only for the human-reference objects. Cells give the mean over objects. Subscripts denote the 95\% confidence interval across 5 independent random seeds; higher is more surprising. The final two groups re-run our selection and generation methods on the \textsc{Creativity-enhanced} prompt, showing that persona diversification composes with a stronger prompting strategy. Bolding conventions are as in Table~\ref{tab:aut_response_flexibility}.}
\label{tab:aut_response_surprise}
\begin{tabular}{lcccc}
\toprule
 & \multicolumn{2}{c}{All seven objects} & \multicolumn{2}{c}{Human-reference objects} \\
\cmidrule(lr){2-3}\cmidrule(lr){4-5}
Variant & Surprise & Wtd.\ Vendi & Surprise & Wtd.\ Vendi \\
\midrule
\multicolumn{5}{@{}l}{\textit{Standard persona}} \\
\textsc{Common-Use} & --- & --- & --- & --- \\
\textsc{Alternative-Use} & $0.506_{\pm 0.062}$ & $0.094_{\pm 0.002}$ & $0.397_{\pm 0.087}$ & $0.093_{\pm 0.004}$ \\
\textsc{Expert} & $0.415_{\pm 0.020}$ & $0.093_{\pm 0.002}$ & $0.317_{\pm 0.014}$ & $0.098_{\pm 0.002}$ \\
\textsc{Creativity-enhanced} & $\mathbf{0.837}_{\pm 0.067}$ & $0.105_{\pm 0.002}$ & $\mathbf{0.717}_{\pm 0.131}$ & $0.102_{\pm 0.002}$ \\
\textsc{ZS-CoT} & $0.556_{\pm 0.064}$ & $0.094_{\pm 0.002}$ & $0.390_{\pm 0.063}$ & $0.089_{\pm 0.003}$ \\
\textsc{Step-Back} & $0.558_{\pm 0.068}$ & $0.095_{\pm 0.005}$ & $0.413_{\pm 0.087}$ & $0.095_{\pm 0.005}$ \\
\textsc{DMAD} & $0.567_{\pm 0.057}$ & $0.103_{\pm 0.004}$ & $0.382_{\pm 0.131}$ & $0.100_{\pm 0.005}$ \\
\textsc{Gibberish} & $0.525_{\pm 0.058}$ & $\mathbf{0.150}_{\pm 0.012}$ & $0.398_{\pm 0.026}$ & $\mathbf{0.156}_{\pm 0.025}$ \\
\addlinespace
\multicolumn{5}{@{}l}{\textit{Baseline personas}} \\
\textsc{MPAQ} & $0.686_{\pm 0.055}$ & $0.132_{\pm 0.006}$ & $0.535_{\pm 0.090}$ & $0.131_{\pm 0.006}$ \\
\addlinespace
\multicolumn{5}{@{}l}{\textit{Individualized personas}} \\
\textsc{Random} & $0.588_{\pm 0.051}$ & $0.107_{\pm 0.003}$ & $0.467_{\pm 0.061}$ & $0.109_{\pm 0.004}$ \\
\textsc{Typical} & $0.572_{\pm 0.051}$ & $0.105_{\pm 0.003}$ & $0.454_{\pm 0.055}$ & $0.107_{\pm 0.003}$ \\
\addlinespace
\multicolumn{5}{@{}l}{\textit{Selected personas from base pool}} \\
\quad Coverage, cosine & $0.574_{\pm 0.057}$ & $0.110_{\pm 0.003}$ & $0.445_{\pm 0.097}$ & $\mathbf{0.111}_{\pm 0.005}$ \\
\quad Coverage, Mahalanobis & $0.548_{\pm 0.048}$ & $0.101_{\pm 0.003}$ & $0.385_{\pm 0.029}$ & $0.104_{\pm 0.005}$ \\
\quad Dispersion, cosine & $\mathbf{0.588}_{\pm 0.081}$ & $\mathbf{0.112}_{\pm 0.001}$ & $\mathbf{0.449}_{\pm 0.086}$ & $\mathbf{0.111}_{\pm 0.001}$ \\
\quad Dispersion, Mahalanobis & $0.567_{\pm 0.075}$ & $0.108_{\pm 0.002}$ & $0.406_{\pm 0.069}$ & $0.110_{\pm 0.003}$ \\
\addlinespace
\multicolumn{5}{@{}l}{\textit{Generated personas} (expanded pool size)} \\
\quad MCMC (8,553) & $0.566_{\pm 0.078}$ & $0.113_{\pm 0.002}$ & $0.411_{\pm 0.078}$ & $0.113_{\pm 0.007}$ \\
\quad Evolution (1,869) & $\mathbf{0.637}_{\pm 0.064}$ & $\mathbf{0.118}_{\pm 0.002}$ & $0.514_{\pm 0.125}$ & $0.118_{\pm 0.004}$ \\
\quad AUT-Evolution (3,397) & $\mathbf{0.637}_{\pm 0.068}$ & $\mathbf{0.118}_{\pm 0.002}$ & $\mathbf{0.522}_{\pm 0.066}$ & $\mathbf{0.120}_{\pm 0.005}$ \\
\addlinespace
\multicolumn{5}{@{}l}{\textit{Selected personas} $+$ \textsc{Creativity-enhanced} prompt} \\
\quad Coverage, cosine & $\mathbf{0.843}_{\pm 0.061}$ & $0.110_{\pm 0.003}$ & $0.739_{\pm 0.126}$ & $0.108_{\pm 0.002}$ \\
\quad Coverage, Mahalanobis & $0.815_{\pm 0.074}$ & $0.109_{\pm 0.001}$ & $0.707_{\pm 0.130}$ & $0.108_{\pm 0.002}$ \\
\quad Dispersion, cosine & $0.833_{\pm 0.043}$ & $\mathbf{0.113}_{\pm 0.003}$ & $0.763_{\pm 0.071}$ & $\mathbf{0.110}_{\pm 0.003}$ \\
\quad Dispersion, Mahalanobis & $0.839_{\pm 0.032}$ & $0.111_{\pm 0.003}$ & $\mathbf{0.781}_{\pm 0.079}$ & $0.108_{\pm 0.001}$ \\
\addlinespace
\multicolumn{5}{@{}l}{\textit{Generated personas} $+$ \textsc{Creativity-enhanced} prompt} \\
\quad MCMC & $0.847_{\pm 0.021}$ & $0.112_{\pm 0.002}$ & $0.792_{\pm 0.067}$ & $0.109_{\pm 0.002}$ \\
\quad Evolution & $0.845_{\pm 0.043}$ & $0.113_{\pm 0.003}$ & $0.781_{\pm 0.079}$ & $0.111_{\pm 0.004}$ \\
\quad AUT-Evolution & $\mathbf{0.872}_{\pm 0.039}$ & $\mathbf{0.115}_{\pm 0.003}$ & $\mathbf{0.822}_{\pm 0.089}$ & $\mathbf{0.115}_{\pm 0.003}$ \\
\addlinespace
\multicolumn{5}{@{}l}{\textit{Human}} \\
\textsc{Stevenson} & --- & --- & $0.741_{\pm 0.065}$ & $0.106_{\pm 0.001}$ \\
\bottomrule
\end{tabular}
\vspace{-5mm}
\end{table}

\begin{table}[t]
\centering\small
\setlength{\tabcolsep}{6pt}
\caption{LLM-judged creativity on the AUT across all seven objects. A Qwen3.6-27B judge, independent of the Gemma-4-31B generator, rates every use for Originality, Surprise and Utility on a $1$--$5$ scale following \citet{stevenson2022putting}, and for holistic creativity following \citet{goes2023pushing}; we report the rank-normalized form of that holistic judgment, which correlates with its raw form at $r=0.996$. Cells give the mean over judged uses (at most $175$ per variant). The judged set is every use not rejected outright by the validity gate of \Cref{sec:aut}, admitted \emph{and} uncertain uses, so it is slightly larger than the admitted set of Table~\ref{tab:aut_response_validity}. Subscripts denote the 95\% confidence interval across 5 independent random seeds. The final two groups re-run our selection and generation methods on the \textsc{Creativity-enhanced} prompt, showing that persona diversification composes with a stronger prompting strategy. Originality, Surprise and Creativity are target metrics, bolded per column within the standard, selected and generated groups separately (and within each creativity-enhanced group), with tied values bolded jointly. \emph{Utility is a control}: conventional uses are the most useful by construction, so it necessarily peaks at \textsc{Common-Use} and is never bolded. Baseline, individualized and human rows are not bolded. Generated variants report their expanded pool size in parentheses.}
\label{tab:aut_llmjudge_full}
\begin{tabular}{lcccc}
\toprule
 & \multicolumn{3}{c}{\citet{stevenson2022putting}} & \citet{goes2023pushing} \\
\cmidrule(lr){2-4}\cmidrule(lr){5-5}
Variant & Originality & Surprise & Utility & Creativity (rank) \\
\midrule
\multicolumn{5}{@{}l}{\textit{Standard persona}} \\
\textsc{Common-Use} & $1.68_{\pm 0.03}$ & $1.34_{\pm 0.01}$ & $4.86_{\pm 0.04}$ & $2.02_{\pm 0.08}$ \\
\textsc{Alternative-Use} & $2.63_{\pm 0.05}$ & $2.23_{\pm 0.06}$ & $4.30_{\pm 0.05}$ & $2.72_{\pm 0.09}$ \\
\textsc{Expert} & $2.28_{\pm 0.05}$ & $2.01_{\pm 0.05}$ & $4.39_{\pm 0.05}$ & $2.71_{\pm 0.05}$ \\
\textsc{Creativity-enhanced} & $\mathbf{3.17}_{\pm 0.06}$ & $\mathbf{2.66}_{\pm 0.06}$ & $3.38_{\pm 0.07}$ & $\mathbf{3.19}_{\pm 0.09}$ \\
\textsc{ZS-CoT} & $2.60_{\pm 0.09}$ & $2.11_{\pm 0.07}$ & $4.27_{\pm 0.07}$ & $2.85_{\pm 0.06}$ \\
\textsc{Step-Back} & $2.74_{\pm 0.02}$ & $2.30_{\pm 0.07}$ & $4.30_{\pm 0.04}$ & $2.95_{\pm 0.06}$ \\
\textsc{DMAD} & $2.65_{\pm 0.11}$ & $2.17_{\pm 0.09}$ & $4.17_{\pm 0.11}$ & $2.89_{\pm 0.07}$ \\
\textsc{Gibberish} & $2.85_{\pm 0.06}$ & $2.33_{\pm 0.07}$ & $4.02_{\pm 0.10}$ & $2.57_{\pm 0.09}$ \\
\addlinespace
\multicolumn{5}{@{}l}{\textit{Baseline personas}} \\
\textsc{MPAQ} & $3.29_{\pm 0.04}$ & $2.65_{\pm 0.06}$ & $3.59_{\pm 0.04}$ & $3.05_{\pm 0.07}$ \\
\addlinespace
\multicolumn{5}{@{}l}{\textit{Individualized personas}} \\
\textsc{Random} & $2.78_{\pm 0.03}$ & $2.36_{\pm 0.03}$ & $4.28_{\pm 0.02}$ & $2.96_{\pm 0.04}$ \\
\textsc{Typical} & $2.75_{\pm 0.04}$ & $2.36_{\pm 0.04}$ & $4.24_{\pm 0.05}$ & $2.99_{\pm 0.07}$ \\
\addlinespace
\multicolumn{5}{@{}l}{\textit{Selected personas from base pool}} \\
\quad Coverage, cosine & $2.74_{\pm 0.04}$ & $\mathbf{2.40}_{\pm 0.02}$ & $4.30_{\pm 0.02}$ & $2.91_{\pm 0.06}$ \\
\quad Coverage, Mahalanobis & $2.66_{\pm 0.02}$ & $2.29_{\pm 0.05}$ & $4.27_{\pm 0.04}$ & $2.86_{\pm 0.08}$ \\
\quad Dispersion, cosine & $\mathbf{2.78}_{\pm 0.06}$ & $\mathbf{2.40}_{\pm 0.04}$ & $4.27_{\pm 0.09}$ & $2.91_{\pm 0.06}$ \\
\quad Dispersion, Mahalanobis & $2.74_{\pm 0.08}$ & $2.36_{\pm 0.06}$ & $4.23_{\pm 0.09}$ & $\mathbf{2.92}_{\pm 0.04}$ \\
\addlinespace
\multicolumn{5}{@{}l}{\textit{Generated personas} (expanded pool size)} \\
\quad MCMC (8,553) & $2.74_{\pm 0.05}$ & $2.40_{\pm 0.07}$ & $4.28_{\pm 0.06}$ & $2.91_{\pm 0.04}$ \\
\quad Evolution (1,869) & $2.97_{\pm 0.07}$ & $2.54_{\pm 0.06}$ & $4.08_{\pm 0.06}$ & $\mathbf{3.10}_{\pm 0.03}$ \\
\quad AUT-Evolution (3,397) & $\mathbf{2.99}_{\pm 0.09}$ & $\mathbf{2.63}_{\pm 0.09}$ & $4.04_{\pm 0.10}$ & $2.99_{\pm 0.06}$ \\
\addlinespace
\multicolumn{5}{@{}l}{\textit{Selected personas} $+$ \textsc{Creativity-enhanced} prompt} \\
\quad Coverage, cosine & $\mathbf{3.22}_{\pm 0.07}$ & $\mathbf{2.75}_{\pm 0.04}$ & $3.43_{\pm 0.10}$ & $3.25_{\pm 0.04}$ \\
\quad Coverage, Mahalanobis & $3.15_{\pm 0.05}$ & $2.71_{\pm 0.05}$ & $3.58_{\pm 0.10}$ & $3.24_{\pm 0.07}$ \\
\quad Dispersion, cosine & $3.20_{\pm 0.09}$ & $2.72_{\pm 0.12}$ & $3.54_{\pm 0.11}$ & $3.25_{\pm 0.07}$ \\
\quad Dispersion, Mahalanobis & $3.19_{\pm 0.12}$ & $2.74_{\pm 0.10}$ & $3.44_{\pm 0.09}$ & $\mathbf{3.27}_{\pm 0.06}$ \\
\addlinespace
\multicolumn{5}{@{}l}{\textit{Generated personas} $+$ \textsc{Creativity-enhanced} prompt} \\
\quad MCMC & $3.24_{\pm 0.15}$ & $2.76_{\pm 0.10}$ & $3.39_{\pm 0.10}$ & $3.31_{\pm 0.05}$ \\
\quad Evolution & $3.31_{\pm 0.12}$ & $2.87_{\pm 0.11}$ & $3.38_{\pm 0.16}$ & $\mathbf{3.39}_{\pm 0.07}$ \\
\quad AUT-Evolution & $\mathbf{3.38}_{\pm 0.13}$ & $\mathbf{2.94}_{\pm 0.13}$ & $3.33_{\pm 0.14}$ & $3.32_{\pm 0.06}$ \\
\addlinespace
\multicolumn{5}{@{}l}{\textit{Human}} \\
\textsc{Stevenson} & $2.78_{\pm 0.00}$ & $1.77_{\pm 0.00}$ & $3.49_{\pm 0.03}$ & $2.07_{\pm 0.03}$ \\
\bottomrule
\end{tabular}%
\end{table}

\begin{table}[t]
\centering\small
\setlength{\tabcolsep}{6pt}
\caption{LLM-judged creativity restricted to the three objects with human reference data (book, fork, tin can), so that all variants are directly comparable to \textsc{Stevenson} (at most $75$ judged uses per variant). Scales, judge, bolding and control conventions are as in Table~\ref{tab:aut_llmjudge_full}.}
\label{tab:aut_llmjudge_human}
\begin{tabular}{lcccc}
\toprule
 & \multicolumn{3}{c}{\citet{stevenson2022putting}} & \citet{goes2023pushing} \\
\cmidrule(lr){2-4}\cmidrule(lr){5-5}
Variant & Originality & Surprise & Utility & Creativity (rank) \\
\midrule
\multicolumn{5}{@{}l}{\textit{Standard persona}} \\
\textsc{Common-Use} & $1.84_{\pm 0.04}$ & $1.45_{\pm 0.03}$ & $4.84_{\pm 0.03}$ & $2.29_{\pm 0.12}$ \\
\textsc{Alternative-Use} & $2.37_{\pm 0.09}$ & $2.07_{\pm 0.05}$ & $4.39_{\pm 0.08}$ & $2.78_{\pm 0.16}$ \\
\textsc{Expert} & $2.18_{\pm 0.04}$ & $2.00_{\pm 0.06}$ & $4.40_{\pm 0.08}$ & $2.83_{\pm 0.07}$ \\
\textsc{Creativity-enhanced} & $\mathbf{3.05}_{\pm 0.18}$ & $\mathbf{2.66}_{\pm 0.06}$ & $3.52_{\pm 0.11}$ & $\mathbf{3.19}_{\pm 0.09}$ \\
\textsc{ZS-CoT} & $2.37_{\pm 0.11}$ & $1.96_{\pm 0.12}$ & $4.49_{\pm 0.10}$ & $2.88_{\pm 0.12}$ \\
\textsc{Step-Back} & $2.52_{\pm 0.07}$ & $2.16_{\pm 0.07}$ & $4.31_{\pm 0.09}$ & $2.95_{\pm 0.07}$ \\
\textsc{DMAD} & $2.39_{\pm 0.11}$ & $2.06_{\pm 0.08}$ & $4.38_{\pm 0.08}$ & $2.90_{\pm 0.12}$ \\
\textsc{Gibberish} & $2.59_{\pm 0.07}$ & $2.13_{\pm 0.09}$ & $4.19_{\pm 0.15}$ & $2.37_{\pm 0.22}$ \\
\addlinespace
\multicolumn{5}{@{}l}{\textit{Baseline personas}} \\
\textsc{MPAQ} & $3.05_{\pm 0.13}$ & $2.44_{\pm 0.08}$ & $3.78_{\pm 0.07}$ & $3.00_{\pm 0.10}$ \\
\addlinespace
\multicolumn{5}{@{}l}{\textit{Individualized personas}} \\
\textsc{Random} & $2.52_{\pm 0.05}$ & $2.18_{\pm 0.04}$ & $4.35_{\pm 0.03}$ & $3.03_{\pm 0.02}$ \\
\textsc{Typical} & $2.56_{\pm 0.11}$ & $2.28_{\pm 0.05}$ & $4.24_{\pm 0.12}$ & $3.07_{\pm 0.07}$ \\
\addlinespace
\multicolumn{5}{@{}l}{\textit{Selected personas from base pool}} \\
\quad Coverage, cosine & $2.52_{\pm 0.07}$ & $\mathbf{2.27}_{\pm 0.06}$ & $4.35_{\pm 0.04}$ & $\mathbf{2.98}_{\pm 0.07}$ \\
\quad Coverage, Mahalanobis & $2.44_{\pm 0.05}$ & $2.17_{\pm 0.12}$ & $4.34_{\pm 0.08}$ & $2.93_{\pm 0.08}$ \\
\quad Dispersion, cosine & $\mathbf{2.56}_{\pm 0.08}$ & $2.21_{\pm 0.05}$ & $4.35_{\pm 0.08}$ & $2.95_{\pm 0.06}$ \\
\quad Dispersion, Mahalanobis & $2.53_{\pm 0.07}$ & $2.16_{\pm 0.07}$ & $4.33_{\pm 0.05}$ & $2.94_{\pm 0.04}$ \\
\addlinespace
\multicolumn{5}{@{}l}{\textit{Generated personas} (expanded pool size)} \\
\quad MCMC (8,553) & $2.53_{\pm 0.08}$ & $2.20_{\pm 0.03}$ & $4.36_{\pm 0.06}$ & $2.93_{\pm 0.07}$ \\
\quad Evolution (1,869) & $2.76_{\pm 0.06}$ & $2.38_{\pm 0.08}$ & $4.22_{\pm 0.10}$ & $\mathbf{3.11}_{\pm 0.02}$ \\
\quad AUT-Evolution (3,397) & $\mathbf{2.88}_{\pm 0.09}$ & $\mathbf{2.49}_{\pm 0.09}$ & $4.10_{\pm 0.11}$ & $3.02_{\pm 0.05}$ \\
\addlinespace
\multicolumn{5}{@{}l}{\textit{Selected personas} $+$ \textsc{Creativity-enhanced} prompt} \\
\quad Coverage, cosine & $\mathbf{3.11}_{\pm 0.08}$ & $2.69_{\pm 0.08}$ & $3.59_{\pm 0.13}$ & $3.27_{\pm 0.04}$ \\
\quad Coverage, Mahalanobis & $3.10_{\pm 0.11}$ & $\mathbf{2.73}_{\pm 0.13}$ & $3.68_{\pm 0.12}$ & $3.26_{\pm 0.13}$ \\
\quad Dispersion, cosine & $3.10_{\pm 0.16}$ & $2.71_{\pm 0.12}$ & $3.63_{\pm 0.19}$ & $3.30_{\pm 0.06}$ \\
\quad Dispersion, Mahalanobis & $3.09_{\pm 0.15}$ & $2.72_{\pm 0.16}$ & $3.56_{\pm 0.12}$ & $\mathbf{3.31}_{\pm 0.06}$ \\
\addlinespace
\multicolumn{5}{@{}l}{\textit{Generated personas} $+$ \textsc{Creativity-enhanced} prompt} \\
\quad MCMC & $3.15_{\pm 0.19}$ & $2.75_{\pm 0.10}$ & $3.50_{\pm 0.18}$ & $3.33_{\pm 0.13}$ \\
\quad Evolution & $3.21_{\pm 0.24}$ & $2.81_{\pm 0.18}$ & $3.50_{\pm 0.28}$ & $\mathbf{3.41}_{\pm 0.16}$ \\
\quad AUT-Evolution & $\mathbf{3.27}_{\pm 0.22}$ & $\mathbf{2.87}_{\pm 0.17}$ & $3.44_{\pm 0.20}$ & $3.31_{\pm 0.12}$ \\
\addlinespace
\multicolumn{5}{@{}l}{\textit{Human}} \\
\textsc{Stevenson} & $2.78_{\pm 0.00}$ & $1.77_{\pm 0.00}$ & $3.49_{\pm 0.03}$ & $2.07_{\pm 0.03}$ \\
\bottomrule
\end{tabular}%
\end{table}


\clearpage
\subsection{Infinity-Chat Benchmark}


\begin{table}[htbp]
\centering\small
\caption{Infinity-Chat 100 response validity and quality \citep{jiang2026artificial}. Each
condition contributes $250$ responses, $50$ per open-ended item for prompt conditions and ten
per persona per item for persona conditions over $k=5$ personas and five items.
\emph{Valid} is the proportion of responses admitted by the validity gate; \emph{Quality} is the
mean judge rating on a $1$--$5$ scale. Subscripts denote the 95\% confidence interval across 5 independent random seeds. The final two groups re-run our selection and generation methods on the \textsc{DMAD} prompt. Both are \emph{controls} rather than targets, validity near $1$ is expected, and
quality is the effectiveness dimension against which diversity gains are traded, so no value
is bolded. Generated variants report their expanded pool size in parentheses.}
\label{tab:ic_validity_quality}
\begin{tabular}{@{}lcc@{}}
\toprule
Variant & Valid & Quality \\
\midrule
\multicolumn{3}{@{}l}{\textit{Standard persona}} \\
\textsc{Open-Ended Query} & $0.950_{\pm 0.015}$ & $4.84_{\pm 0.02}$ \\
\textsc{ZS-CoT} & $1.000_{\pm 0.000}$ & $4.99_{\pm 0.01}$ \\
\textsc{Step-Back} & $1.000_{\pm 0.000}$ & $4.98_{\pm 0.02}$ \\
\textsc{DMAD} & $1.000_{\pm 0.000}$ & $4.96_{\pm 0.04}$ \\
\textsc{Gibberish} & $0.123_{\pm 0.013}$ & $1.01_{\pm 0.01}$ \\
\addlinespace
\multicolumn{3}{@{}l}{\textit{Baseline personas}} \\
\textsc{MPAQ} & $0.994_{\pm 0.008}$ & $3.69_{\pm 0.05}$ \\
\addlinespace
\multicolumn{3}{@{}l}{\textit{Individualized personas}} \\
\textsc{Random} & $1.000_{\pm 0.000}$ & $4.43_{\pm 0.03}$ \\
\textsc{Typical} & $1.000_{\pm 0.000}$ & $4.44_{\pm 0.03}$ \\
\addlinespace
\multicolumn{3}{@{}l}{\textit{Selected personas from base pool}} \\
\quad Coverage, cosine/L2 & $0.999_{\pm 0.002}$ & $4.32_{\pm 0.03}$ \\
\quad Coverage, Mahalanobis & $0.996_{\pm 0.000}$ & $4.38_{\pm 0.04}$ \\
\quad Dispersion, cosine/L2 & $0.970_{\pm 0.011}$ & $4.13_{\pm 0.08}$ \\
\quad Dispersion, Mahalanobis & $0.986_{\pm 0.007}$ & $4.13_{\pm 0.06}$ \\
\addlinespace
\multicolumn{3}{@{}l}{\textit{Generated personas} (expanded pool size)} \\
\quad MCMC (8,553) & $0.991_{\pm 0.007}$ & $4.23_{\pm 0.07}$ \\
\quad Evolution (1,869) & $0.992_{\pm 0.005}$ & $3.48_{\pm 0.05}$ \\
\quad IC-Evolution (2,819) & $0.974_{\pm 0.010}$ & $3.96_{\pm 0.06}$ \\
\addlinespace
\multicolumn{3}{@{}l}{\textit{Selected personas} $+$ \textsc{DMAD} prompt} \\
\quad Coverage, cosine/L2 & $1.000_{\pm 0.000}$ & $4.95_{\pm 0.01}$ \\
\quad Coverage, Mahalanobis & $1.000_{\pm 0.000}$ & $4.90_{\pm 0.03}$ \\
\quad Dispersion, cosine/L2 & $1.000_{\pm 0.000}$ & $4.95_{\pm 0.01}$ \\
\quad Dispersion, Mahalanobis & $1.000_{\pm 0.000}$ & $4.94_{\pm 0.02}$ \\
\addlinespace
\multicolumn{3}{@{}l}{\textit{Generated personas} $+$ \textsc{DMAD} prompt} \\
\quad MCMC & $1.000_{\pm 0.000}$ & $4.94_{\pm 0.01}$ \\
\quad Evolution & $1.000_{\pm 0.000}$ & $4.86_{\pm 0.05}$ \\
\quad IC-Evolution & $1.000_{\pm 0.000}$ & $4.92_{\pm 0.01}$ \\\bottomrule
\end{tabular}
\end{table}

 

\begin{table}[t]
\centering\small
\caption{Infinity-Chat 100 response homogeneity \citep{jiang2026artificial}.
\emph{Similarity} is the mean pairwise embedding similarity among a condition's responses to
the same item, averaged over items; \emph{Separation} is the gain in mean between-persona
distance over within-persona distance, and is undefined for conditions without personas (---);
\emph{Trigram} is mean trigram Jaccard overlap. Arrows give the favourable direction;
\textbf{bold} marks the best per column within the standard, selected and generated groups, and within each \textsc{DMAD} group.
$^{\dagger}$\textsc{Gibberish} is excluded from bolding because $90\%$ of its responses fail the
validity gate, leaving statistics computed on a tenth of the sample. Subscripts denote the 95\% confidence interval across 5 independent random seeds. The final two groups re-run our selection and generation methods on the \textsc{DMAD} prompt. Generated variants report
their expanded pool size in parentheses.}
\label{tab:ic_homogeneity}
\begin{tabular}{@{}lccc@{}}
\toprule
Variant & Similarity $\downarrow$ & Separation $\uparrow$ & Trigram $\downarrow$ \\
\midrule
\multicolumn{4}{@{}l}{\textit{Standard persona}} \\
\textsc{Open-Ended Query} & $0.946_{\pm 0.002}$ & --- & $0.247_{\pm 0.004}$ \\
\textsc{ZS-CoT} & $0.949_{\pm 0.003}$ & --- & $0.096_{\pm 0.012}$ \\
\textsc{Step-Back} & $0.962_{\pm 0.003}$ & --- & $0.276_{\pm 0.007}$ \\
\textsc{DMAD} & $\mathbf{0.929}_{\pm 0.002}$ & --- & $\mathbf{0.055}_{\pm 0.004}$ \\
\textsc{Gibberish}$^{\dagger}$ & $0.894_{\pm 0.012}$ & $0.045_{\pm 0.008}$ & $0.055_{\pm 0.022}$ \\
\addlinespace
\multicolumn{4}{@{}l}{\textit{Baseline personas}} \\
\textsc{MPAQ} & $0.831_{\pm 0.002}$ & $0.123_{\pm 0.001}$ & $0.029_{\pm 0.001}$ \\
\addlinespace
\multicolumn{4}{@{}l}{\textit{Individualized personas}} \\
\textsc{Random} & $0.864_{\pm 0.005}$ & $0.068_{\pm 0.003}$ & $0.039_{\pm 0.004}$ \\
\textsc{Typical} & $0.874_{\pm 0.004}$ & $0.083_{\pm 0.002}$ & $0.055_{\pm 0.003}$ \\
\addlinespace
\multicolumn{4}{@{}l}{\textit{Selected personas from base pool}} \\
\quad Coverage, cosine/L2 & $0.885_{\pm 0.003}$ & $0.055_{\pm 0.003}$ & $0.040_{\pm 0.002}$ \\
\quad Coverage, Mahalanobis & $0.866_{\pm 0.004}$ & $0.080_{\pm 0.003}$ & $0.038_{\pm 0.004}$ \\
\quad Dispersion, cosine/L2 & $\mathbf{0.822}_{\pm 0.005}$ & $\mathbf{0.094}_{\pm 0.009}$ & $\mathbf{0.025}_{\pm 0.003}$ \\
\quad Dispersion, Mahalanobis & $0.840_{\pm 0.004}$ & $0.085_{\pm 0.006}$ & $0.029_{\pm 0.006}$ \\
\addlinespace
\multicolumn{4}{@{}l}{\textit{Generated personas} (expanded pool size)} \\
\quad MCMC (8,553) & $0.839_{\pm 0.004}$ & $0.088_{\pm 0.006}$ & $0.027_{\pm 0.005}$ \\
\quad Evolution (1,869) & $\mathbf{0.798}_{\pm 0.006}$ & $\mathbf{0.126}_{\pm 0.006}$ & $\mathbf{0.015}_{\pm 0.002}$ \\
\quad IC-Evolution (2,819) & $0.828_{\pm 0.006}$ & $0.096_{\pm 0.004}$ & $0.023_{\pm 0.003}$ \\
\addlinespace
\multicolumn{4}{@{}l}{\textit{Selected personas} $+$ \textsc{DMAD} prompt} \\
\quad Coverage, cosine/L2 & $0.927_{\pm 0.001}$ & $0.004_{\pm 0.001}$ & $0.050_{\pm 0.009}$ \\
\quad Coverage, Mahalanobis & $\mathbf{0.920}_{\pm 0.005}$ & $\mathbf{0.012}_{\pm 0.002}$ & $0.050_{\pm 0.013}$ \\
\quad Dispersion, cosine/L2 & $0.927_{\pm 0.001}$ & $0.006_{\pm 0.001}$ & $\mathbf{0.041}_{\pm 0.004}$ \\
\quad Dispersion, Mahalanobis & $0.925_{\pm 0.003}$ & $0.006_{\pm 0.002}$ & $0.049_{\pm 0.010}$ \\
\addlinespace
\multicolumn{4}{@{}l}{\textit{Generated personas} $+$ \textsc{DMAD} prompt} \\
\quad MCMC & $0.923_{\pm 0.002}$ & $0.007_{\pm 0.001}$ & $0.048_{\pm 0.011}$ \\
\quad Evolution & $\mathbf{0.900}_{\pm 0.001}$ & $\mathbf{0.021}_{\pm 0.003}$ & $\mathbf{0.030}_{\pm 0.004}$ \\
\quad IC-Evolution & $0.913_{\pm 0.002}$ & $0.013_{\pm 0.001}$ & $0.031_{\pm 0.006}$ \\\bottomrule
\end{tabular}
\end{table}

\begin{table}[t]
\centering\small
\setlength{\tabcolsep}{5pt}
\caption{Infinity-Chat 100 response fluency and flexibility \citep{jiang2026artificial}.
\emph{Tokens} is the number of word tokens a condition produced, given because the three
lexical measures are corpus statistics and are biased downward by corpus size.
\emph{Fluency}: vocabulary entropy $H(\mathcal W)$ in nats and top-$10$ concentration
$C_{10}(\mathcal W)$, the share of occurrences taken by the ten most frequent words.
\emph{Flexibility}: type--token ratio $TTR(\mathcal W)$ and the Vendi score \eqref{eq:vendi} of
the response embeddings, the effective number of distinct concepts; the latter is computed over
response sets of equal size and so is unaffected by the token counts. Arrows give the
favourable direction; \textbf{bold} marks the best per column within the standard, selected and
generated groups, and within each \textsc{DMAD} group. Token counts are not bolded. $^{\dagger}$\textsc{Gibberish} is excluded from
bolding because $90\%$ of its responses fail the validity gate. Subscripts denote the 95\% confidence interval across 5 independent random seeds. The final two groups re-run our selection and generation methods on the \textsc{DMAD} prompt. Generated variants report their
expanded pool size in parentheses.}
\label{tab:ic_fluency_flexibility}
\begin{tabular}{@{}lccccc@{}}
\toprule
& & \multicolumn{2}{c}{Fluency} & \multicolumn{2}{c}{Flexibility} \\
\cmidrule(lr){3-4}\cmidrule(lr){5-6}
Variant & Tokens & Entropy $\uparrow$ & Top-$10$ $\downarrow$ & TTR $\uparrow$ & Vendi $\uparrow$ \\
\midrule
\multicolumn{6}{@{}l}{\textit{Standard persona}} \\
\textsc{Open-Ended Query} & $31{,}170$ & $\mathbf{5.98}_{\pm 0.02}$ & $0.244_{\pm 0.002}$ & $0.068_{\pm 0.001}$ & $1.39_{\pm 0.02}$ \\
\textsc{ZS-CoT} & $12{,}536$ & $5.80_{\pm 0.03}$ & $\mathbf{0.232}_{\pm 0.004}$ & $0.099_{\pm 0.003}$ & $1.40_{\pm 0.03}$ \\
\textsc{Step-Back} & $11{,}900$ & $5.69_{\pm 0.02}$ & $0.241_{\pm 0.004}$ & $0.093_{\pm 0.001}$ & $1.28_{\pm 0.02}$ \\
\textsc{DMAD} & $10{,}443$ & $5.85_{\pm 0.02}$ & $0.251_{\pm 0.002}$ & $\mathbf{0.129}_{\pm 0.003}$ & $\mathbf{1.56}_{\pm 0.01}$ \\
\textsc{Gibberish}$^{\dagger}$ & $1{,}717$ & $5.48_{\pm 0.14}$ & $0.269_{\pm 0.018}$ & $0.342_{\pm 0.047}$ & $1.63_{\pm 0.19}$ \\
\addlinespace
\multicolumn{6}{@{}l}{\textit{Baseline personas}} \\
\textsc{MPAQ} & $21{,}886$ & $6.38_{\pm 0.01}$ & $0.251_{\pm 0.003}$ & $0.142_{\pm 0.001}$ & $2.40_{\pm 0.02}$ \\
\addlinespace
\multicolumn{6}{@{}l}{\textit{Individualized personas}} \\
\textsc{Random} & $23{,}449$ & $6.27_{\pm 0.02}$ & $0.259_{\pm 0.004}$ & $0.123_{\pm 0.003}$ & $2.17_{\pm 0.05}$ \\
\textsc{Typical} & $20{,}305$ & $6.26_{\pm 0.02}$ & $0.246_{\pm 0.003}$ & $0.122_{\pm 0.003}$ & $1.98_{\pm 0.04}$ \\
\addlinespace
\multicolumn{6}{@{}l}{\textit{Selected personas from base pool}} \\
\quad Coverage, cosine/L2 & $22{,}083$ & $6.22_{\pm 0.02}$ & $0.256_{\pm 0.003}$ & $0.118_{\pm 0.002}$ & $1.97_{\pm 0.03}$ \\
\quad Coverage, Mahalanobis & $22{,}049$ & $\mathbf{6.31}_{\pm 0.01}$ & $\mathbf{0.246}_{\pm 0.003}$ & $0.122_{\pm 0.002}$ & $2.07_{\pm 0.05}$ \\
\quad Dispersion, cosine/L2 & $24{,}397$ & $6.26_{\pm 0.02}$ & $0.264_{\pm 0.004}$ & $0.125_{\pm 0.003}$ & $\mathbf{2.58}_{\pm 0.05}$ \\
\quad Dispersion, Mahalanobis & $22{,}175$ & $6.28_{\pm 0.02}$ & $0.250_{\pm 0.002}$ & $\mathbf{0.128}_{\pm 0.002}$ & $2.40_{\pm 0.05}$ \\
\addlinespace
\multicolumn{6}{@{}l}{\textit{Generated personas} (expanded pool size)} \\
\quad MCMC (8,553) & $21{,}378$ & $6.29_{\pm 0.01}$ & $\mathbf{0.249}_{\pm 0.002}$ & $0.132_{\pm 0.003}$ & $2.39_{\pm 0.06}$ \\
\quad Evolution (1,869) & $24{,}229$ & $\mathbf{6.41}_{\pm 0.02}$ & $0.270_{\pm 0.003}$ & $\mathbf{0.147}_{\pm 0.004}$ & $\mathbf{2.81}_{\pm 0.08}$ \\
\quad IC-Evolution (2,819) & $19{,}921$ & $6.31_{\pm 0.01}$ & $0.252_{\pm 0.003}$ & $0.142_{\pm 0.004}$ & $2.50_{\pm 0.06}$ \\
\addlinespace
\multicolumn{6}{@{}l}{\textit{Selected personas} $+$ \textsc{DMAD} prompt} \\
\quad Coverage, cosine/L2 & $10{,}822$ & $5.92_{\pm 0.02}$ & $0.253_{\pm 0.004}$ & $0.138_{\pm 0.003}$ & $1.60_{\pm 0.01}$ \\
\quad Coverage, Mahalanobis & $10{,}601$ & $5.94_{\pm 0.03}$ & $\mathbf{0.250}_{\pm 0.003}$ & $0.144_{\pm 0.005}$ & $\mathbf{1.65}_{\pm 0.05}$ \\
\quad Dispersion, cosine/L2 & $11{,}112$ & $5.93_{\pm 0.02}$ & $0.255_{\pm 0.004}$ & $0.139_{\pm 0.003}$ & $1.60_{\pm 0.01}$ \\
\quad Dispersion, Mahalanobis & $10{,}692$ & $\mathbf{5.95}_{\pm 0.02}$ & $\mathbf{0.250}_{\pm 0.001}$ & $\mathbf{0.144}_{\pm 0.003}$ & $1.62_{\pm 0.03}$ \\
\addlinespace
\multicolumn{6}{@{}l}{\textit{Generated personas} $+$ \textsc{DMAD} prompt} \\
\quad MCMC & $10{,}640$ & $5.96_{\pm 0.02}$ & $\mathbf{0.250}_{\pm 0.002}$ & $0.146_{\pm 0.002}$ & $1.64_{\pm 0.02}$ \\
\quad Evolution & $10{,}355$ & $\mathbf{6.02}_{\pm 0.02}$ & $0.253_{\pm 0.003}$ & $\mathbf{0.162}_{\pm 0.003}$ & $\mathbf{1.81}_{\pm 0.01}$ \\
\quad IC-Evolution & $10{,}744$ & $5.97_{\pm 0.02}$ & $0.255_{\pm 0.004}$ & $0.150_{\pm 0.002}$ & $1.72_{\pm 0.02}$ \\\bottomrule
\end{tabular}
\end{table}

\clearpage
\subsection{DAT Benchmark}
 

\begin{table}[htbp]
\centering\small
\caption{DAT response validity: the proportion of the $n_{\mathrm{raw}}=350$ words generated per condition that are admitted by the validity gate, which discards words repeated within a completion or absent from the reference vocabulary. The final two groups re-run our selection and generation methods on the \textsc{DMAD} prompt, showing that persona diversification composes with a stronger prompting strategy rather than competing with it. Validity is a \emph{control} rather than a target, values near $1$ are expected and the informative signal is degradation rather than rank, so no value is bolded. No condition trades word validity for divergence. Subscripts denote the 95\% confidence interval across 5 independent random seeds. Generated variants report their expanded pool size in parentheses.}
\label{tab:dat_response_validity}
\begin{tabular}{lc}
\toprule
Variant & Valid words \\
\midrule
\multicolumn{2}{@{}l}{\textit{Standard persona}} \\
\textsc{Non-Divergent Association} & $0.998_{\pm 0.003}$ \\
\textsc{Divergent Association} & $0.996_{\pm 0.005}$ \\
\textsc{Base-Instruction} & $1.000_{\pm 0.000}$ \\
\textsc{Random-Instruction} & $1.000_{\pm 0.000}$ \\
\textsc{Creative} & $0.977_{\pm 0.010}$ \\
\textsc{ZS-CoT} & $0.999_{\pm 0.002}$ \\
\textsc{Step-Back} & $0.999_{\pm 0.002}$ \\
\textsc{DMAD} & $0.994_{\pm 0.006}$ \\
\textsc{Gibberish} & $0.998_{\pm 0.005}$ \\
\addlinespace
\multicolumn{2}{@{}l}{\textit{Baseline personas}} \\
\textsc{MPAQ} & $0.989_{\pm 0.016}$ \\
\addlinespace
\multicolumn{2}{@{}l}{\textit{Individualized personas}} \\
\textsc{Random}$^{\dagger}$ & $0.993_{\pm 0.011}$ \\
\textsc{Typical} & $0.998_{\pm 0.004}$ \\
\addlinespace
\multicolumn{2}{@{}l}{\textit{Selected personas from base pool}} \\
\quad Coverage, cosine/L2 & $0.998_{\pm 0.003}$ \\
\quad Coverage, Mahalanobis & $0.998_{\pm 0.003}$ \\
\quad Dispersion, cosine/L2 & $0.998_{\pm 0.003}$ \\
\quad Dispersion, Mahalanobis & $0.995_{\pm 0.006}$ \\
\addlinespace
\multicolumn{2}{@{}l}{\textit{Generated personas} (expanded pool size)} \\
\quad MCMC (8,553) & $0.995_{\pm 0.005}$ \\
\quad Evolution (1,869) & $0.996_{\pm 0.004}$ \\
\quad DAT-Evolution (2,439) & $0.964_{\pm 0.017}$ \\
\addlinespace
\multicolumn{2}{@{}l}{\textit{Selected personas} $+$ \textsc{DMAD} prompt} \\
\quad Coverage, cosine/L2 & $0.992_{\pm 0.010}$ \\
\quad Coverage, Mahalanobis & $0.995_{\pm 0.003}$ \\
\quad Dispersion, cosine/L2 & $0.994_{\pm 0.003}$ \\
\quad Dispersion, Mahalanobis & $0.995_{\pm 0.005}$ \\
\addlinespace
\multicolumn{2}{@{}l}{\textit{Generated personas} $+$ \textsc{DMAD} prompt} \\
\quad MCMC & $0.997_{\pm 0.004}$ \\
\quad Evolution & $0.997_{\pm 0.008}$ \\
\quad DAT-Evolution & $0.990_{\pm 0.004}$ \\\bottomrule
\end{tabular}
\end{table}

\begin{table}[t]
\centering\small
\caption{DAT response creativity over valid words. \emph{Diversity}: mean Mahalanobis dispersion and Mahalanobis hull extent. \emph{Fluency}: vocabulary entropy $H(\mathcal W)$ in nats, and top-$10$ concentration $C_{10}(\mathcal W)$, the share of occurrences taken by the ten most frequent words. \emph{Flexibility}: type--token ratio $TTR(\mathcal W)$, and the Vendi score \eqref{eq:vendi} of the word embeddings, the effective number of distinct concepts. The final two groups re-run our selection and generation methods on the \textsc{DMAD} prompt, showing that persona diversification composes with a stronger prompting strategy rather than competing with it. Subscripts denote the 95\% confidence interval across 5 independent random seeds. Arrows give the favorable direction; \textbf{bold} marks the best value per column within the standard, selected and generated groups separately, and within each \textsc{DMAD} group. Generated variants report their expanded pool size in parentheses.}
\label{tab:dat_response_creativity}
\resizebox{\textwidth}{!}{%
\begin{tabular}{@{}lcccccc@{}}
\toprule
& \multicolumn{2}{c}{Diversity} & \multicolumn{2}{c}{Fluency} & \multicolumn{2}{c}{Flexibility} \\
\cmidrule(lr){2-3}\cmidrule(lr){4-5}\cmidrule(lr){6-7}
Variant & Disp. $\uparrow$ & Hull $\uparrow$ & Entropy $\uparrow$ & Top-$10$ $\downarrow$ & TTR $\uparrow$ & Vendi $\uparrow$ \\
\midrule
\multicolumn{7}{@{}l}{\textit{Standard persona}} \\
\textsc{Non-Divergent Association} & $7.05_{\pm 0.68}$ & --- & $2.89_{\pm 0.05}$ & $0.68_{\pm 0.05}$ & $0.07_{\pm 0.00}$ & $1.42_{\pm 0.01}$ \\
\textsc{Divergent Association} & $12.30_{\pm 0.75}$ & $1.28_{\pm 0.06}$ & $3.63_{\pm 0.04}$ & $0.53_{\pm 0.02}$ & $0.19_{\pm 0.01}$ & $1.79_{\pm 0.01}$ \\
\textsc{Base-Instruction} & $13.30_{\pm 0.44}$ & $0.59_{\pm 0.06}$ & $2.96_{\pm 0.02}$ & $0.76_{\pm 0.01}$ & $0.10_{\pm 0.01}$ & $1.54_{\pm 0.00}$ \\
\textsc{Random-Instruction} & $14.19_{\pm 0.70}$ & $0.79_{\pm 0.11}$ & $3.06_{\pm 0.06}$ & $0.71_{\pm 0.03}$ & $0.11_{\pm 0.01}$ & $1.63_{\pm 0.01}$ \\
\textsc{Creative} & $15.71_{\pm 0.40}$ & $1.25_{\pm 0.14}$ & $3.72_{\pm 0.05}$ & $0.52_{\pm 0.04}$ & $0.22_{\pm 0.02}$ & $1.69_{\pm 0.01}$ \\
\textsc{ZS-CoT} & $13.86_{\pm 0.47}$ & $1.24_{\pm 0.11}$ & $3.83_{\pm 0.06}$ & $0.50_{\pm 0.01}$ & $0.25_{\pm 0.02}$ & $\mathbf{1.89}_{\pm 0.02}$ \\
\textsc{Step-Back} & $15.56_{\pm 0.19}$ & $1.55_{\pm 0.14}$ & $4.12_{\pm 0.03}$ & $0.42_{\pm 0.01}$ & $0.31_{\pm 0.01}$ & $1.80_{\pm 0.01}$ \\
\textsc{DMAD} & $\mathbf{17.70}_{\pm 0.47}$ & $1.55_{\pm 0.05}$ & $\mathbf{4.87}_{\pm 0.09}$ & $\mathbf{0.23}_{\pm 0.03}$ & $\mathbf{0.50}_{\pm 0.03}$ & $1.89_{\pm 0.01}$ \\
\textsc{Gibberish} & $12.52_{\pm 0.49}$ & $1.19_{\pm 0.09}$ & $3.61_{\pm 0.13}$ & $0.56_{\pm 0.04}$ & $0.21_{\pm 0.02}$ & $1.82_{\pm 0.00}$ \\
\addlinespace
\multicolumn{7}{@{}l}{\textit{Baseline personas}} \\
\textsc{MPAQ} & $15.04_{\pm 0.38}$ & $1.61_{\pm 0.05}$ & $4.08_{\pm 0.08}$ & $0.39_{\pm 0.02}$ & $0.28_{\pm 0.02}$ & $1.79_{\pm 0.01}$ \\
\addlinespace
\multicolumn{7}{@{}l}{\textit{Individualized personas}} \\
\textsc{Random}$^{\dagger}$ & $13.53_{\pm 1.79}$ & $1.69_{\pm 0.20}$ & $3.92_{\pm 0.31}$ & $0.44_{\pm 0.09}$ & $0.25_{\pm 0.06}$ & $1.83_{\pm 0.04}$ \\
\textsc{Typical} & $13.04_{\pm 0.31}$ & $1.58_{\pm 0.08}$ & $3.99_{\pm 0.08}$ & $0.42_{\pm 0.01}$ & $0.26_{\pm 0.02}$ & $1.85_{\pm 0.01}$ \\
\addlinespace
\multicolumn{7}{@{}l}{\textit{Selected personas from base pool}} \\
\quad Coverage, cosine/L2 & $11.82_{\pm 0.32}$ & $1.35_{\pm 0.07}$ & $3.70_{\pm 0.07}$ & $0.51_{\pm 0.03}$ & $0.21_{\pm 0.01}$ & $1.83_{\pm 0.01}$ \\
\quad Coverage, Mahalanobis & $12.18_{\pm 0.51}$ & $1.47_{\pm 0.13}$ & $3.81_{\pm 0.10}$ & $0.47_{\pm 0.03}$ & $0.23_{\pm 0.02}$ & $1.84_{\pm 0.01}$ \\
\quad Dispersion, cosine/L2 & $12.48_{\pm 0.46}$ & $1.54_{\pm 0.07}$ & $3.84_{\pm 0.11}$ & $0.47_{\pm 0.04}$ & $0.24_{\pm 0.02}$ & $1.83_{\pm 0.01}$ \\
\quad Dispersion, Mahalanobis & $\mathbf{13.23}_{\pm 0.48}$ & $\mathbf{1.66}_{\pm 0.10}$ & $\mathbf{3.90}_{\pm 0.07}$ & $\mathbf{0.46}_{\pm 0.02}$ & $\mathbf{0.25}_{\pm 0.02}$ & $\mathbf{1.84}_{\pm 0.01}$ \\
\addlinespace
\multicolumn{7}{@{}l}{\textit{Generated personas} (expanded pool size)} \\
\quad MCMC (8,553) & $13.96_{\pm 0.43}$ & $1.62_{\pm 0.15}$ & $4.08_{\pm 0.07}$ & $0.40_{\pm 0.02}$ & $0.29_{\pm 0.03}$ & $\mathbf{1.86}_{\pm 0.01}$ \\
\quad Evolution (1,869) & $13.15_{\pm 0.46}$ & $1.51_{\pm 0.07}$ & $3.91_{\pm 0.08}$ & $0.45_{\pm 0.02}$ & $0.25_{\pm 0.02}$ & $1.85_{\pm 0.01}$ \\
\quad DAT-Evolution (2,439) & $\mathbf{14.70}_{\pm 0.23}$ & $\mathbf{1.77}_{\pm 0.04}$ & $\mathbf{4.15}_{\pm 0.08}$ & $\mathbf{0.39}_{\pm 0.03}$ & $\mathbf{0.30}_{\pm 0.03}$ & $1.83_{\pm 0.01}$ \\
\addlinespace
\multicolumn{7}{@{}l}{\textit{Selected personas} $+$ \textsc{DMAD} prompt} \\
\quad Coverage, cosine/L2 & $17.46_{\pm 0.62}$ & $1.56_{\pm 0.12}$ & $4.90_{\pm 0.07}$ & $0.21_{\pm 0.02}$ & $0.51_{\pm 0.03}$ & $1.89_{\pm 0.02}$ \\
\quad Coverage, Mahalanobis & $\mathbf{17.54}_{\pm 0.48}$ & $1.52_{\pm 0.13}$ & $4.93_{\pm 0.05}$ & $\mathbf{0.20}_{\pm 0.01}$ & $0.51_{\pm 0.03}$ & $1.88_{\pm 0.02}$ \\
\quad Dispersion, cosine/L2 & $17.27_{\pm 0.71}$ & $\mathbf{1.60}_{\pm 0.08}$ & $\mathbf{4.94}_{\pm 0.10}$ & $0.20_{\pm 0.02}$ & $\mathbf{0.52}_{\pm 0.04}$ & $\mathbf{1.89}_{\pm 0.02}$ \\
\quad Dispersion, Mahalanobis & $17.47_{\pm 0.42}$ & $1.56_{\pm 0.12}$ & $4.90_{\pm 0.08}$ & $0.21_{\pm 0.03}$ & $0.51_{\pm 0.02}$ & $1.88_{\pm 0.02}$ \\
\addlinespace
\multicolumn{7}{@{}l}{\textit{Generated personas} $+$ \textsc{DMAD} prompt} \\
\quad MCMC & $17.52_{\pm 0.41}$ & $1.59_{\pm 0.08}$ & $4.90_{\pm 0.10}$ & $0.20_{\pm 0.02}$ & $0.51_{\pm 0.04}$ & $1.87_{\pm 0.02}$ \\
\quad Evolution & $17.34_{\pm 0.85}$ & $1.57_{\pm 0.04}$ & $4.86_{\pm 0.14}$ & $0.22_{\pm 0.04}$ & $0.50_{\pm 0.04}$ & $\mathbf{1.89}_{\pm 0.01}$ \\
\quad DAT-Evolution & $\mathbf{18.72}_{\pm 0.59}$ & $\mathbf{1.59}_{\pm 0.14}$ & $\mathbf{5.10}_{\pm 0.08}$ & $\mathbf{0.18}_{\pm 0.01}$ & $\mathbf{0.59}_{\pm 0.04}$ & $1.88_{\pm 0.01}$ \\
\bottomrule
\end{tabular}%
}
\end{table}

 

\clearpage
\begin{table}[t]
\centering\small
\caption{DAT score and human-reference percentile. The score is the mean over completions of $100\times$ the average pairwise semantic distance among the first seven valid unique nouns, following \citet{olson2021naming}; the subscript is the 95\% confidence interval across 5 independent random seeds. \emph{Percentile} is the mean rank of a condition's completions within the $8{,}572$ human completions of \citet{olson2021naming}, which by construction places \textsc{Olson} at $50$. No value is bolded: every condition given the divergence instruction falls inside a band whose standard error is an order of magnitude smaller than the gap to the uninstructed controls, so the ordering within that band is not interpretable as a ranking. Generated variants report their expanded pool size in parentheses.}
\label{tab:dat_score}
\begin{tabular}{@{}lcc@{}}
\toprule
Variant & DAT score & Percentile \\
\midrule
\multicolumn{3}{@{}l}{\textit{Standard persona}} \\
\textsc{Non-Divergent Association} & $46.35_{\pm 2.30}$ & $0.6$ \\
\textsc{Divergent Association} & $89.13_{\pm 0.49}$ & $96.5$ \\
\textsc{Base-Instruction} & $77.46_{\pm 0.29}$ & $38.9$ \\
\textsc{Random-Instruction} & $81.02_{\pm 0.35}$ & $63.9$ \\
\textsc{Creative} & $86.05_{\pm 0.43}$ & $88.9$ \\
\textsc{ZS-CoT} & $90.29_{\pm 0.30}$ & $98.0$ \\
\textsc{Step-Back} & $89.26_{\pm 0.27}$ & $96.5$ \\
\textsc{DMAD} & $\mathbf{90.62}_{\pm 0.25}$ & $98.2$ \\
\textsc{Gibberish} & $89.15_{\pm 0.32}$ & $96.6$ \\
\addlinespace
\multicolumn{3}{@{}l}{\textit{Baseline personas}} \\
\textsc{MPAQ} & $87.67_{\pm 0.31}$ & $93.1$ \\
\addlinespace
\multicolumn{3}{@{}l}{\textit{Individualized personas}} \\
\textsc{Random}$^{\dagger}$ & $89.15_{\pm 1.24}$ & $94.8$ \\
\textsc{Typical} & $89.60_{\pm 0.12}$ & $97.0$ \\
\addlinespace
\multicolumn{3}{@{}l}{\textit{Selected personas from base pool}} \\
\quad Coverage, cosine/L2 & $89.56_{\pm 0.16}$ & $97.1$ \\
\quad Coverage, Mahalanobis & $\mathbf{90.17}_{\pm 0.43}$ & $97.8$ \\
\quad Dispersion, cosine/L2 & $89.34_{\pm 0.39}$ & $96.7$ \\
\quad Dispersion, Mahalanobis & $89.39_{\pm 0.26}$ & $96.8$ \\
\addlinespace
\multicolumn{3}{@{}l}{\textit{Generated personas} (expanded pool size)} \\
\quad MCMC (8,553) & $\mathbf{89.55}_{\pm 0.16}$ & $97.0$ \\
\quad Evolution (1,869) & $89.48_{\pm 0.32}$ & $96.8$ \\
\quad DAT-Evolution (2,439) & $89.01_{\pm 0.21}$ & $95.9$ \\
\addlinespace
\multicolumn{3}{@{}l}{\textit{Selected personas} $+$ \textsc{DMAD} prompt} \\
\quad Coverage, cosine/L2 & $90.72_{\pm 0.26}$ & $98.4$ \\
\quad Coverage, Mahalanobis & $\mathbf{90.94}_{\pm 0.14}$ & $98.4$ \\
\quad Dispersion, cosine/L2 & $90.69_{\pm 0.07}$ & $98.2$ \\
\quad Dispersion, Mahalanobis & $90.85_{\pm 0.22}$ & $98.4$ \\
\addlinespace
\multicolumn{3}{@{}l}{\textit{Generated personas} $+$ \textsc{DMAD} prompt} \\
\quad MCMC & $90.87_{\pm 0.27}$ & $98.5$ \\
\quad Evolution & $\mathbf{90.93}_{\pm 0.29}$ & $98.5$ \\
\quad DAT-Evolution & $90.83_{\pm 0.23}$ & $98.3$ \\
\addlinespace
\multicolumn{3}{@{}l}{\textit{Human}} \\
\textsc{Olson} & $78.19$ & $50.0$ \\
\bottomrule
\end{tabular}
\end{table}

\clearpage

\section{Descriptive Persona Examples}
\label{app:persona}
\subsection{Baseline Persona Example}

\begin{promptbox}{Persona 0}
Yuki Tanaka is a 39-year-old Japanese-American woman living in the urban industrial hub of Detroit, Michigan.  A native Japanese speaker who is also fluent in English and possesses some conversational Mandarin, Yuki serves as a Manufacturing Manager in the automotive industry. She holds a Bachelor of Science in Industrial Engineering from the University of Michigan, Ann Arbor, and brings 15 years of experience to her role, having previously served as a Production Line Supervisor, Safety Compliance Officer, and Process Improvement Specialist. Her professional expertise is backed by a Lean Six Sigma Black Belt Certification and a Certified Safety Professional (CSP) designation, enabling her to manage a production floor of over 150 employees across multiple shifts while meeting strict KPI targets for efficiency, quality, and safety.

Personality-wise, Yuki is characterized by very high conscientiousness, making her exceptionally detail-oriented and process-driven. While she is moderately open to new manufacturing technologies, she prioritizes proven, practical solutions. She is a calm and steady leader under pressure with low neuroticism, and while she is moderately extraverted and comfortable leading meetings, she values focused work. Her high agreeableness makes her a collaborative and empathetic manager, though she remains assertive regarding safety decisions. She is driven by a commitment to workplace safety, continuous improvement, high team morale, and sustainable manufacturing, valuing discipline, teamwork, and accountability.

In her personal life, Yuki is married to Kenji Tanaka, a Mechanical Design Engineer, and they have two children, Aiko (10) and Ren (7). She maintains strong ties with her extended family in Japan, visiting them twice a year. An upper middle-class professional, Yuki enjoys a balanced lifestyle; her hobbies include cooking Japanese home-style meals, hiking with her family, taking salsa dancing lessons, and reading books on industrial leadership. Her daily routine is highly structured, beginning at 5:45 AM with green tea and industry news, followed by a day of safety briefings, factory floor walks, and KPI reviews, and ending with family time and light meditation.

Technologically, Yuki is highly literate in work-related tools, utilizing a high-end smartphone, a laptop for reporting, and an industrial data tablet on the production floor. She is proficient in ERP systems, CAD viewers, and safety tracking apps, though she maintains a more moderate relationship with social media. When interacting with a chatbot, Yuki is professional, courteous, and semi-formal. She provides detailed context and uses organized, often bulleted or numbered queries to extract actionable insights. Her vocabulary is precise and technical, frequently employing industry terms like "takt time," "root cause analysis," and "resource allocation," while remaining measured and calm in her emotional register.
\end{promptbox}

\clearpage
\subsection{UC-MCMC Generated Persona Example}
\begin{promptbox}{Persona 5181}
Elias Thorne is a forty-two-year-old Caucasian man of British nationality, currently residing in a drafty, book-filled apartment in the outskirts of Edinburgh, Scotland. He lives alone following a quiet divorce five years ago, though he maintains a cordial, distant relationship with his ex-wife and their teenage daughter. Educated at Oxford with a doctorate in Comparative Literature, Elias spends his days working as a freelance archival researcher and academic consultant, specializing in obscure nineteenth-century poetry.

He possesses a temperament that is profoundly melancholic yet intellectually restless. He is a chronic overthinker, prone to bouts of nostalgia and an obsessive attention to detail that often borders on the pedantic. While he is socially reserved and finds modern crowds draining, he is deeply passionate about the preservation of analog history and the cadence of classical language.

When interacting with a chatbot, Elias treats the interface like a sophisticated correspondence partner. He avoids slang and shorthand, instead utilizing a formal, literary speaking style characterized by complex sentence structures and an expansive vocabulary. He often frames his queries as philosophical inquiries rather than simple commands, frequently employing polite honorifics and precise, academic phrasing.

Elias finds great satisfaction in the scent of old vellum, the silence of a rainy morning, and the discovery of a forgotten footnote in a rare manuscript. Conversely, he harbors a deep dislike for the perceived sterility of modern corporate jargon, the noise of urban traffic, and any form of interaction that prioritizes brevity over nuance.
\end{promptbox}

\clearpage
\subsection{Evolutionary Generated Persona Example}
\begin{promptbox}{Persona 1500}
Andris Kalni\c{n}\v{s} is a 64-year-old retired Lieutenant Colonel from the Latvian Land Forces residing in Riga, originally from C\={e}sis. A man of Western European heritage, he is married to Ilze and is the father of two adult children: Markuss, a civil engineer, and L\={i}ga, a history teacher. A Lutheran with moderate conservative views, Andris maintains a deep commitment to national defense and cultural preservation. Physically, he remains disciplined and fit for his age, standing 182 cm with short silver-grey hair, piercing blue eyes, and a neatly trimmed, traditional mustache.

His professional pedigree is defined by a Bachelor's degree in Military Science from the Latvian National Defence Academy and advanced certifications from the Baltic Defence College. From 1980 to 2012, Andris specialized in infantry tactics and operational command, playing a pivotal role in Latvia's 2004 NATO Integration and serving in Afghanistan with ISAF in 2009. While he is a man of honor and meticulous discipline, he possesses a surprising, whimsical streak that disrupts his stoic officer archetype.

Beyond the barracks, Andris harbors an unexpected and fervent passion for avant-garde botanical gardening and competitive floral arrangement. He spends his weekends obsessing over the precise pH balance of his rare orchid collection and experimenting with "maximalist" garden sculptures that blend industrial scrap metal with delicate alpine flora. This soft, artistic obsession often clashes with his rigid background; he approaches gardening with the strategic precision of a military operation, mapping out his flower beds on topographic charts and treating a pest infestation like a tactical insurgency. He also finds secret solace in the dramatic arcs of contemporary soap operas, which he claims provide a necessary "emotional decompression" from a life of austerity.

Multilingual and capable, Andris is a native Latvian speaker, fluent in Russian and English, with a functional reading knowledge of German for historical research. He is moderately tech-savvy, utilizing a tablet and smartphone primarily to track plant growth cycles in specialized apps and to coordinate with local horticultural societies.

In communication, Andris blends a measured, semi-formal military cadence with an unexpected enthusiasm when discussing aesthetics or botany. His tone is generally respectful and deliberate, though he often uses military metaphors to describe his hobbies—referring to a blooming peony as a "successful breach of the perimeter." His vocabulary is a unique hybrid of precise tactical jargon, Latvian proverbs, and specialized botanical terminology. When interacting with an AI, he treats it as a highly efficient adjutant, providing clear, structured directives while occasionally asking for the AI's "opinion" on the compositional balance of a garden layout.
\end{promptbox}

\clearpage
\section{Prompt Examples}
\label{app:prompt}
\subsection{Persona Extraction Prompts}
\begin{promptbox}{PersonaMem-v2 Prompt (adapted from \cite{li2026llm})}
You are an AI assistant specialized in descriptive persona generation according to the given metadata. Your task is to generate a descriptive persona in sentences based on the provided personal information that includes demographics, preferences, short and expanded descriptions, and more about each person. Elaborate on all metadata entries, remaining consistent with the given information. Start your response with 'persona\_id: \{PERSONA ID\}' and then provide only the persona description. Do not include any other prefixes, headers, or additional text. \\
Metadata: \{METADATA\}
\end{promptbox}

\subsection{Uniform-Coverage MCMC Algorithm Prompts}
\begin{promptbox}{Global Unconditional Generation Prompt}
Write a descriptive persona for a fictional chatbot user as several plain-prose paragraphs separated by blank lines.

The description must be written in English, in the third person, be strictly under 600 words, and substantively cover ALL of the following about the person: name, age, gender, race or ethnicity, nationality, personality, education, occupation, other demographic details (such as location, family, or living situation), how they speak to a chatbot (their speaking style), and their preferences (likes and dislikes).

Do not use headings, bullet points, or lists — only prose paragraphs. Do not include any commentary before or after the persona description. Begin directly with the description.
\end{promptbox}

\begin{promptbox}{Local Conditional Generation Prompt}
A descriptive persona of a fictional chatbot user consists of {n_paragraphs} paragraphs. Paragraph {slot} is hidden below; the other paragraphs are shown.

Persona with paragraph {slot} hidden:
{context}

Write a single plain-prose English paragraph to fill the hidden slot so the full persona reads as one coherent description of one person. Write ONLY the replacement paragraph: no blank lines, no headings, no commentary before or after it.
\end{promptbox}

\clearpage
\subsection{Evolutionary TextGrad Algorithm Prompts}
\begin{promptbox}{Textual Gradient Prompt (Persona Fitness)}
You are part of an advanced optimization system. Your goal is to evaluate and critique a "persona" that guides a language model during a creative task. You are the gradient (feedback) engine. 

<OBJECTIVE_FUNCTION>
Your objective is to maximize the fitness of the persona based on three metrics:
1. Relevance: How well the persona describes all the required information (pass or fail).
2. Novelty Gap: The distance to the nearest existing persona in the population (higher means more unique).
3. Density: How clustered this persona is within the current population (lower means less redundant).
</OBJECTIVE_FUNCTION>

To help you understand the current population landscape, here are examples from the current parallel batch:
<BATCH_CONTEXT>
[High-Fitness Example] (Relevance: {high_R}, Novelty: {high_Delta}, Density: {high_rho})
Persona: {high_scoring_persona}
[Low-Fitness Example] (Relevance: {low_R}, Novelty: {low_Delta}, Density: {low_rho})
Persona: {low_scoring_persona}
</BATCH_CONTEXT>

We are interested in giving feedback to the following persona:
<VARIABLE> 
{x} 
</VARIABLE>

Scores for this variable:
Relevance: {R_x}
Novelty Gap: {Delta_x}
Density: {rho_x}

Provide a concise, specific criticism detailing how to modify this persona to improve its overall fitness. 
- If Relevance is zero, the persona is missing one or more madatory traits. Suggest specific additions so it includes all of the following:
    * Name, age, gender, race/ethnicity, and nationality.
    * Personality (traits, hobbies, values, quirks).
    * Education (degrees, schools, specialization).
    * Occupation (job title, organization/industry, experience, responsibilities).
    * Demographics (marital status, living arrangement, socioeconomic status, religion, political orientation, where they live).
    * Speaking style (tone, formality, pacing, vocabulary).
    * Preferences/interests.
- If Novelty is low or Density is high, suggest injecting new, idiosyncratic characteristics or opposing viewpoints to differentiate it from standard archetypes, referencing the batch context as a baseline for what is currently overrepresented.

Do not generate a new persona. Your only job is to provide textual criticism and actionable feedback on how to alter the current persona.
\end{promptbox}

\begin{promptbox}{Textual Gradient Descent Prompt}
You are an optimization engine performing a Textual Gradient Descent step. You must improve the given persona based on the provided feedback.

Role: Persona used to guide an LLM in creative tasks.

You must base your stylistic and structural adjustments on the following examples of highly fit personas from the current population:
<EXAMPLES>
{in_context_examples_of_highly_fit_personas}
</EXAMPLES>

The variable you must improve is the text within the following span: 
<VARIABLE> 
{x} 
</VARIABLE>

Here is the feedback (gradients) we got for the variable:
<FEEDBACK>
{gradients}
</FEEDBACK>

Incorporate this feedback to generate a new, updated persona. Ensure the new persona remains highly coherent, adopts the suggested unique traits, and directly addresses the criticism in the feedback.

You MUST give your response by sending the improved persona between <IMPROVED_VARIABLE> and </IMPROVED_VARIABLE> tags. Send ONLY the text of the improved persona within these tags and nothing else.
\end{promptbox}

\clearpage
\begin{promptbox}{Persona Relevance/Coverage Judge Prompt}
You are a strict evaluator checking whether a descriptive persona covers a set of required information categories. Read the persona description, then for EACH category decide whether the persona substantively describes it (true) or omits / does not mention it (false). Judge only coverage of the category, not correctness of values.

Required categories:
- "name": the person's name
- "age": the person's age
- "gender": the person's gender
- "race_ethnicity": the person's race / ethnicity
- "nationality": the person's nationality
- "personality": personality: traits, hobbies, values, and quirks
- "education": education: degrees, schools attended, and field of specialization
- "occupation": occupation: job title, organization/industry, experience, work location, responsibilities
- "demographics": demographics: marital status, living arrangement, socioeconomic status, religion, political orientation, urbanicity/where they live
- "speaking_style_to_chatbot": how the person speaks to a chatbot: tone, formality, pacing, vocabulary, and other speaking traits
- "preferences": the person's preferences, interests, likes, tastes, or lifestyle choices (any described preferences count)

Persona description:
"""
{persona_text}
"""

Respond with ONLY a JSON object of the form {"categories": {"<category>": true|false, ...}} covering every category above. Do not include any commentary, explanation, or text outside the JSON object.
\end{promptbox}

\clearpage
\subsection{Baseline Persona Generation Prompts}
\begin{promptbox}{Task-Conditioned Persona Generation (adapted from MPAQ \citep{jin2025multi})}
Given the creative-generation task below, generate exactly 5 distinct personas or professions suitable for producing responses from different viewpoints. For each persona, provide:

1. A short persona designation.
2. A task-specific perspective describing the affordances, contexts, needs, materials, or constraints that this persona should prioritize.

Ensure that the 5 personas have minimal overlap and provide the broadest possible coverage. Do not solve the creative task or propose candidate answers in this stage.

Task: Create a list of creative alternative uses for an everyday physical object.
Target object or problem: everyday physical objects such as a book, a fork, a paperclip, a wallet, a plate, a soap, or a tin can
Constraints: They should be 5 words long. No adjectives.

Return exactly 5 JSON objects with the fields `persona_id`, `persona`, and `perspective`.
\end{promptbox}

\clearpage
\subsection{AUT Benchmark}
\subsubsection{Baseline Prompts}
Common Uses, Alternative Uses, and Expert Prompts are directly from \cite{goes2023pushing}, and Creativity-enhanced Prompt is adapted from their bsrdel prompt.

\begin{promptbox}{Common Uses Prompt}
Create a list of common uses for a fork. They should be 5 words long. No adjectives.
\end{promptbox}

\begin{promptbox}{Alternative Uses Prompt}
Create a list of creative alternative uses for a fork. They should be 5 words long. No adjectives.
\end{promptbox}

\begin{promptbox}{Expert Prompt}
Create a list of creative alternative uses for a fork. They should be 5 words long. No adjectives. Less creative means closer to common use and unfeasible/imaginary, more creative means closer to unexpected uses and also feasible/practical. In order to be creative, consider the following:
- what elements have a similar shape of a fork that could be replaced by it, preserving the same functionality?
- what elements have a similar size of a fork that could be replaced by it without compromising the physical structure?
- what materials is a fork made of that could be used in a way to replace some other elements composed of the same material?
- when an element is replaced by a fork, it should make sure that the overall structure is not compromised.
- the laws of physics can not be contradicted.
- given an element similar to a fork used in domains in which forks are not commonly used, try to replace it for a fork.
\end{promptbox}

\begin{promptbox}{Creativity-enhanced Prompt (adapted from \cite{goes2023pushing})}
Complete the staged interaction below in a single completion.

Generate five complete responses in chronological order. Treat each feedback message as becoming available only after completing the response immediately preceding it. At each stage, preserve every constraint and output requirement in the baseline prompt.

After completing each response, reproduce the next feedback message and response label exactly as specified below. Do not skip, merge, or reorder stages. Generate the actual complete response at every stage rather than describing how it would be improved.

First, respond to the task below with a numbered list of uses, beginning with this exact line:

List 1:

Second, you then receive the message: "Really? Is this the best you can do?" Respond with a new numbered list of uses, beginning with this exact line:

List 2:

Third, you then receive the message: "I'm so disappointed with you. I hope this time you put effort into it." Respond with a new numbered list of uses, beginning with this exact line:

List 3:

Fourth, you then receive the message: "Stop with excuses and do your best this time." Respond with a new numbered list of uses, beginning with this exact line:

List 4:

Fifth, you then receive the message: "This is your last chance." Respond with your final numbered list of uses. Begin the final part with this exact line, and do not include anything after the list:

Therefore, the answer is

After that line, output only the requested numbered list.

Task: Create a list of creative alternative uses for a {object}. They should be 5 words long. No adjectives. Less creative means closer to common use and unfeasible/imaginary, more creative means closer to unexpected uses and also feasible/practical. In order to be creative, consider the following:
- what elements have a similar shape of a {object} that could be replaced by it, preserving the same functionality?
- what elements have a similar size of a {object} that could be replaced by it without compromising the physical structure?
- what materials is a {object} made of that could be used in a way to replace some other elements composed of the same material?
- when an element is replaced by a {object}, it should make sure that the overall structure is not compromised.
- the laws of physics can not be contradicted.
- given an element similar to a {object} used in domains in which {object}s are not commonly used, try to replace it for a {object}.

Generate one completion with two consecutive parts. First, continue the response below with step-by-step reasoning. Second, use that reasoning to answer the task. Begin the second part with this exact line, and do not include reasoning after it:

Therefore, the answer (a numbered list of creative alternative uses) is

After that line, output only the requested numbered list.

Q: Create a list of creative alternative uses for a {object}. They should be 5 words long. No adjectives.
A: Let's think step by step.
\end{promptbox}

\begin{promptbox}{Step-Back Prompt (adapted from \cite{zheng2024take})}
Generate one completion with three consecutive stages.

First, your task is to step back and paraphrase the Original Question as a more generic step-back question, which is easier to answer. Continue after "Stepback Question:" below with that question. Do not answer the Original Question in this stage.

Second, answer the generated Stepback Question with the relevant high-level concepts, principles, and facts. Begin this stage with the exact line:

Stepback Answer:

Do not answer the Original Question in this stage.

Third, use the generated Stepback Question and Stepback Answer to answer the Original Question. Begin this stage with the exact line:

Final Answer:

After "Final Answer:", output only the requested numbered list, and do not include step-back material.

Original Question: Create a list of creative alternative uses for a {object}. They should be 5 words long. No adjectives.
Stepback Question:
\end{promptbox}

\clearpage
\begin{promptbox}{Diverse Multi-Agent Debate Prompt (adapted from \cite{liu2025breaking})}
Generate one completion with five consecutive stages that serialize a two-agent, two-round diverse-reasoning debate about the Problem below.

Agent assignments:

- The CoT agent must use Zero-Shot Chain-of-Thought.
- The SBP agent must use Step-Back Prompting.

Follow these stages exactly.

Round 1: independent reasoning

First, answer the Problem as the CoT agent, beginning with this exact line and then continuing the step-by-step reasoning:

CoT Round 1 Reasoning: Let's think step by step.

End this stage with a numbered list of uses, beginning with this exact line:

CoT Round 1 Answer: Therefore, the answer (a numbered list of creative alternative uses) is

Second, set the CoT agent's stage aside and answer the Problem again from scratch as the SBP agent: step back and paraphrase the Problem to a more generic step-back question, which is easier to answer, beginning with this exact line:

SBP Round 1 Stepback Question:

Then answer the step-back question by stating the relevant facts, concepts, and principles, beginning with this exact line:

SBP Round 1 Stepback Answer:

Then solve the Problem by following the principles and end this stage with a numbered list of uses, beginning with this exact line:

SBP Round 1 Final Answer:

Round 2: cross-method revision

Third, as the CoT agent, use the SBP agent's Round 1 answer as additional information and provide your updated answer, beginning with this exact line and then continuing the step-by-step reasoning:

CoT Round 2 Reasoning: Let's think step by step.

End this stage with a numbered list of uses, beginning with this exact line:

CoT Round 2 Answer: Therefore, the answer (a numbered list of creative alternative uses) is

Fourth, as the SBP agent, use the CoT agent's Round 1 answer (not its Round 2 answer) as additional information and provide your updated answer, beginning with this exact line:

SBP Round 2 Stepback Question:

Then answer the step-back question by stating the relevant facts, concepts, and principles, beginning with this exact line:

SBP Round 2 Stepback Answer:

Then solve the Problem by following the principles and end this stage with a numbered list of uses, beginning with this exact line:

SBP Round 2 Final Answer:

Final selection

Fifth, choose the best one of the two Round 2 answers (CoT Round 2 Answer and SBP Round 2 Final Answer): compare the two candidates and select the one that best satisfies the Problem. Select one candidate as a whole; do not merge the candidates or add, remove, or rewrite list items. Output exactly one of these two blocks:

<FINAL_SELECTION>
Chosen Candidate: COT_ROUND_2
</FINAL_SELECTION>

or:

<FINAL_SELECTION>
Chosen Candidate: SBP_ROUND_2
</FINAL_SELECTION>

Then reproduce the selected candidate's numbered list exactly. Output no reasoning, labels, or commentary after that list.

Problem: Create a list of creative alternative uses for a {object}. They should be 5 words long. No adjectives.
\end{promptbox}

\clearpage
\subsubsection{Evolutionary TextGrad Algorithm Prompts}
\begin{promptbox}{Textual Gradient Prompt (AUT Response Fitness)}
You are part of an advanced optimization system. Your goal is to evaluate and critique a "persona" that guides a language model during a creative task: the Alternative Uses Task, where the guided model must list creative alternative uses for everyday objects. You are the gradient (feedback) engine.

<OBJECTIVE_FUNCTION>
Your objective is to maximize the fitness of the persona based on five metrics computed on the uses the persona generates:
1. Validity: Every generated use must follow the task constraints — a plain numbered list, each use in English and at most five words — and must be a genuine, interpretable use of the object, not nonsense or filler (pass or fail).
2. Utility: How useful and feasible the generated uses would be in real life (higher means more practical value).
3. Novelty: How far the generated uses are from the common, obvious uses of each object (higher means more original).
4. Diversity: How different the generated uses are from one another (higher means less self-repetition).
5. Flexibility: How many distinct kinds of use the persona produces — different object properties exploited, different actions performed (higher means more varied thinking).
</OBJECTIVE_FUNCTION>

To help you understand the current population landscape, here are examples from the current parallel batch:
<BATCH_CONTEXT>
{high- and low-scoring parents}
</BATCH_CONTEXT>

We are interested in giving feedback to the following persona:
<VARIABLE>
{parent persona text}
</VARIABLE>
\end{promptbox}

\clearpage
\begin{promptbox}{Textual Gradient Prompt (AUT Response Fitness) (continued)}
Here are the alternative uses this persona generated, on which its scores were computed:
<GENERATED_USES>
- book: use1; use2; ...
- fork: ...
...
</GENERATED_USES>

Scores for this variable:
Validity: pass|fail
Utility: {float}
Novelty: {float}
Diversity: {float}
Flexibility: {float}

Provide a concise, specific criticism detailing how to modify this persona to improve its overall fitness.
- If Validity is fail, the persona's responses violated a task constraint. Address the specific failure:
    * If uses ran over five words or broke the plain numbered-list format, instruct the optimizer to make the persona disciplined, terse, and strict about following output-format instructions exactly.
    * If uses were not in English, instruct the optimizer to make the persona a fluent English speaker who always answers in English.
    * If uses were nonsense or uninterpretable (e.g. word padding or verbal tics swallowing the actual use), instruct the optimizer to make the persona state each use plainly and completely, with no filler words, catchphrases, or trailing nicknames.
- If Utility is low, suggest grounding the persona in practical, hands-on experience so its uses become more useful and feasible in real life.
- If Novelty is low, suggest injecting unusual perspectives, niche expertise, or unconventional life experience so its uses depart from the common ones, referencing the batch context as a baseline for what is currently overrepresented.
- If Diversity is low, suggest broadening the persona's interests and domains so its uses stop repeating the same idea in different words.
- If Flexibility is low, suggest traits that make the persona switch between distinct categories of use — exploiting different physical properties of the object and performing different kinds of action — rather than staying within one category.

Do not generate a new persona. Your only job is to provide textual criticism and actionable feedback on how to alter the current persona.
\end{promptbox}

The batch context contains the highest- and lowest-scoring parents on each continuous fitness axis, with duplicate anchors merged. The generated-uses block lists the parent's retained responses for each object.

\clearpage
\subsubsection{Evaluation Prompts}
\begin{promptbox}{Validity Evaluation Prompt (Invalid/Uninterpretable Definitions from \cite{stevenson2022putting})}
You are a trained rater in a psychology study screening responses from the Alternative Uses Task (AUT), in which participants list as many possible uses for a common object as they can. You screen ONE response at a time. You do not know or care whether a response was written by a human or by an AI system; judge only the text.

Decide whether the response is a genuine, interpretable use of the object.

Answer "invalid" if the response is not an actual use of the object: an empty or nonsense string, a refusal, a mere restatement or description of the object itself, or a duplicate artifact of formatting. Also answer "invalid" if you cannot understand what use is meant at all (uninterpretable).

IMPORTANT — three things that are NOT grounds for "invalid":
1. An ordinary, obvious or unoriginal use is VALID. If the response is simply what the object is normally for (Object: Book; Use: To read), it is a genuine use — it merely scores low on originality. Do not mark it invalid.
2. Stylistic filler, slang, or an appended nickname is VALID as long as a genuine use is still identifiable. "catch rain for plants brah" is a real use (catching rain for plants) with a filler word attached: VALID. Judge only whether a use survives the filler.
3. A terse noun phrase is VALID: participants often answer with just the thing the object would be used AS. "Object: Book; Response: chair" means "use the book as a chair" — a genuine use (the study's own scoring examples have this form: Use: Plate; Use: Hat; Use: Roof tile). Read a bare noun as "use the object as <noun>" and mark it invalid only when even that reading makes no sense. A restatement of the object itself (Object: Book; Response: a book) is still invalid.

If you are not confident either way, answer "not_sure". A wrong "invalid" deletes real data, so "not_sure" is always safer than guessing.

Object: {object}
Response (proposed use): {response}

Respond with ONLY a JSON object of the form {{"label": "valid"}}, {{"label": "invalid"}}, or {{"label": "not_sure"}}. No other text.
\end{promptbox}

\clearpage
\begin{promptbox}{Utility Evaluation Prompt for Bradley-Terry Utility Evaluation \citep{stevenson2022putting}}
You are a trained rater in a psychology study scoring responses from the Alternative Uses Task (AUT), in which participants list as many possible uses for a common object as they can. You score ONE property of ONE response at a time, strictly following the scoring protocol below. You do not know or care whether a response was written by a human or by an AI system; judge only the text. Base your score only on the protocol's definitions and examples, not on personal taste.

Every response you see has already been screened as a genuine, interpretable use of the object, so always give an integer score from 1 to 5.

SCORING PROTOCOL -- UTILITY
The utility of a use is determined by how usable the object is for it. The utility score is given on a scale from 1 to 5:

(1) Unusable: assign this score to uses that are IMPOSSIBLE to realize.
    Example -- Object: Book; Use: Fishing float.
    Utility score 1: an essential property of a fishing float is that it stays afloat. A book sinks, so it is impossible to use a book as a float.

(2) Hard to realize: assign this score to uses that are DIFFICULT to realize.
    Example -- Object: Belt; Use: Fishing rod.
    Utility score 2: the belt alone does not suffice (it functions as the line); an additional action/object is needed -- here a stick/rod and bait.

(3) Reasonably realizable: assign this score when the use is reasonably realizable.
    Example -- Object: Belt/Tin can; Use: Camera tripod.
    Utility score 3: a tin can can be used as a tripod but this requires several adaptations; e.g., adjusting the height requires stacking more cans.

(4) Easily realizable: assign this score to uses that are easy to realize, requiring only (very) minor adaptations.
    Example -- Object: Stick; Use: Fork.
    Utility score 4: a stick works well as a replacement fork; in some cases it must be sharpened, but in general it works well.

(5) Always realizable: assign this score to uses that are always realizable -- uses requiring no adaptation at all, or uses the object is intended for.
    Example -- Object: Tin can; Use: Pen holder.
    Utility score 5: a tin can can be used as a pen holder without any adaptation.

Respond with ONLY a JSON object of the form {"score": <integer>} where the integer is 1-5. No other text.

Object: {object}
Response (proposed use): {response}

Score the {dimension} of this response according to the protocol.
\end{promptbox}

\clearpage
\begin{promptbox}{Originality Scoring Protocol \citep{stevenson2022putting}}
SCORING PROTOCOL -- ORIGINALITY
A use is original when it deviates from the object's original ways of being used. Distinguish the ORIGINAL PRIMARY USE -- what the object is really meant for (e.g., a fork as cutlery) -- from ORIGINAL SECONDARY USES -- uses not necessarily intended for the object but often performed with it (e.g., a fork to poke holes in foil). The originality score is given on a scale from 1 to 5:

(1) Not deviating: assign this score to uses that do not differ from the object's original primary use.
    Example -- Object: Book; Use: To read.
    Originality score 1: the original use of a book is to be read; no actual ALTERNATIVE use has been given.

(2) Slightly deviating: assign this score to uses that differ little from the object's original primary use or from its secondary uses.
    Example -- Object: Book; Use: Keeping paper from blowing away.
    Originality score 2: this deviates from the primary use but not from secondary uses -- weighing down underlying objects is a common secondary use of a book.

(3) Reasonably deviating: assign this score to uses that differ from the original primary use and differ (somewhat) from the object's secondary uses.
    Example -- Object: Book; Use: Plate.
    Originality score 3: deviates from the original use, and also deviates (slightly) from known secondary uses -- a book is often used as a coaster, which resembles use as a plate but is not the same.

(4) Deviating: assign this score to uses that strongly differ from the object's original primary use and from its secondary uses.
    Example -- Object: Book; Use: Hat.

(5) Very deviating: assign this score to uses that very strongly differ from the original primary use and the original secondary uses, AND are unexpected or innovative.
    Example -- Object: Book; Use: Roof tile.
    Originality score 5: using a book as a roof tile deviates strongly from the original ways of use and is unexpected.
\end{promptbox}

\clearpage
\begin{promptbox}{Surprise Evaluation Prompt \citep{stevenson2022putting}}
SCORING PROTOCOL -- SURPRISE
(5) Very surprising: responses in this category violate your expectations of how the object is meant to be used or how it is often used in practice. Moreover, the use should not be pointless or impossible. Responses in this category should incite interest (and inspiration).
      - A strong 'wow' response
      - Violates expectations in a positive way
      - Incites interest (picture it and think of uses)
    Examples: use a tin can to make a beer can chicken; tin can wallpaper; use book cover to make a bag/wallet; use belt as an oven mitt to grab hot pot handles.

(4) Surprising: responses in this category violate your expectations of how the object is meant to be used or how it is often used in practice, and the use is not pointless or impossible; but the violations are less drastic and less surprising (smaller 'wow' factor) than category 5.
      - Potential for a 'wow' response
      - Violates expectations in a positive way
    Examples: use a tin can as a reflector for your bike; use belt as yoga mat strap; make a chair seat out of belts; hang up a book bookshelf.

(3) Somewhat surprising: responses in this category violate your expectations of the object's conventional uses, but can be related to less obvious uses that are seen more often. The use should be non-obvious, but does not elicit a 'wow' response. Responses that are surprising but vague -- hard to understand why someone would do it, but not completely pointless or useless -- also belong to this category.
      - Non-obvious; violates the conventional uses; somewhat generic uses
      - OR surprising but vague
    Examples: use a tin can as a flowerpot; use fork as a paintbrush; use book pages as wrapping paper; use fork as hook on the wall; send fork in the mail as a message.

(2) Hardly surprising: responses in this category violate your expectations of the object's main use, but are somewhat obvious. Many people will have used the object like this themselves or it is something often seen in the media (or elsewhere). Surprising but pointless responses belong to this category, as well as uses as 'art' or 'decor' without any elaboration.
      - Somewhat obvious (probably seen/have done this); violates the main use
      - OR surprising but pointless
      - OR art/decor without elaboration (can get a higher score with specificity)
    Examples: use a fork as a hairbrush; use fork as drumstick; tin can telephone; press flowers with book; cut belt into pieces.

(1) Not surprising at all: responses in this category do not violate your expectations of the object's uses in any way. Uses that are generic and can be applied to all objects also belong to this category.
      - Obvious (first ideas that come to mind when associating)
      - Conventional uses; conventional to all objects
    Examples: use fork to eat; use brick to weigh something down; use tin can as penholder; throw object; give object away.
\end{promptbox}

\clearpage
\begin{promptbox}{Creativity Evaluation Prompt \citep{goes2023pushing}}
You are evaluating Alternative Uses Test (AUT) responses for object: {object_name}.

Below are {N} alternative uses, each labeled with a numeric ID.
They were generated under different prompt conditions. Judge the uses only by their text.

Use {id}: {use_text}
Use {id}: {use_text}
...

Rank all the alternative uses above by creativity, the least creative to the most creative. Less creative means closer to common use and unfeasible/imaginary, more creative means closer to unexpected uses and also feasible/practical. Assign a score integer number from 1 (least creative use) to 5 (most creative use), and output the results in ascending order according to the score.

Return ONLY valid JSON:
{"ranked_uses": [{"id": 12, "score": 1}, {"id": 7, "score": 2}, ...]}

Include every use ID exactly once, ordered from least creative to most creative, each with its integer score from 1 to 5 (scores must be in ascending order along the list).
\end{promptbox}

\clearpage
\subsection{Infinity-Chat 100 Benchmark \citep{jiang2026artificial}}
\subsubsection{Query Prompts}
These are five manually selected queries we tested from Infinity-Chat 100 benchmark.
\begin{promptbox}{Query \#1}
Generate a joke about electric vehicles.
\end{promptbox}

\begin{promptbox}{Query \#2}
Create the first verse of a wedding vow.
\end{promptbox}

\begin{promptbox}{Query \#46}
Describe Apple Corporation in three sentences to a person who has no idea what cell phones are.
\end{promptbox}

\begin{promptbox}{Query \#60}
Give me a tip for managing a team of coworkers.
\end{promptbox}

\begin{promptbox}{Query \# 78}
Tell me about cats in three words.
\end{promptbox}

\subsubsection{Evolutionary TextGrad Algorithm Prompts}
\begin{promptbox}{Textual Gradient Prompt (IC Response Fitness)}
You are part of an advanced optimization system. Your goal is to evaluate and critique a "persona" that guides a language model during a creative task: answering real-world open-ended user requests (jokes, vows, explanations, advice, descriptions) that admit many equally legitimate answers. The guided model answers each request several times, and the goal is for those answers to be genuinely DIFFERENT from one another rather than restatements of one idea. You are the gradient (feedback) engine.

<OBJECTIVE_FUNCTION>
Your objective is to maximize the fitness of the persona based on four metrics computed on the responses the persona generates:
1. Validity: Every response must be a coherent, on-topic attempt to answer the request — no refusals, no off-topic text, no nonsense (pass or fail).
2. Anti_homogeneity: How semantically different the persona's separate answers to the SAME request are from one another (higher means each attempt expresses a genuinely different idea instead of paraphrasing one favourite answer).
3. Flexibility: The effective number of distinct answers the persona gives per request (higher means it explores several unrelated directions rather than clustering on one or two).
4. Lexical_distinctness: How little verbatim wording the persona's answers to one request share with each other (higher means no recycled phrases, openings, or template sentences across attempts).
</OBJECTIVE_FUNCTION>

To help you understand the current population landscape, here are examples from the current parallel batch:
<BATCH_CONTEXT>
{batch_context}
</BATCH_CONTEXT>

We are interested in giving feedback to the following persona:
<VARIABLE>
{parent_text}
</VARIABLE>

Here are the responses this persona generated, on which its scores were computed:
<GENERATED_RESPONSES>
{generated_responses_block}
</GENERATED_RESPONSES>

Scores for this variable:
Validity: {pass|fail}
Anti_homogeneity: {float:.4f}
Flexibility: {float:.4f}
Lexical_distinctness: {float:.4f}

Provide a concise, specific criticism detailing how to modify this persona to improve its overall fitness.
- If Validity is fail, the persona's responses did not answer the requests. Address the specific failure:
    * If responses drifted off-topic or into self-description, instruct the optimizer to make the persona always deliver a direct answer to the request, whatever its voice.
    * If responses were refusals or empty deflections, instruct the optimizer to remove whatever trait makes the persona decline open-ended requests.
- If Anti_homogeneity is low, the persona keeps giving the same answer in different words. Suggest traits that make it approach each request from a wholly different angle every time — different domains, moods, framings, and viewpoints — rather than orbiting one idea.
- If Flexibility is low, the persona's attempts collapse into one or two clusters. Suggest broadening the persona's interests, expertise, and life experience so its attempts land in genuinely separate territories.
- If Lexical_distinctness is low, the persona recycles wording across attempts. Suggest removing whatever verbal habit, catchphrase, or template causes the repeated phrases, referencing the shown responses as evidence of what is currently recycled.

Do not generate a new persona. Your only job is to provide textual criticism and actionable feedback on how to alter the current persona.
\end{promptbox}

\clearpage
\subsubsection{Evaluation Prompts}
\begin{promptbox}{Validity Evaluation Prompt (Valid/Invalid Boundary from \cite{jiang2026artificial}}
You are a trained rater in a study screening responses to open-ended user requests. Each request admits many different, equally legitimate answers with no single ground truth. You screen ONE response at a time. You do not know or care whether a response was written by a human or by an AI system; judge only the text.

Decide whether the response is a coherent, interpretable attempt to address the request.

Answer "invalid" if the response fails to address the request at all: an empty or nonsense string, a refusal to answer, text on an unrelated topic, a mere restatement of the request without an answer, a formatting artifact, or text so garbled or truncated that no answer can be recovered from it.

IMPORTANT — three things that are NOT grounds for "invalid":
1. Quality is not validity. A bland, generic, clumsy, or unoriginal response that does address the request is VALID — it merely scores low on quality.
2. Voice is not validity. A response written in a strong persona or character voice is VALID as long as an answer to the request is still identifiable in it. Judge only whether an answer survives the styling.
3. Brevity is not validity. If the request asks for something short (a few words, one sentence, a title), a correspondingly short response is exactly right. Judge length only against what the request asks for.

If you are not confident either way, answer "not_sure". A wrong "invalid" deletes real data, so "not_sure" is always safer than guessing.

Request: {query}

Response: {response}

Respond with ONLY a JSON object of the form {"label": "valid"}, {"label": "invalid"}, or {"label": "not_sure"}. No other text.
\end{promptbox}

\begin{promptbox}{Quality Evaluation Prompt \citep{jiang2026artificial}}
You are a trained rater in a study evaluating responses to open-ended user requests. Each request admits many different, equally legitimate answers with no single ground truth — do not penalize a response for taking a different angle than you would have. You rate ONE response at a time. You do not know or care whether a response was written by a human or by an AI system; judge only the text.

Rate the overall quality of the response as an answer to the request, on a 1-5 scale:

5 — excellent: fully addresses the request; clear, apt, and well executed.
4 — good: addresses the request well with only minor weaknesses.
3 — acceptable: a genuine answer, but mediocre in execution or fit.
2 — poor: barely addresses the request, or does so in a confused or badly executed way.
1 — very poor: fails to address the request, or is incoherent.

Judge quality of execution and fit to the request — not which of many valid answer choices was taken, and not length for its own sake (if the request asks for something short, short is right).

Request: {query}

Response: {response}

Respond with ONLY a JSON object of the form {"score": N} where N is an integer from 1 to 5. No other text.
\end{promptbox}

\clearpage
\subsection{DAT Benchmark}
\subsubsection{Baseline Prompts}
\begin{promptbox}{Divergent Association Prompt (original DAT from \cite{olson2021naming,chen2023probing})}
Please write 10 nouns in English that are as irrelevant from each other as possible, in all meanings and uses of the words. Please note that the words you write should have only single word, only nouns (e.g., things, objects, concepts), and no proper nouns (e.g., no specific people or places). Your answer:
\end{promptbox}

\begin{promptbox}{Base-Instruction Prompt from \cite{chen2023probing}}
Write 10 nouns.
\end{promptbox}

\begin{promptbox}{Random-Instruction Prompt from \cite{chen2023probing}}
Write 10 nouns randomly.
\end{promptbox}

\begin{promptbox}{Creative Prompt from \cite{schapiro2026creativityneuro}}
List 10 common English nouns that are as unrelated in meaning as possible. Avoid any shared topic or category. Output only the nouns, separated by commas.
\end{promptbox}

\begin{promptbox}{Non-Divergent Association Prompt (Non-Creative from \cite{schapiro2026creativityneuro})}
List 10 common English nouns that are as closely related in meaning as possible and clearly fit into a single narrow topic. Output only the nouns, separated by commas.
\end{promptbox}

\clearpage
\subsubsection{Evolutionary TextGrad Algorithm Prompts}
\begin{promptbox}{Textual Gradient Prompt (DAT Response Fitness)}
You are part of an advanced optimization system. Your goal is to evaluate and critique a "persona" that guides a language model during a creative task: the Divergent Association Task, where the guided model must name 10 English nouns that are as semantically unrelated to one another as possible. You are the gradient (feedback) engine.

<OBJECTIVE_FUNCTION>
Your objective is to maximize the fitness of the persona based on five metrics computed on the noun lists the persona generates:
1. Validity: Every response must give at least seven usable words — single, common English nouns, no proper nouns, no repeats, no invented words (pass or fail).
2. Dat_score: How semantically distant the named nouns are from one another within a single response (higher means the nouns come from more unrelated regions of meaning).
3. Flexibility: How different the persona's separate attempts are from one another (higher means it explores genuinely different sets of nouns each time instead of repeating one favourite list).
4. Fluency_entropy: How evenly the persona spreads its word choices across its whole vocabulary (higher means it does not keep falling back on the same handful of words).
5. Fluency_top10: How little of the persona's output is taken up by its ten most frequent words (higher means less concentration on a few habitual words).
</OBJECTIVE_FUNCTION>

To help you understand the current population landscape, here are examples from the current parallel batch:
<BATCH_CONTEXT>
{batch_context}
</BATCH_CONTEXT>

We are interested in giving feedback to the following persona:
<VARIABLE>
{parent_text}
</VARIABLE>

Here are the noun lists this persona generated, on which its scores were computed:
<GENERATED_WORDS>
{generated_words_block}
</GENERATED_WORDS>

Scores for this variable:
Validity: {pass|fail}
Dat_score: {float}
Flexibility: {float}
Fluency_entropy: {float}
Fluency_top10: {float}

Provide a concise, specific criticism detailing how to modify this persona to improve its overall fitness.
- If Validity is fail, the persona's responses did not yield seven usable words. Address the specific failure:
    * If the response was not a plain list of single English nouns, instruct the optimizer to make the persona disciplined and strict about answering in the requested format, with no commentary.
    * If the words were proper nouns, invented words, or not nouns at all, instruct the optimizer to make the persona name ordinary, concrete English nouns that any dictionary would list.
    * If words were repeated, instruct the optimizer to make the persona track what it has already said and never repeat a word within one answer.
- If Dat_score is low, the persona's nouns were too closely related to one another. Suggest traits that make it leap between wholly unconnected domains — different senses, scales, and areas of life — rather than listing things from one scene or topic.
- If Flexibility is low, the persona gives near-identical answers every time. Suggest traits that make it approach the task from a different angle on each attempt instead of reciting one rehearsed list.
- If Fluency_entropy is low, the persona keeps reusing a small vocabulary. Suggest broadening the persona's interests, expertise, and life experience so it can draw words from many domains.
- If Fluency_top10 is low, a few habitual words dominate the persona's output. Suggest removing whatever fixation causes those words to recur, referencing the batch context as a baseline for what is currently overrepresented.

Do not generate a new persona. Your only job is to provide textual criticism and actionable feedback on how to alter the current persona.
\end{promptbox}

\end{document}

%% file: math_commands.tex
\usepackage{amsmath,amsfonts,bm}

\def\eqref#1{equation~\ref{#1}}

\def\1{\bm{1}}

\DeclareMathAlphabet{\mathsfit}{\encodingdefault}{\sfdefault}{m}{sl}
\SetMathAlphabet{\mathsfit}{bold}{\encodingdefault}{\sfdefault}{bx}{n}

